\documentclass[journal,transmag]{IEEEtran}
\ifCLASSINFOpdf
\else
\fi
\usepackage[cmex10]{amsmath} 
\usepackage{amsthm,latexsym}
\usepackage{multirow,tabularx}
\usepackage{pdflscape,setspace} 
\usepackage[CaptionAfterwards]{fltpage}
\newtheorem{prop}{Proposition}
\theoremstyle{remark}

\usepackage{algorithm,latexsym,mathrsfs,graphicx}
\usepackage{algpseudocode}
\usepackage{ccaption}

\usepackage[normalem]{ulem}
\usepackage{tikz}
\newcommand*\circled[1]{\tikz[baseline=(char.base)]{
            \node[shape=circle,draw,inner sep=0.5pt] (char) {#1};}}

\usepackage{array}
\ifCLASSOPTIONcompsoc
\else
  \usepackage[caption=false,font=footnotesize]{subfig}
\fi

\begin{document}
%\onecolumn

\title{A Novel Adaptive Possibilistic Clustering Algorithm}

\author{\IEEEauthorblockN{Spyridoula~D.~Xenaki\IEEEauthorrefmark{1,}\IEEEauthorrefmark{2}, Konstantinos~D.~Koutroumbas\IEEEauthorrefmark{1}, and Athanasios~A.~Rontogiannis\IEEEauthorrefmark{1}}%
\IEEEauthorblockA{\IEEEauthorrefmark{1}Institute for Astronomy, Astrophysics, Space Applications and Remote Sensing (IAASARS),\\ National Observatory of Athens, Penteli, GR-15236 Greece}
\IEEEauthorblockA{\IEEEauthorrefmark{2}Department of Informatics and Telecommunications, National \& Kapodistrian University of Athens,\\ GR-157 84, Ilissia, Greece}%
%\thanks{Manuscript received November 21, 2014; 1st revision April 3, 2015; 2nd revision June 19, 2015; 3rd revision August 15, 2015. Corresponding author: S.D. Xenaki (email: ixenaki@noa.gr). 
%
%This research has been co-financed by the European Union (European Social Fund - ESF) and Greek national funds through the Operational Program "Education and Lifelong Learning" of the National Strategic Reference Framework (NSRF) - Research Funding Program: ARISTEIA- HSI-MARS-1413.}
}

\IEEEtitleabstractindextext{%
\begin{abstract}
In this paper a novel possibilistic c-means clustering algorithm, called Adaptive Possibilistic c-means, is presented. Its main feature is that its parameters, after their initialization, are properly adapted during its execution. Provided that the algorithm starts with a reasonable overestimate of the number of physical clusters formed by the data, it is capable, in principle, to unravel them (a long-standing issue in the clustering literature). This is due to the fully adaptive nature of the proposed algorithm that enables the removal of the clusters that gradually become obsolete. In addition, the adaptation of all its parameters increases the flexibility of the algorithm in following the variations in the formation of the clusters that occur from iteration to iteration. Theoretical results that are indicative of the convergence behavior of the algorithm are also provided. Finally, extensive simulation results on both synthetic and real data highlight the effectiveness of the proposed algorithm.
\end{abstract}

}

\maketitle

\IEEEdisplaynontitleabstractindextext

% \IfCLASSOPTIONpeerreview
% \fi
\IEEEpeerreviewmaketitle

\section{Introduction}
%
%
%
%\IfCLASSOPTIONcompsoc
%  \noindent\raisebox{2\baselineskip}[0pt][0pt]%
%  {\parbox{\columnwidth}{\section{Introduction}\label{sec:introduction}%
%  \global\everypar=\everypar}}%
%  \vspace{-1\baselineskip}\vspace{-\parskip}\par
%\Else
%  \section{Introduction}\label{sec:introduction}\par
%\fi
%
% Admittedly, this is a hack and may well be fragile, but seems to do the
% trick for me. Note the need to keep any \label that may be used right
% after \section in the above as the hack puts \section within a raised box.

\IEEEPARstart{C}{lustering} is a well established data analysis methodology that has been extensively used in various fields of applications during the last decades, such as life sciences, medical sciences and engineering \cite{Ande73}. Given a set of entities, its aim is the identification of groups (clusters) formed by ``similar" entities (e.g. \cite{Theo09}, \cite{Duda73}, \cite{Jain88}, \cite{Ever01}). Usually, each entity is represented by a set of measurements, which forms its associated {\it feature vector}. This is also called {\it data vector} and the set of all these vectors forms the {\it data set} under study. The space where all these vectors live is called {\it feature space}. The clustering of the entities under study is based exclusively on the clustering of their corresponding feature vectors. According to the way a data vector is associated with various clusters, three main philosophies have been developed: (a) {\it hard clustering}, where each vector belongs exclusively to a single cluster, (b) {\it fuzzy clustering}, where each vector may be shared among two or more clusters and (c) {\it possibilistic clustering}, where the association ({\it degree of compatibility}) of a data vector with a given cluster is independent of its association with any other cluster.

Most of the work on clustering has been focused on compact and hyperellipsoidally shaped clusters and the most well-known algorithms that deal with this case and follow one of the previous philosophies are the {\it k-means} (hard clustering), e.g. \cite{Hart79}, the {\it fuzzy c-means} (FCM), e.g. \cite{Bezd80}, \cite{Bezd81} and the {\it possibilistic c-means} (PCM), e.g. \cite{Theo09}, \cite{Kris93}, \cite{Kris96}, \cite{Pal05}, \cite{Yang06}, \cite{Tree13}, respectively. In all these algorithms the clusters are represented by vectors that lie in the feature space, called {\it cluster representatives}. The aim of all these algorithms is to move the representatives to the ``centers" of the regions that are ``dense in data points" (dense regions), that is to regions where there is significant aggregation of data points (clusters). Under this perspective, we say that each such vector {\it represents} a cluster and their movement towards the center of the clusters is carried out via the minimization of appropriately defined cost functions.

Notwithstanding their popularity, both k-means and FCM have two shortcomings. First, they are vulnerable to noisy data and outliers \cite{Theo09}, \cite{Pal05}\footnote{A method for facing this problem with FCM is discussed in \cite{Yang08}.}. Second, they require prior knowledge of the number of clusters, $m$, underlying in the data set (which, of course, is rarely known in practice)\footnote{A method for estimating $m$ for FCM is via the use of suitable validity indices (e.g., \cite{Xie91}, \cite{Yang06}).}. An additional characteristic that both of these algorithms share is that they {\it impose} a clustering structure on the data set, in the sense that they will return $m$ clusters irrespectively of the fact that more or less than $m$ clusters may actually underlie in the data set. Specifically, if $m$ is less than the actual number of clusters, at least some representatives will fail to move to dense regions, while in the opposite case, some naturally formed clusters will split into more than one pieces. 

As far as the PCM algorithms are concerned, the cluster representatives are updated, based on the {\it degree of compatibility} of a data vector with a given cluster. Contrastingly to the FCM, in PCM algorithms, the degrees of compatibility of a data vector with the various clusters are independent to each other and no sum-to-one constraint is imposed on them. A consequence of this fact is that even if the number of clusters is overestimated, in principle, all representatives will be driven to dense regions, making thus feasible the uncovering of the true clusters. However, in this case, the scenario where two or more cluster representatives are led to the same dense in data region, may arise \cite{Zhan04}, \cite{Barn96}. In addition, although PCM deals well with noisy data points and outliers, compared to k-means and FCM, it involves additional parameters, usually denoted by $\gamma$ \footnote{In other works the letter $\eta$ is used.}, each one being associated with a cluster, which require good estimates. In addition, once they have been estimated, they are kept fixed during its execution. Poor initial estimation of these parameters often leads to poor clustering performance, especially in more demanding data sets. 

Many variants of PCM have been proposed to deal with its weaknesses. More specifically, \cite{Timm04} tries to avoid coincident clusters by introducing mutual repulsion of the clusters, so that they are forced away from each other. The same problem is treated in \cite{Zhan04}, \cite{Masu06} and \cite{Pal05} by combining possibilistic and fuzzy arguments. Also, in \cite{Xie08} a strategy is proposed that intoduces a ``gray zone" around each representative, which contains the points around the cluster boundary. The latter deals with the coincident clusters problem, is robust to outliers and uses less ad hoc defined parameters than PCM. Another algorithm that involves very few parameters and is robust to noise and outliers is described in \cite{Yang06}. In \cite{Wu10} ideas from \cite{Yang06} and \cite{Pal05} are combined for dealing additionaly with the coincident clusters issue. The same issues are also addressed in \cite{Tree13} using, however, a different approach than \cite{Wu10}.

%Variants of PCM that try to address these issues have been proposed in \cite{Pal05}, \cite{Yang06}, \cite{Tree13} and \cite{Zhan04}. 
%
The original versions of PCM algorithms have no cluster elimination ability, that is, if they are initialized with an overestimated number of clusters, they cannot eliminate any of them as they evolve. Inspired by \cite{Yang04}, PCM-type algorithms that perform cluster elimination during their execution are described in \cite{Yang11} and \cite{Liao11}. However, in these algorithms the parameters $\gamma$ are considered equal for all clusters and are kept fixed as they evolve. Consequently, their ability to deal with closely located clusters with significantly different variances, is drastically decreased. In addition, their computational complexity is dramatically increased \cite{Yang11}. 
%
%, in order to end up with the correct number of clusters (provided that they are initialized by a larger number of them), are described in \cite{Yang11} and \cite{Liao11}. In these, the $\gamma$ parameters are considered equal for all clusters. However, this decreases significantly their ability to deal with cases where clusters of signifanctly different variances are closely located. An additional issue with them is that they require high computational effort, since the way the merging is performed is inspired from the hierarchical agglomrative algorithms. Also, they require the definition of several (more than one) parameters. 
% In particular, PCM has the undesireble tendency to produce coincident clusters, even by starting from {\it good} initial parameters estimation \cite{Barn96}.

In the present work, we focus on PCM. More specifically, we extent the classical PCM algorithm, proposed in \cite{Kris96}, by modifying the way the parameters $\gamma$ are defined and treated, giving rise to a new algorithm called {\it Adaptive Possibilistic c-means} ({\it APCM})\footnote{A preliminary version of APCM has been presented in \cite{Xen13}.}. In APCM the parameters $\gamma$, after their initialization, are properly adapted as the algorithm evolves. In particular, for each specific cluster, we propose to adapt its parameter $\gamma$ based on the {\it mean absolute deviation} of only those data vectors that are {\it most compatible} with this cluster. 
%Note that this is in contrast to the usual policy of estimating $\gamma$'s, as a measure of variance. However, due to this modification, it is shown that an appropriate (but fully specified in most practical cases) scaling of the data should be conducted, before the execution of the algorithm.

The adaptation of $\gamma$'s renders the algorithm more flexible in uncovering the underlying clustering structure, compared to other related possibilistic algorithms, especially in demanding data sets such as those consisting of closely located to each other clusters or even with big difference in their variances. In addition, as a direct consequence of this adaptation, the algorithm has the ability to estimate the (unknown in most cases in practice) true number of {\it physical} (or {\it natural}) clusters. More specifically, if the number of the representatives with which APCM starts is a crude overestimation of the number of natural clusters, the algorithm gradually reduces this number, as it progresses, and, finally, it places a single representative to the center of each dense region. In this sense, it provides not only the number of natural clusters, which is a long-standing issue in the clustering framework, but also the clusters themselves. Analytical results are presented that justify the cluster elimination capability of the proposed algorithm and provide strong indications of its convergence behavior. Extensive simulation results on both synthetic and real data, corroborate our theoretical analysis and show that APCM offers in general superior clustering performance compared to relative state-of-the-art clustering schemes.
%. Obviously, this is a very important property that improves significantly clustering performance.

The rest of the paper is organized as follows. In Section \ref{sec2}, a brief description of PCM algorithms is given, as well as previous attempts for dealing with their shortcomings. In Section \ref{sec3}, the proposed Adaptive PCM (APCM) clustering algorithm is presented in detail and its rationale is fully explained in a separate subsection. In Section \ref{sec4}, the performance of APCM is tested against several related state-of-the-art algorithms. Concluding remarks are provided in Section \ref{sec5}. Finally, indicative theoretical convergence results of APCM are given in Appendix B.

%\section{Related Work}

%\section{A Brief Review of PCM}

\section{A review of PCM, issues and potential solutions}
\label{sec2}
In this section, the PCM clustering algorithm is reviewed and its main features are discussed. Also, possible solutions from the literature are commented that try to deal with its weak points.

\subsection{PCM review}

Let $X=\{\mathbf{x}_i \in \Re^\ell, i=1,...,N\}$ be a set of $N$, $l$-dimensional data vectors to be clustered and $\Theta=\{\boldsymbol{\theta}_j \in \Re^\ell, j=1,...,m\}$ be a set of $m$ vectors that will be used for the representation of the clusters formed by the points in $X$. Let $U=[u_{ij}], i=1,...,N, j=1,...,m$ be an {$N\times m$} matrix whose $(i,j)$ entry stands for the so-called {\it degree of compatibility} of $\mathbf{x}_i$ with the $j$th cluster, denoted by $C_j$, and represented by the vector $\boldsymbol{\theta}_j$. 
%Let also ${\mathbf{u}_i}^T=[u_{i1},...,u_{im}]$ be the vector containing the elements of the $i$th row of $U$. 
In what follows we consider only Euclidean norms, denoted by $\|\cdot\|$.

%\subsection{Preliminaries}
Unlike fuzzy clustering algorithms, the sum-to-one constraint is not imposed on the rows of $U$ in possibilistic clustering algorithms, i.e. the summation $\sum\nolimits_{j=1}^m u_{ij}$ is not necessarily equal to $1$ for each $\mathbf{x}_i$. According to \cite{Kris93}, \cite{Kris96}, the $u_{ij}$'s should satisfy the conditions, 
%\begin{equation}
\begin{multline}
\label{conditions}
(\text{C1})\ {u_{ij} \in [0,1]},\ \ (\text{C2})\ {\max_{j=1,...,m}u_{ij}>0} \\ \text{and}\ \ (\text{C3})\ \ {0<\sum\limits_{i=1}^N u_{ij}<N}
%\end{equation}
\end{multline}
In words, (C2) means that no vector is allowed to be totally incompatible with all clusters, whereas (C3) means that for a given cluster, there is at least one data point that is not totally incompatible with it. Loosely speaking, each data point should ``belong" to at least one cluster (C2), whereas no cluster is allowed to be ``empty" (C3). The aim of a possibilistic algorithm is to move $\boldsymbol{\theta}_j$'s towards the centers of regions where the data points of $X$ form aggregations (i.e. to dense regions). This is carried out via the minimization of, among others, the following objective function \cite{Kris96}\footnote{We use this cost function, instead of the one given in the seminal paper \cite{Kris93}, since the proposed scheme, to be presented in the next section, is based on it. However, for reasons of thoroughness, we give also the cost function of \cite{Kris93}, which is
{\scriptsize $$J_{PCM}'(\Theta,U)={\sum\limits_{j=1}^m J_j'} \equiv \sum\limits_{j=1}^m \left[ \sum\limits_{i=1}^N u_{ij}^q\|\mathbf{x}_i-\boldsymbol{\theta}_j\|^2 + \gamma_j \sum\limits_{i=1}^N (1-u_{ij})^q \right ]$$}
where $q$ is a parameter that ``resembles'' to the fuzzifier in FCM (such a parameter does not appear in $J_{PCM}$ in eq.~({\ref{Jpcm}})).\label{foot2}}:
%
%\begin{equation}
\begin{multline}
J_{PCM}(\Theta,U)=\sum\limits_{j=1}^m J_j \equiv \\ \equiv\sum\limits_{j=1}^m \overbrace{\left[ \sum\limits_{i=1}^N u_{ij}\|\mathbf{x}_i-\boldsymbol{\theta}_j\|^2 + \gamma_j \sum\limits_{i=1}^N (u_{ij}\ln u_{ij}-u_{ij}) \right ]}^{J_j}
\label{Jpcm}
\end{multline}
%\end{equation}
%
with respect to $\boldsymbol{\theta}_j$'s and $u_{ij}$'s, while $\gamma_j$'s are {\it fixed} user-defined positive parameters. Note that the second term in the bracketed expression in the right hand side of eq. (\ref{Jpcm}) prevents the algorithm from ending up with the trivial zero solution for $u_{ij}$'s.
%
%Before we proceed, let us open a parenthesis here to comment on $\eta_j$'s. Let $$H=\{\eta_j \in \Re^\ell, j=1,...,m\}, $$where $\eta_j$ is a positive parameter associated with the cluster $C_j$ \cite{Kris96} and indicates the degree of ``influence" of cluster $C_j$ around its representative $\boldsymbol{\theta}_j$; the smaller (greater) the value of $\eta_j$, the smaller (greater) the influence of $C_j$ around $\boldsymbol{\theta}_j$. In general, $\eta$'s are kept fixed during the execution of the algorithm. One way to estimate $\eta_j$ is to run the FCM algorithm first and after its convergence, to set $\eta_j=\frac{\sum\nolimits_{u_{ij} > \gamma}^{} d^{\boldsymbol{\theta}}_{ij}}{\sum\nolimits_{u_{ij} > \gamma}^{} 1},$ where $\gamma$ is an appropriate threshold \cite{Kris96}. However, since a prerequisite for the FCM to provide good clustering results is the accurate knowledge of the number of clusters (which is rarely the case in practice), the estimates for $\eta$'s are, in most cases, not very accurate. Consequently, this usually leads to poor results, especially for more demanding data sets. Closing the parenthesis, we proceed with the minimization of $J_{PCM}(\Theta,U)$ with respect to $u_{ij}$ and $\boldsymbol{\theta}_j$, which leads to the following two updating equations,
%

Proceeding with the minimization of $J_{PCM}(\Theta,U)$ with respect to $u_{ij}$ and $\boldsymbol{\theta}_j$, we end up with the following PCM updating equations,
%\vspace{-0.1in}
\begin{center}
\begin{minipage}{.9\linewidth}
\begin{equation}
	u_{ij}(t)=\exp\left(-\frac{ \|\mathbf{x}_i - \boldsymbol{\theta}_j(t) \|^2 }{\gamma_j}\right)
\label{uij}
\end{equation}
\end{minipage}
%\text{${}_{,}$ }
\begin{minipage}{.9\linewidth}
\begin{equation}
	\boldsymbol{\theta}_j(t+1)=\frac{\sum\nolimits_{i=1}^N u_{ij}(t) \mathbf{x}_i}{\sum\nolimits_{i=1}^N u_{ij}(t)},
\label{theta}
\end{equation}
\end{minipage}
%\text{${}_.$ }%
\end{center}
for $t=0,1,2,\ldots$, with the iterations being started after the initialization of $\boldsymbol{\theta}_j$'s to $\boldsymbol{\theta}_j(0)$'s, $j=1,\ldots,m$. Iterations are performed until a specific termination criterion is met (e.g., no significant change occurs on $\boldsymbol{\theta}_j$'s between two succesive iterations). Note from the updating eq.~(\ref{uij}) that $u_{ij}$ decreases exponentially fast as the distance between $\mathbf{x}_i$ and $\boldsymbol{\theta}_j$ increases. Also, from eq. (\ref{theta}), it follows that {\it all} data vectors contribute to the estimation of the next location of each one of the representatives. 
However, the farther a data vector lies from the current location of a specific $\boldsymbol{\theta}_j$ the less it contributes to the determination of its new location, as eq.~(\ref{uij}) indicates. 

%Obviously, the estimates of the $u_{ij}$'s highly affect the estimation accuracy of $\boldsymbol{\theta}_j$'s (eq.~\ref{theta}). In addition, $u_{ij}$'s update is highly dependent from the $\eta_j$ parameter (a fact that is further magnified through the presence of the $\exp(\cdot)$ function), thus making imperative an accurate assessment of the latter. It is worth noting from eq.~\ref{uij} that the parameter $\eta_j$ determines the squared Euclidean distance from the cluster representative $\boldsymbol{\theta}_j$, at which the degree of compatibility of a data point becomes equal to $\exp(-1)=0.368$.

Let us comment now on the parameters $\gamma_j$, $j=1,\ldots,m$. These are a priori estimated and kept fixed during the execution of the algorithm. A common strategy for their estimation is to run the FCM algorithm first and set
\begin{equation}
\label{Ketaj}
\gamma_j=K \frac{ \sum_{i=1}^N u^{FCM}_{ij} \|\mathbf{x}_i-\boldsymbol{\theta}_j\|^2 }{ \sum_{i=1}^N u^{FCM}_{ij} }, \ \ \ j=1,\ldots,m
\end{equation}
where $\boldsymbol{\theta}_j$'s and $u^{FCM}_{ij}$'s are the final FCM estimates for cluster representatives and $u_{ij}$ coefficients, respectively\footnote{The version of eq.~(\ref{Ketaj}) proposed in \cite{Kris93} for the cost function $J'_{PCM}$ (see footnote~\ref{foot2}), raises $u^{FCM}_{ij}$'s to the $q$th power. However, in $J_{PCM}$ no parameter $q$ is involved.}. Parameter $K$ is user-defined and is usually set equal to $1$ \footnote{An alternative choice for $\gamma_j$'s, given in \cite{Kris93} is $\gamma_j=\frac{\sum\nolimits_{u_{ij} > k}^{} \|\mathbf{x}_i-\boldsymbol{\theta}_j\|^2}{\sum\nolimits_{u_{ij} > k}^{} 1},$ where $k$ is an appropriate threshold.}. From eq.~\eqref{Ketaj}, $\gamma_j$ turns out to be a measure of variance of cluster $C_j$ around its representative.

It is worth noting that, due to the independence between $u_{ij}$'s, $j=1,\ldots,m$, for a specific $\mathbf{x}_i$, the optimization problem solved by PCM can be decomposed into $m$ sub-problems, each one optimizing a specific $J_j$ function (see eq. (\ref{Jpcm})). Considering the representative $\boldsymbol{\theta}_j$ associated with a given $J_j$, we have from eq. (\ref{uij}) that points that lie closer to the cluster representative will have larger degrees of compatibility with $C_j$. On the other hand, eq.~(\ref{theta}) implies that the new position of $\boldsymbol{\theta}_j$ is mainly specified by the data points that are most compatible with $C_j$. It is not difficult to see that such a coupled iteration is expected to lead representative $\boldsymbol{\theta}_j$ towards the center of the dense in data region that lies closer to its initial position, for appropriate choices of $\gamma_j$'s (see also propositions 3 and 4, in Appendix B).

\subsection{PCM issues and potential solutions}
\label{subsec22}
Having described the main characteristics of the algorithm and the rationale behind them, let us focus now on some issues that a user faces with PCM. The first one concerns the $m$ parameters $\gamma_j$'s. An improper choice of $\gamma_j$'s may lead PCM to failure in identifying a sparse cluster that is located very close to a denser cluster (see also experiment 1, in section~\ref{subsec41}), or it may even lead the algorithm to recover the whole data set as a single cluster \cite{Barn96}. Referring to eq.~(\ref{Ketaj}), the $u_{ij}$'s produced by the FCM ($u^{FCM}_{ij}$'s), are not always accurate (e.g. in the presence of noise, \cite{Kris96}). In addition, the choice of the parameter $K$ is clearly data-dependent and there is no general clue on how to select it. In order to deal with this problem, \cite{Yang06} proposes the replacement of all $\gamma_j$'s by a single quantity that is controlled by only two parameters: (a) the number of clusters and (b) a parameter that plays a ``fuzzifier" role.

An additional source of inconveniences concerning $\gamma_j$'s is the fact that, once they have been set, they remain fixed during the execution of PCM. This reduces the ability of the algorithm to track the variations in the clusters formation during its evolution. A way out of this problem is to allow $\gamma_j$'s to vary during the execution of the algorithm. A hint on this issue has been given in \cite{Kris93}, but, to the best of our knowledge, no further work has been done towards this direction.

The second issue, which is related with the first one, is that of {\em coincident clusters}. As stated before, with a proper choice of $\gamma_j$'s, PCM drives, in principle, the cluster representatives towards the centers of the dense in data regions that are closer to their initial positions. Therefore, if two or more representatives are initialized close to the same dense region, they will move towards its center, i.e., all of them will represent the same cluster. Alternatively, one could say that the clusters represented by these representatives are  coincident\footnote{This point of view justifies the term ``coincident clusters".}. This situation arises due to the absence of dependence between the coefficients $u_{ij}$, $j=1,\ldots,m$, associated with a specific $\mathbf{x}_i$ (see eq. (\ref{uij})), which, as an indirect consequence, allows the representatives to move independently from each other (see eq. (\ref{theta})). Note that such an issue does not arise in FCM due to the sum-to-one constraint imposed on the $u_{ij}$'s associated with each $\mathbf{x}_i$. Several ways to deal with this problem have been proposed in the literature. More specifically, in \cite{Pal05}, a variation of PCM is proposed, named Possibilistic Fuzzy c-means (PFCM), which combines concepts from PCM and FCM.
%, so that, two sets of coefficients are used to denote the relationship of each vector $\mathbf{x}_i$ with the clusters $C_j$'s: the first set, denoted by $t_{ij}$, $j=1,\ldots,m$, measures the degree of compatibility of $\mathbf{x}_i$ with $C_j$'s (coming from the PCM rationale) and the second set, denoted by $u_{ij}$, $j=1,\ldots,m$, measures the degree of belongness of $\mathbf{x}_i$ to each $C_j$ (coming from the FCM rationale), where in the latter set the sum to one constraint is imposed, i.e. $\sum_j u_{ij}=1$, for each $\mathbf{x}_i$. This gives the required dependence between parameters referring to a given $\mathbf{x}_i$, which will prevent the appearence of coincident clusters. 
Relative approaches are discussed in \cite{Zhan04}, \cite{Pal97}, \cite{Wu10}, while other approaches are proposed in \cite{Tree13}, \cite{Timm04}.
% An alternative approach is discussed in \cite{Tree13}, where the $u_{ij}$'s themselves associated with a given $\mathbf{x}_i$ are normalized, giving thus the required dependence for not ending up with coincident clusters. Finally, a conceptually different aproach to the issue of coincident clusters is described in \cite{Timm04}, where the cluster representatives themselves are interact to each other; the closer they are located to each other, the bigger the repulsion between them becomes.

A common feature in all the previously mentioned works, is that condition (C3), which basically requires all clusters to be non-empty, is respected. Thus, in all the algorithms, the true number of clusters $m$ is implicitly required, in order to give them the ability to recover all clusters, without, hopefully, returning coincident clusters. Thus, the requirement of the knowledge of the number of clusters is still here in disguise. A conceptually simple solution to address this requirement, while respecting condition (C3), comes from the PCM itself. Specifically, one could run the original PCM with an overestimated number of cluster representatives which will be initialized appropriately (at least one representative should lie at each dense in data region). Then, after a proper selection of $\gamma_j$'s, PCM will (hopefully) recover the physical clusters, that is, it will move at least one representative to the center of each dense region. Then, an additional step is required in order to identify coincident clusters and remove duplicates. This idea has been partially discussed in \cite{Kris96}, without, however, proposing explicitly to run the algorithm with an overdetermined number of clusters. However, in this case a reliable method for identifying duplicate clusters should be invented.

%In this work an alternative solution for facing the previous issues is proposed, that is, to remove completely condition (C3), i.e., to allow some clusters to become ``empty". Doing so, one can initialize PCM with an overestimated number of (appropriately initialized) cluster representatives and (after some suitable modifications) as the algorithm evolves, some clusters will become empty and will be removed during its execution. Thus, the algorithm is likely to end up with the correct number of clusters and, of course, with no coincident clusters.

The APCM algorithm proposed in this paper alleviates the shortcomings of PCM discussed previously. The aims of APCM are (a) to place initially at least one representative to each physical cluster and (b) to retain each representative to the physical cluster where it was first placed, leading it gradually to its center. The first aim is achieved by adopting the results provided by the FCM algorithm, when the latter is executed with an overestimated number of physical clusters. The second one is achieved through a different definition of $\gamma_j$'s, while, in addition (and perhaps more importantly), $\gamma_j$'s are allowed to adapt at each iteration of the algorithm. In contrast to PCM the adopted expression of $\gamma_j$'s takes into account only the points that are most compatible with the corresponding $C_j$'s at each iteration of the algorithm. The benefit of the proposed approach is twofold. First, the algorithm becomes more flexible in tracking the variations in the formation of the clusters, as it evovles. As a result, APCM is capable in dealing with difficult clustering problems in which PCM frequently fails, e.g. the identification of small and/or sparse physical clusters that are located close to bigger and/or denser clusters. Second, the algorithm allows the possibility for a cluster to become empty, by reducing its corresponding $\gamma_j$ towards zero. This allows us to start the APCM with an overestimated number of natural clusters and end up with a single representative placed to the center of each natural cluster, eliminating all the remaining representatives\footnote{This is theoretically justified in proposition 5, in Appendix B}. Thus, in APCM the number of clusters, $m$, is also considered to be a time-varying quantity.

\section{The Adaptive PCM (APCM)}
\label{sec3}
In this section, we describe in detail the various stages of the algorithm. Specifically, we first describe the way its parameters are initialized. Next, we comment on the updating of its parameters ($u_{ij}$'s, $\boldsymbol{\theta}_j$'s, $\gamma_j$'s, $m$) and we discuss in detail, how the initial estimate of the number of natural clusters can be reduced to the true one, by exploiting the adaptation of $\gamma_j$'s.

The proposed APCM algorithm stems from the optimization of the cost function of the original PCM (eq.~\eqref{Jpcm}), by setting
\begin{equation}
\gamma_j=\frac{\hat{\eta}}{\alpha}\eta_j
\label{gamma}
\end{equation}
where, parameter $\eta_j$ is a measure of the mean absolute deviation of the current form of cluster $C_j$, $\hat{\eta}$ is a constant defined as the minimum among all initial $\eta_j$'s, $\hat{\eta}=\min\limits_{j} \eta_j$ and $\alpha$ is a user-defined positive parameter. The rationale of the adopted expression for $\gamma_j$'s as given in eq.~\eqref{gamma} will be analyzed and further discussed in subsection~\ref{subsec3}.

\subsection{Initialization in APCM}
\label{subsec31}
As mentioned previously, first, we make an overestimation, denoted by $m_{ini}$, of the true number of natural clusters $m$, formed by the data points. Regarding $\boldsymbol{\theta}_j$'s and $\eta_j$'s, their initialization drastically affects the final clustering result in APCM. Thus, a good starting point for them is of crucial importance. To this end, the initialization of $\boldsymbol{\theta}_j$'s is carried out using the final cluster representatives obtained from the FCM algorithm, when the latter is executed with $m_{ini}$ clusters. Taking into account that FCM is very likely to drive the representatives to dense in data regions (since $m_{ini}>m$), the probability of at least one of the initial $\boldsymbol{\theta}_j$'s to be placed in each dense region (cluster) of the data set, increases with $m_{ini}$.

After the initialization of $\boldsymbol{\theta}_j$'s, we propose to initialize $\eta_j$'s as follows:
\begin{equation}
\label{initetaj}
\eta_j=\frac{ \sum_{i=1}^N u^{FCM}_{ij} \|\mathbf{x}_i-\boldsymbol{\theta}_j\|}{ \sum_{i=1}^N u^{FCM}_{ij} }, \ \ \ j=1,\ldots,m_{ini}
\end{equation}
where $\boldsymbol{\theta}_j$'s and $u^{FCM}_{ij}$'s in eq.~(\ref{initetaj}) are the final parameter estimates obtained by FCM\footnote{An alternative initialization for $\boldsymbol{\theta}_j$'s and $\eta_j$'s is proposed in \cite{Xen13}.}.

It is worth noting that the above initialization of $\eta_j$'s involves {\it Euclidean instead of squared Euclidean distances}, as is the case for $\gamma_j$'s in the classical PCM. As it will be shown next, this convention will also be kept in the update expressions of $\eta_j$'s, given below, while its rationale is explained in Section \ref{subsec3}.

\subsection{Parameter adaptation in APCM}
In the proposed APCM algorithm, all parameters are adapted during its execution. More specifically, this refers to, (a) the degrees of compatibility $u_{ij}$'s and the cluster representatives $\boldsymbol{\theta}_j$'s, (b) the adjustment of the number of clusters and (c) the adaptation of $\eta_j$'s, with (b) and (c) being two interrelated processes. 

As far as the updating of $u_{ij}$'s is concerned, after setting $\gamma_j=\frac{\hat{\eta}}{\alpha}\eta_j$ in eq.~\eqref{Jpcm} and minimizing $J_{PCM}(\Theta,U)$ with respect to $u_{ij}$, we end up with the following equation
\begin{equation}
	u_{ij}(t)=\exp\left(-\frac{ \|\mathbf{x}_i - \boldsymbol{\theta}_j(t) \|^2 }{\gamma_j(t)}\right)= \exp\left(-\frac{\alpha}{\hat{\eta}}\frac{ \|\mathbf{x}_i - \boldsymbol{\theta}_j(t) \|^2 }{\eta_j(t)}\right)
\label{uij2}
\end{equation}
where iteration dependence on $\eta_j$'s has now been inserted. The updating of $\boldsymbol{\theta}_j$'s is done as in the original PCM scheme according to eq.~(\ref{theta}). Concerning the adjustment of the number of clusters $m(t)$ at $t$th iteration, we proceed as follows. Let $label$ be a $N$-dimensional vector, whose $i$th element is the index of the cluster which is {\it most compatible} with $\mathbf{x}_i$, that is the index $j$ for which $u_{ij}(t)=\max_{r=1,...,m(t)} u_{ir}(t)$. At each iteration of the algorithm, the adjustment (reduction) of the number of clusters $m(t)$ is achieved by examining, for each cluster $C_j$, if its index $j$ appears at least once in the vector $label$ (i.e. if there exists at least one vector $\mathbf{x}_i$ that is most compatible with $C_j$). If this is the case, $C_j$ is preserved. Otherwise, $C_j$ is eliminated and, thus, $U$ and $\Theta$ are updated accordingly. As a result, the current number of clusters $m(t)$ is reduced (see {\it Possible cluster elimination} part in Algorithm~\ref{alg:apcm}).
%\textbf{Algorithm 1}). 

Finally, concerning $\gamma_j(t)$'s, in contrast to the classical PCM where $\gamma_j$'s remain constant during the execution of the algorithm, in APCM the parameters $\gamma_j$'s, given in eq.~\eqref{gamma}, {\it are adapted} at each iteration through the adaptation of the corresponding $\eta_j$'s. More specifically, we propose to compute the parameter $\eta_j$ of a cluster $C_j$ at each iteration, as the {\it mean absolute deviation} of the most compatible to cluster $C_j$ data vectors, i.e., 
\begin{equation}
\eta_j(t+1)=\frac{1}{n_j(t)}\sum\nolimits_{\mathbf{x}_i:u_{ij}(t)=\max_{r=1,...,m(t+1)} u_{ir}(t)}^{} \|\mathbf{x}_i-\boldsymbol{\mu}_j(t)\|
\label{adapteta}
\end{equation}
where $n_j(t)$ denotes the number of the data points $\mathbf{x}_i$ that are most compatible with the cluster $C_j$ at iteration $t$ and $\boldsymbol{\mu}_j(t)$ the mean vector of these data points (see also {\it Adaptation of $\eta_j$'s} part in Algorithm~\ref{alg:apcm}). Note that, the definition of $\gamma_j$'s in the proposed updating mechanism from eqs.~\eqref{gamma},~\eqref{adapteta}, differs from others used in the classical PCM, as well as in many of its variants, in two distinctive points. First, $\eta_j$'s in APCM are updated taking into account {\it only} the data vectors that are most compatible to cluster $C_j$ and not all the data points weighted by their corresponding $u_{ij}$ coefficients. This particularity is an essential condition for succeeding cluster elimination, as by this way a parameter $\eta_j$ may be pushed to zero value, thus eliminating the corresponding cluster $C_j$, whereas in the case where all data points are taken into account, $\eta_j$ would remain always positive. Second, the distances involved in eq.~(\ref{adapteta}) are between a data vector and the mean vector $\boldsymbol{\mu}_j(t)$ of the most compatible points of the cluster; {\it not} from $\boldsymbol{\theta}_j(t)$, as in previous works (e.g. \cite{Kris93}, \cite{Zhan04}). This allows more accurate estimates of $\eta_j$'s, since $\boldsymbol{\mu}_j(t)$ is expected to be closer to the next location of $\boldsymbol{\theta}_j$, $\boldsymbol{\theta}_j(t+1)$, than $\boldsymbol{\theta}_j(t)$. This is crucial mainly during the first few iterations of the algorithm where the position of $\boldsymbol{\theta}_j$ may vary significantly from iteration to iteration. It is also noted that, in the (rare) case where there are two or more clusters, that are equally compatible with a specific $\mathbf{x}_i$, the latter will contribute to the determination of the parameter $\eta$ of {\it only} one of them, which is chosen arbitrarily. This modification prevents a situation of having equal $\eta_j$'s in such exceptional cases (e.g. in data sets consisting of symmetrically arranged data points) which assists the successful cluster elimination procedure, in situations where this must be carried out. Finally, it is worth pointing out that the definition of eq.~\eqref{adapteta}, implicitly interrelates the various $\gamma_j$'s and this interrelation passes to the $u_{ij}$'s concerning a given $\mathbf{x}_i$ through eq.~\eqref{uij}. 

The APCM algorithm can be stated as follows.

\begin{algorithm}[H]
\caption{ [$\Theta$, $U$, $label$] = APCM($X$, $m_{ini}$, $\alpha$)}
\label{alg:apcm}
\begin{algorithmic}[1]
\Require {$X$, $m_{ini}$, $\alpha$}
\State $t=0$
\Statex {\Comment{Initialization of $\boldsymbol{\theta}_j$'s}}
\State \textbf{Initialize:} $\boldsymbol{\theta}_j(t)$ {\it via FCM (see subsection \ref{subsec31})}
\Statex {\Comment{Initialization of $\eta_j$'s}}
   \State \textbf{Set:} $\eta_j(t)=\frac{ \sum_{i=1}^n u^{FCM}_{ij} \|\mathbf{x}_i-\boldsymbol{\theta}_j(t)\|}{ \sum_{i=1}^n u^{FCM}_{ij} }$, $j=1,...,m_{ini}$ {\it (see subsection \ref{subsec31})}
	\State \textbf{Set:} $\hat{\eta}=\min\nolimits_{j=1,...,m_{ini}} \eta_j(t)$
\State $m(t)=m_{ini}$
%\algstore{bkbreak}
%\end{algorithmic}
%\end{algorithm}
%
%\begin{algorithm}
%\begin{algorithmic}[1]
%\algrestore{bkbreak}
%\doublespacing
%\Statex 
\Repeat 
\Statex {\Comment {Update $U$}}
\State $u_{ij}(t)=\exp\left(-\frac{\alpha}{\hat{\eta}}\frac{ ||\mathbf{x}_i - \boldsymbol{\theta}_j(t) ||^2 }{\eta_j(t)}\right)$, $i=1,...,N$, $j=1,...,m(t)$
\Statex {\Comment{Update $\Theta$}}
\State $\boldsymbol{\theta}_j(t+1)=\left.{\sum\limits_{i=1}^N u_{ij}(t)\mathbf{x}_i} \middle/ {\sum\limits_{i=1}^N u_{ij}(t)} \right.$, $j=1,...,m(t)$
\Statex {\Comment{Possible cluster elimination}}
\For{$i \leftarrow 1 \textbf{ to } N$}
	\State \textbf{Determine:} $u_{ir}(t)=\max_{j=1,...,m(t)} u_{ij}(t)$
	\State \textbf{Set:} $label(i)=r$
\EndFor
\State $p=0$ {\it \ \ \ //number of removed clusters at iteration $t$}
\For{$j \leftarrow 1 \textbf{ to } m$}
	\If{$j \notin label$}
		\State \textbf{Remove:} $C_j$ (and \textbf{renumber} accordingly $\Theta$ and the columns of $U$)
		\State $p=p+1$
	\EndIf
\EndFor
\State $m(t+1)=m(t)-p$
\Statex {\Comment{Adaptation of $\eta_j$'s}}
\State $\eta_j(t+1)=$ 

$=\frac{1}{n_j(t)}\sum\nolimits_{\mathbf{x}_i:u_{ij}(t)=\max\limits_{r=1,...,m(t+1)} u_{ir}(t)}^{} \|\mathbf{x}_i-\boldsymbol{\mu}_j(t)\|$, $j=1,...,m(t+1)$ 
\State $t=t+1$
\Until{the change in $\boldsymbol{\theta}_j$'s between two successive iterations becomes sufficiently small}\\
\Return {$\Theta$, $U$, $label$}
\end{algorithmic}
\end{algorithm}

\subsection{Rationale of the algorithm}
\label{subsec3}

\begin{figure*}[htb!]
%\captionsetup{width=0.55\textwidth}
\centering
\subfloat[Data set]{\includegraphics[width=0.32\textwidth]{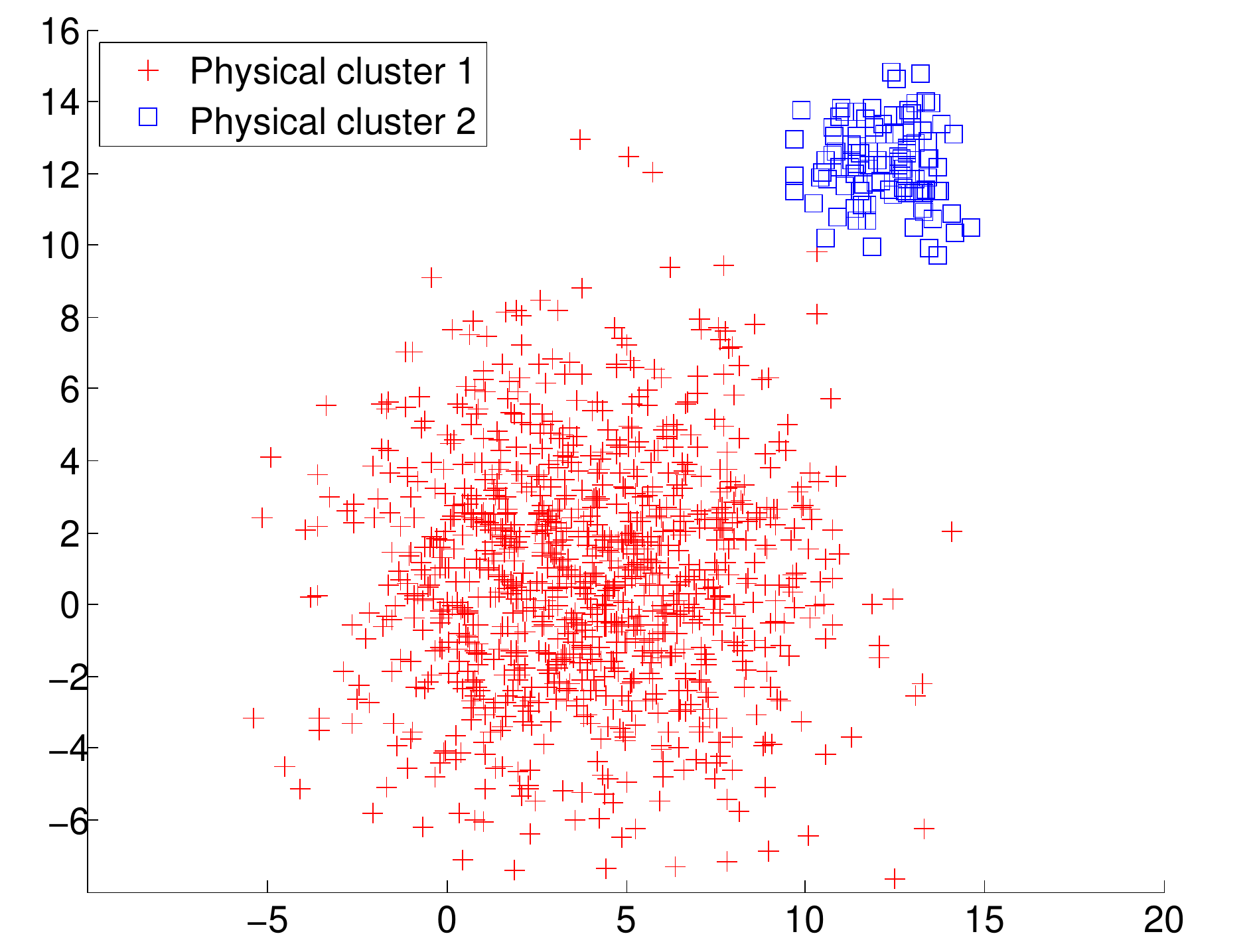}%\vspace{20pt} %47
\label{cex1a}}
\hfil
\centering
\subfloat[Initialization of PCM]{\includegraphics[width=0.32\textwidth]{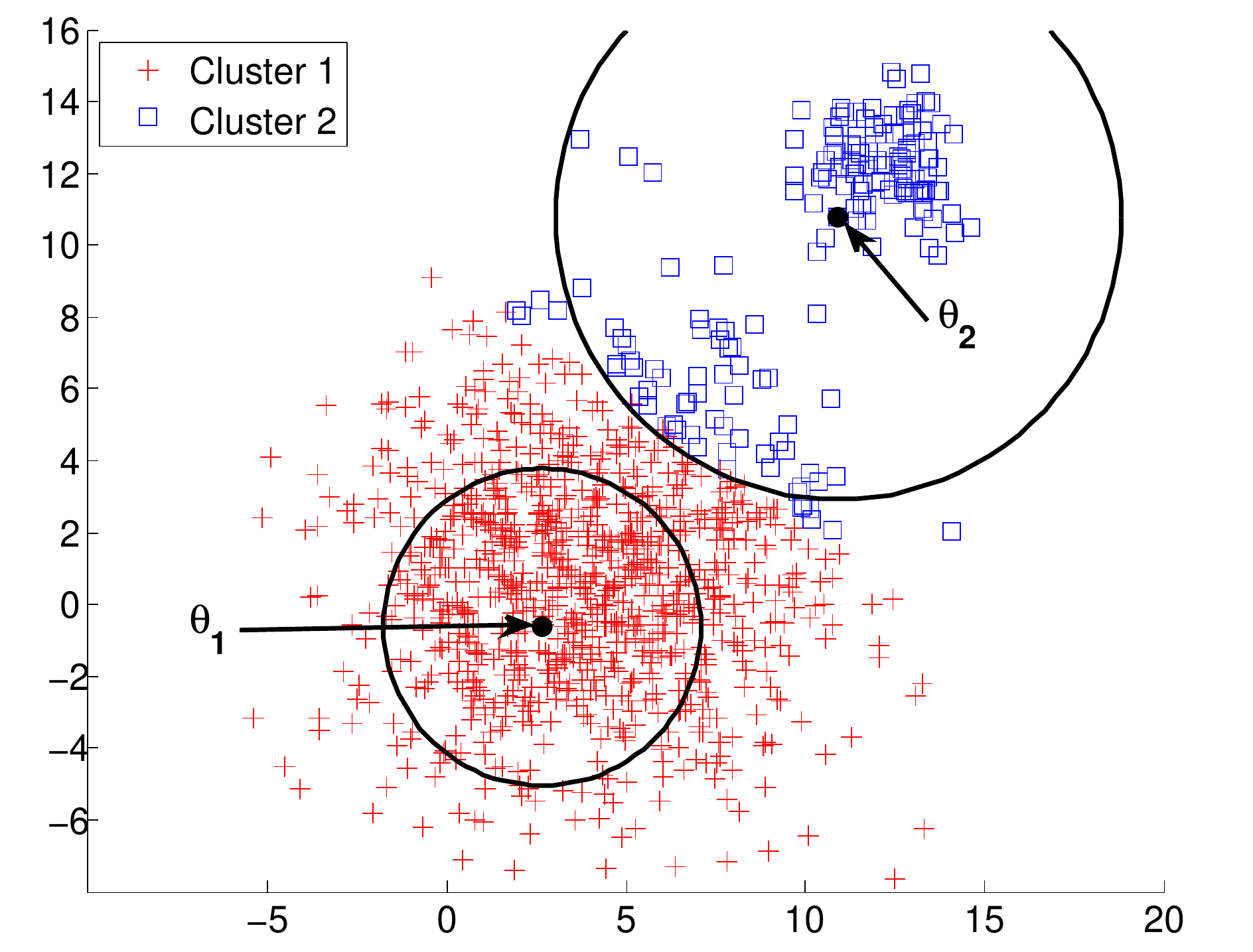} %51
\label{cex1b}}
\hfil
\centering
\subfloat[3rd iteration of PCM]{\includegraphics[width=0.32\textwidth]{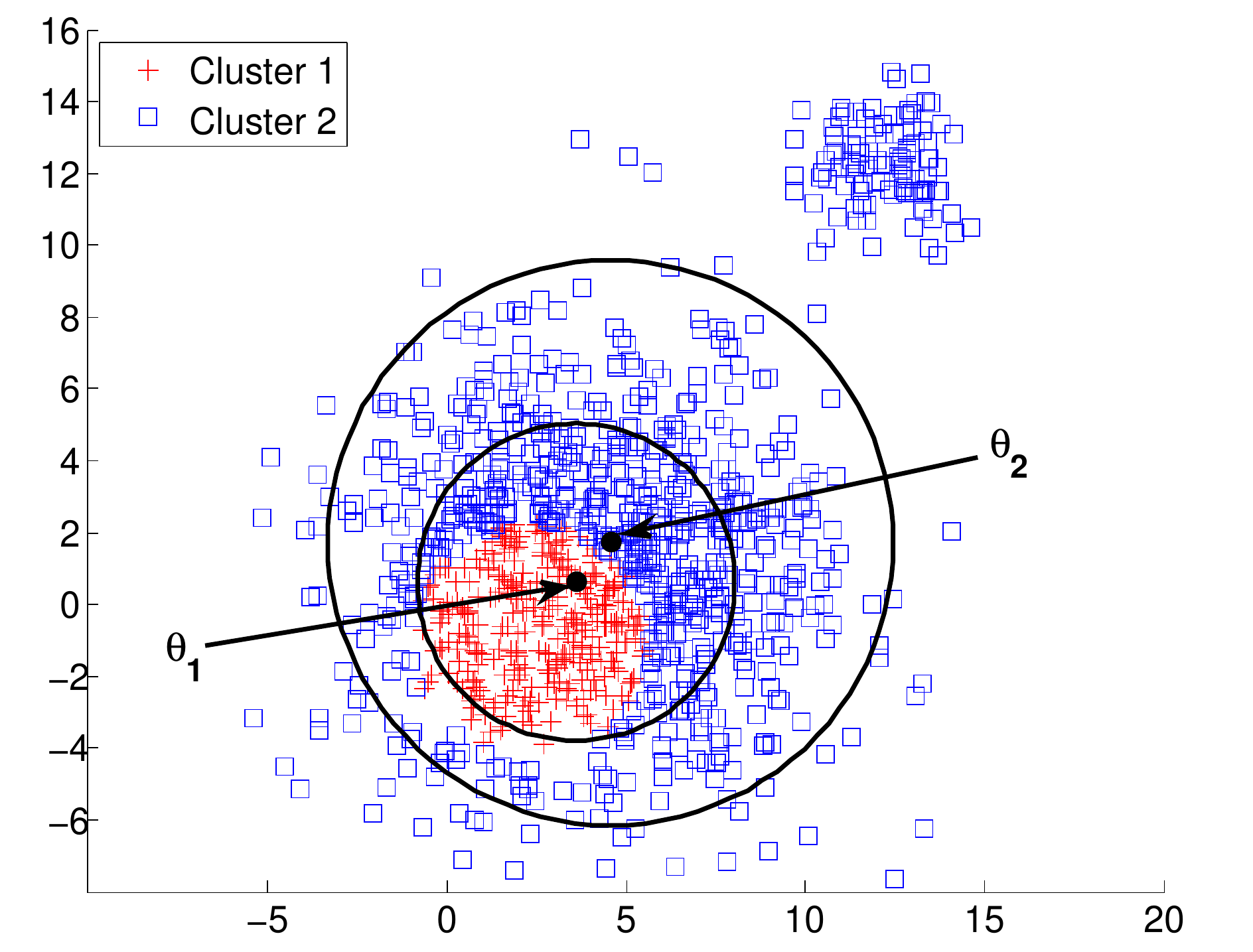} %51
\label{cex1c}}
\hfil
\centering
\subfloat[Initialization of APCM]{\includegraphics[width=0.32\textwidth]{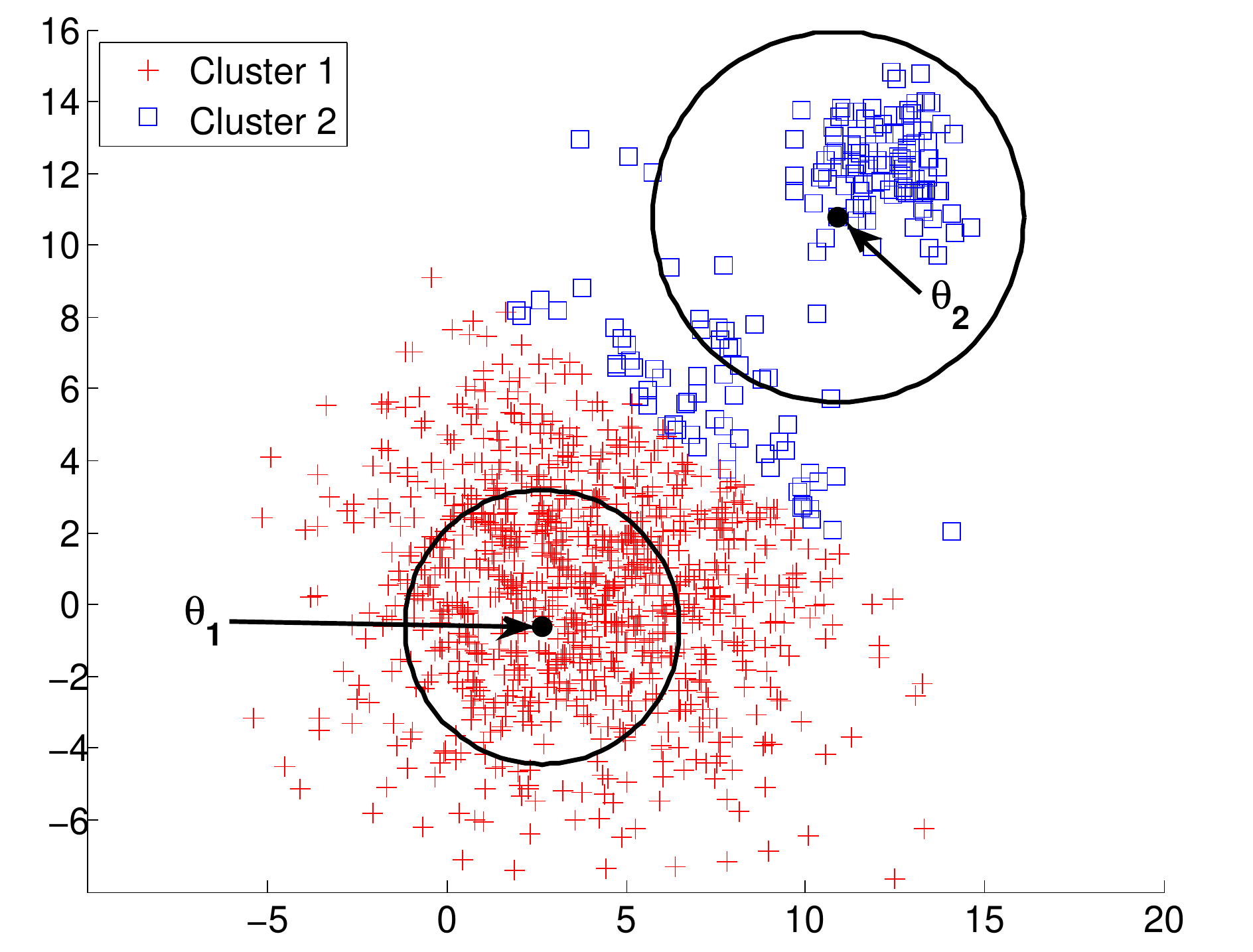} %51
\label{cex1d}}
\hfil
\centering
\subfloat[3rd iteration of APCM]{\includegraphics[width=0.32\textwidth]{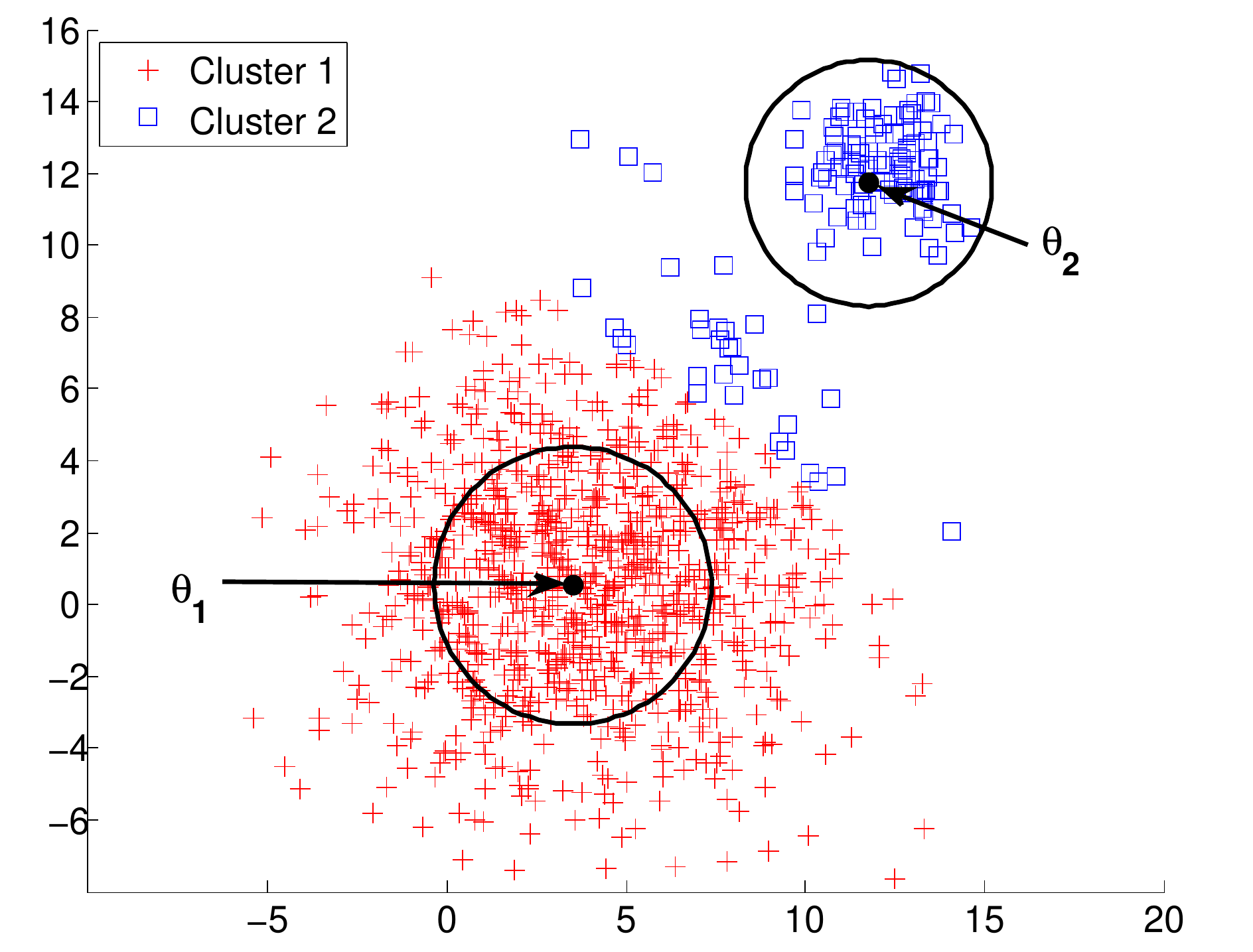} %51
\label{cex1e}}
\hfil
\centering
\subfloat[Final iteration of APCM]{\includegraphics[width=0.32\textwidth]{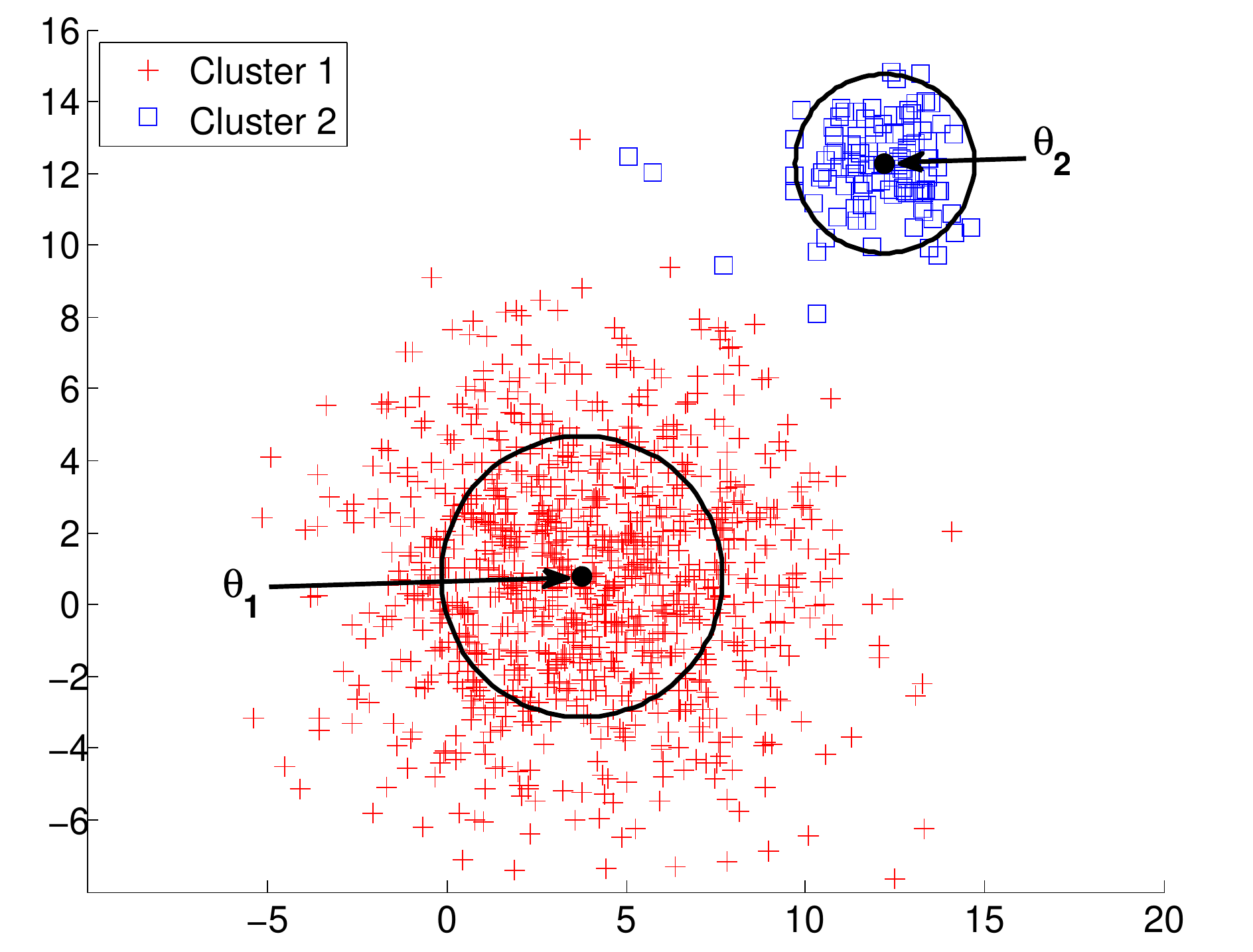} %51
\label{cex1f}}
\hfil
\centering{\caption{An example of a two dimensional data set consisting of two physical clusters that have big difference in their variances and are located very close to each other. (a) The data set, (b) the initial stage of PCM, (c) the 3rd iteration of PCM, (d) the initial stage of APCM, (e) the 3rd iteration of APCM and (f) the final stage of APCM. The circles are centered at $\boldsymbol{\theta}_j$'s and have radius $\sqrt{\gamma_j}$'s.}
\label{cex11}}
\end{figure*}

%In the sequel, we will study the APCM algorithm behavior on the challenging case where there are true clusters that lie very close to each other and have big difference in their variances, in order to explain the rationale of the modification made in the original PCM, leading to the proposed APCM algorithm.

As mentioned in the previous section, the modifications made in the original PCM leading to APCM aim at a) making the algorithm capable to handling stringent clustering situations and b) allowing for cluster elimination. In the following we describe in more detail the hidden mechanisms of APCM that render these two goals feasible.

%%%%%%%%%%%%%%%%%%%%%%%%%%%%%%%%%%%%%%%%%%%%%%%%%%%%%%%%%%%%%%%%%%%%%%%%
%Moreover, we will explain how the proposed adaptation of the parameters $\eta_j$'s leads to the gradual reduction of $m_{ini}$ towards the true number of clusters, $m$, through the cluster elimination procedure.
%%%%%%%%%%%%%%%%%%%%%%%%%%%%%%%%%%%%%%%%%%%%%%%%%%%%%%%%%%%%%%%%%%%%%%%%

First, we consider the case where we have two physical clusters of very different variances that are located very close to each other (Fig.~\ref{cex1a}). This is a difficult clustering problem, in which most state-of-the-art clustering techniques fail. We assume that after initialization with FCM, PCM has two representatives in the areas of the physical clusters, as shown in Fig.~\ref{cex1b}, with $\boldsymbol{\theta}_1$ lying in the high variance physical cluster and $\boldsymbol{\theta}_2$ in the low variance physical cluster. Then, from eq.~\eqref{Ketaj} and due to the proximity of the two physical clusters, it turns out that $\gamma_2$ will be much larger than the actual variance of physical cluster 2. This is so because, besides the points of physical cluster 2, the numerous, yet more distant, points of physical cluster 1, will contribute to the computation of $\gamma_2$ from eq.~\eqref{Ketaj}. 
%
% Assume that FCM runs for such an overestimated number of clusters, so that to give at least one initial representative somewhere in the small variance class (class 2) and suppose that the PCM algorithm has two representatives at a specific ($t$th) iteration (see fig.\ref{cex1b}). Based on eq.~\ref{Ketaj}, it turns out that the representative $\boldsymbol{\theta}_j$ that lie in the region of the small variance class ($\boldsymbol{\theta}_2$) will have much larger $\gamma_j$ ($\gamma_2$) compared to the actual variance of the class (class 2), as it will have positive $u_{ij}$'s ($u_{i2}$'s) for too many distant data points of the nearby large variance class (the data points of class 1).
%
This means that the representative of the small variance cluster ($\boldsymbol{\theta}_2$) is affected by the data points of its nearby cluster ($C_1$), according to eqs.~\eqref{uij},~\eqref{theta}. As a result, PCM is likely to end up with all representatives converging erroneously in the center of the large variance physical cluster 1 (Fig.~\ref{cex1c}).

This issue of PCM is alleviated in APCM, by taking care for each representative to stay in the region of the physical cluster where it was first placed. To this end, APCM reduces (compared to PCM) the range of influence arround each $\boldsymbol{\theta}_j$ that has $\gamma_j$ larger than the variance of the smallest physical cluster formed in the data set. In this way, the probability of the movement of a representative which is initialized in the region of a specific physical cluster with a given variance towards the center of a nearby physical cluster with a larger variance, is reduced. In particular, the larger (smaller) the $\gamma_j$ than the variance of the smallest physical cluster is, the more it is reduced (enhanced). On the other hand, a $\gamma_j$ that is equal to the variance of the smallest physical cluster is not affected at all. Focusing on a given iteration (dropping the index $t$), this is achieved in APCM by defining $\gamma_j$ as in eq.~\eqref{gamma}. This definition results from the $\gamma_j$'s as defined in the original PCM via the following transformations.

\begin{align*} 
\gamma_j^{PCM}=\frac{\sum_{i=1}^N u^{FCM}_{ij} \|\mathbf{x}_i-\boldsymbol{\theta}_j\|^2}{\sum_{i=1}^N u^{FCM}_{ij}}
\text{ $\overset{\circled{\scriptsize{1}}}{\leadsto}$ } \\ \gamma_j'=\frac{\sum\nolimits_{\mathbf{x}_i:u_{ij}=\max_{r=1,...,m} u_{ir}} \|\mathbf{x}_i-\boldsymbol{\mu}_j\|^2}{n_j}
\text{ $\overset{\circled{\scriptsize{2}}}{\leadsto}$ }
\end{align*}

\begin{equation}
\label{transformations2}
%\text{ $\overset{\circled{\scriptsize{2}}}{\leadsto}$ }
\eta_j^2=\left(\frac{\sum\nolimits_{\mathbf{x}_i:u_{ij}=\max_{r=1,...,m} u_{ir}}^{} \|\mathbf{x}_i-\boldsymbol{\mu}_j\|}{n_j}\right)^2 \text{ $\overset{\circled{\scriptsize{3}}}{\leadsto}$ } \eta_j\frac{\hat{\eta}}{\alpha}
\end{equation}

\begin{figure}[htpb!]
\centering
{\includegraphics[width=0.49\textwidth]{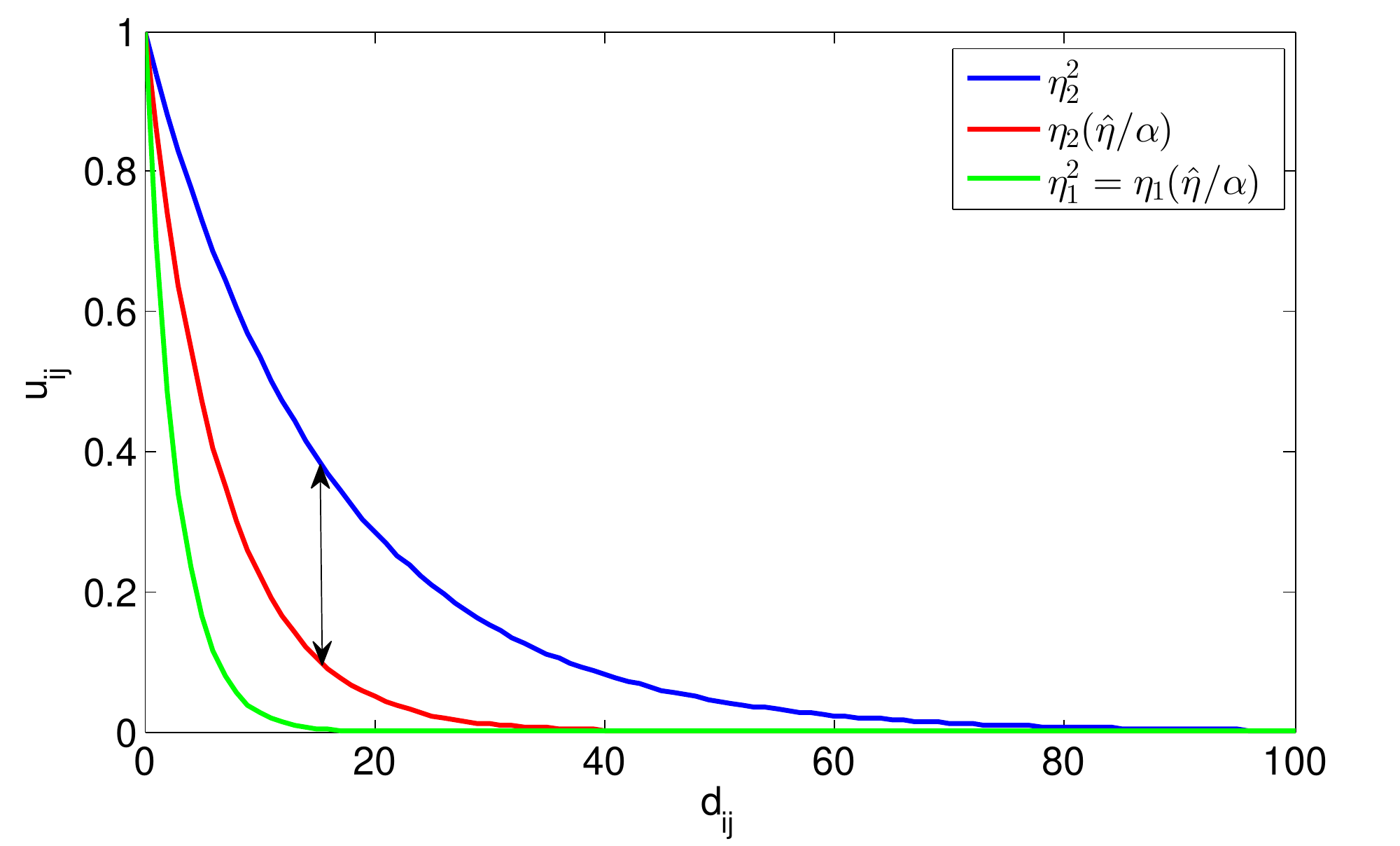}}
\centering{\caption{The degree of compatibility $u_{ij}$ with respect to distance $d_{ij}$ ($\eta_2>\eta_1=\hat{\eta}/\alpha$).}
\label{cdiagramu}}
\end{figure}

Under transformation \circled{\small{1}}, $\gamma_j^{PCM}$ is transformed to $\gamma_j'$, where (a) only the $\mathbf{x}_i$'s that are most compatible with $\boldsymbol{\theta}_j$ are taken into account and (b) $\boldsymbol{\theta}_j$ is replaced by $\boldsymbol{\mu}_j$. The adoption of the above hard computation of $\gamma_j'$'s is necessary for the cluster elimination procedure, as it will be further explained in the sequel. Transformation \circled{\small{2}} leads $\gamma_j'$ to $\eta_j^2$, which carries the same ``quality of information" with its predecessor and moreover, $\eta_j^2$ is upper bounded by $\gamma_j'$, $j=1,\ldots,m$ (see Proposition~\ref{prop1} in Appendix A). This intermediate step on the one hand reduces the influence of clusters arround their representatives while, on the other hand, is a prerequisite for transformation \circled{\small{3}}. Assuming that $\alpha$ is chosen so that the quantity $\hat{\eta}/\alpha$ equals to the mean alsolute deviation of the smallest physical cluster formed in the data set, then for each $\eta_j\geq\hat{\eta}/\alpha$ ($\eta_j\leq\hat{\eta}/\alpha$), we have that $\eta_j^2\geq\eta_j\left(\hat{\eta}/\alpha\right)$ ($\eta_j^2\leq\eta_j\left(\hat{\eta}/\alpha\right)$). That is, by substituting $\eta_j^2$ with $\eta_j(\hat{\eta}/\alpha)$, the greater (smaller) the $\eta_j$ of a cluster $C_j$ than $\hat{\eta}/\alpha$, the more the range of influence arround its $\boldsymbol{\theta}_j$ is reduced (enhanced) (see Fig.~\ref{cdiagramu} and Figs.~\ref{cex1d},~\ref{cex1e}). This justifies our choice for the $\gamma_j$'s given in eq.~\eqref{gamma}. % and \eqref{adapteta}.

%%%%%%%%%%%%%%%%%%%%%%%%%%%%%%%%%%%%%%%%%%%%%%%%%%%%%%%%%%%%%%%
\begin{figure*}[htb!]
%\captionsetup{width=0.55\textwidth}
\centering
\subfloat[1st iteration of APCM]{\includegraphics[width=0.32\textwidth]{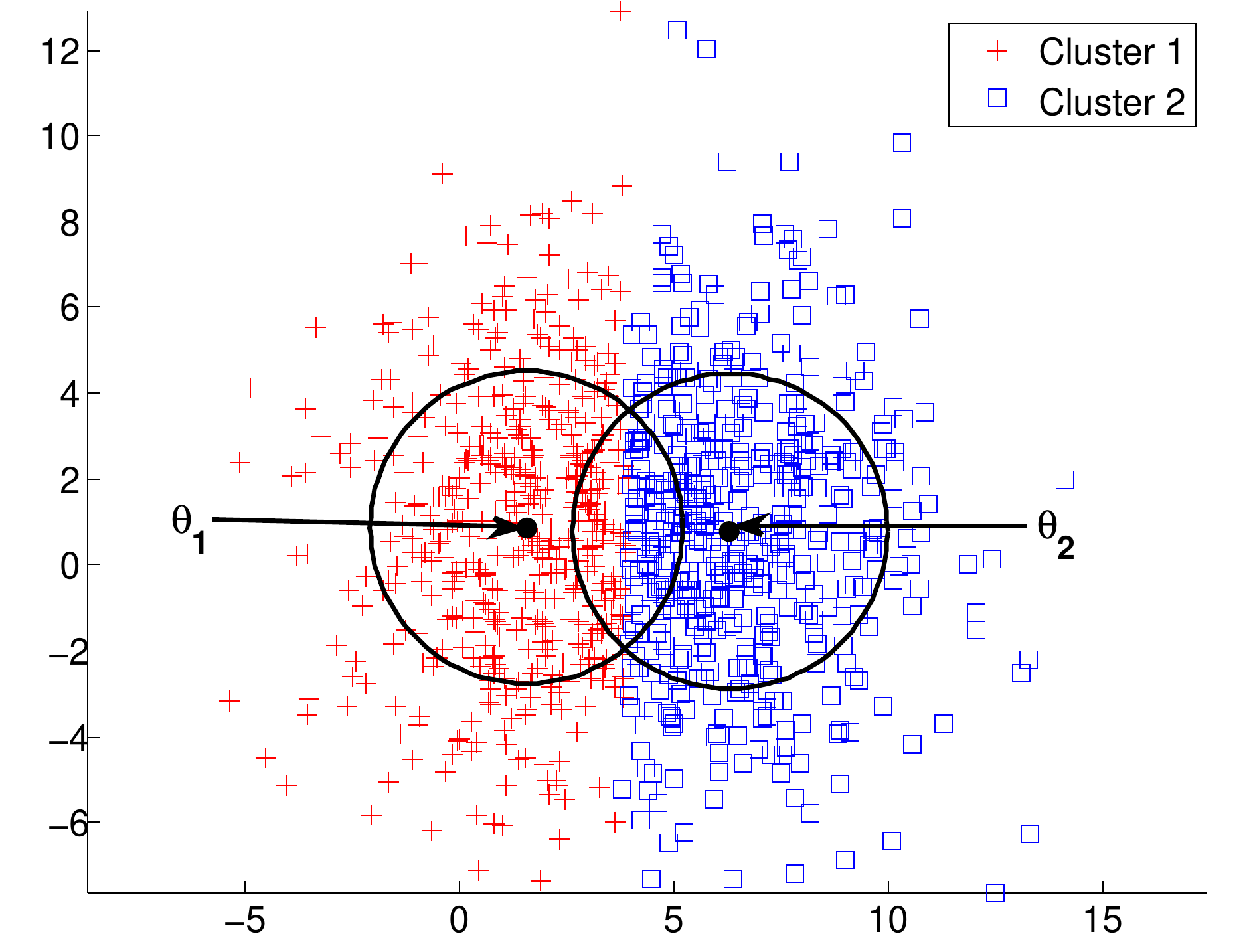}%\vspace{20pt} %47
\label{cex2a}}
\hfil
\centering
\subfloat[4th iteration of APCM]{\includegraphics[width=0.32\textwidth]{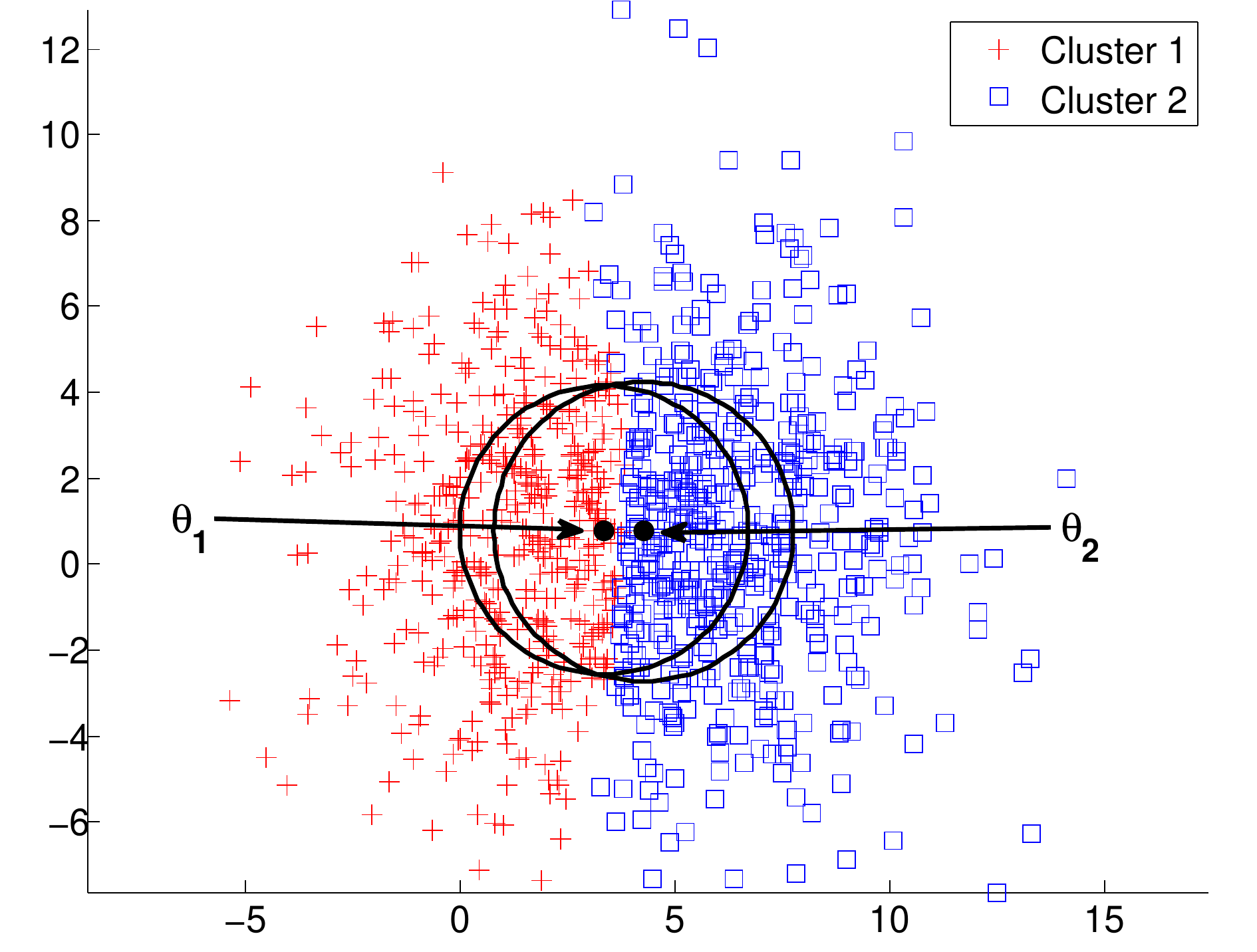} %51
\label{cex2b}}
\centering
\subfloat[5th iteration of PCM]{\includegraphics[width=0.32\textwidth]{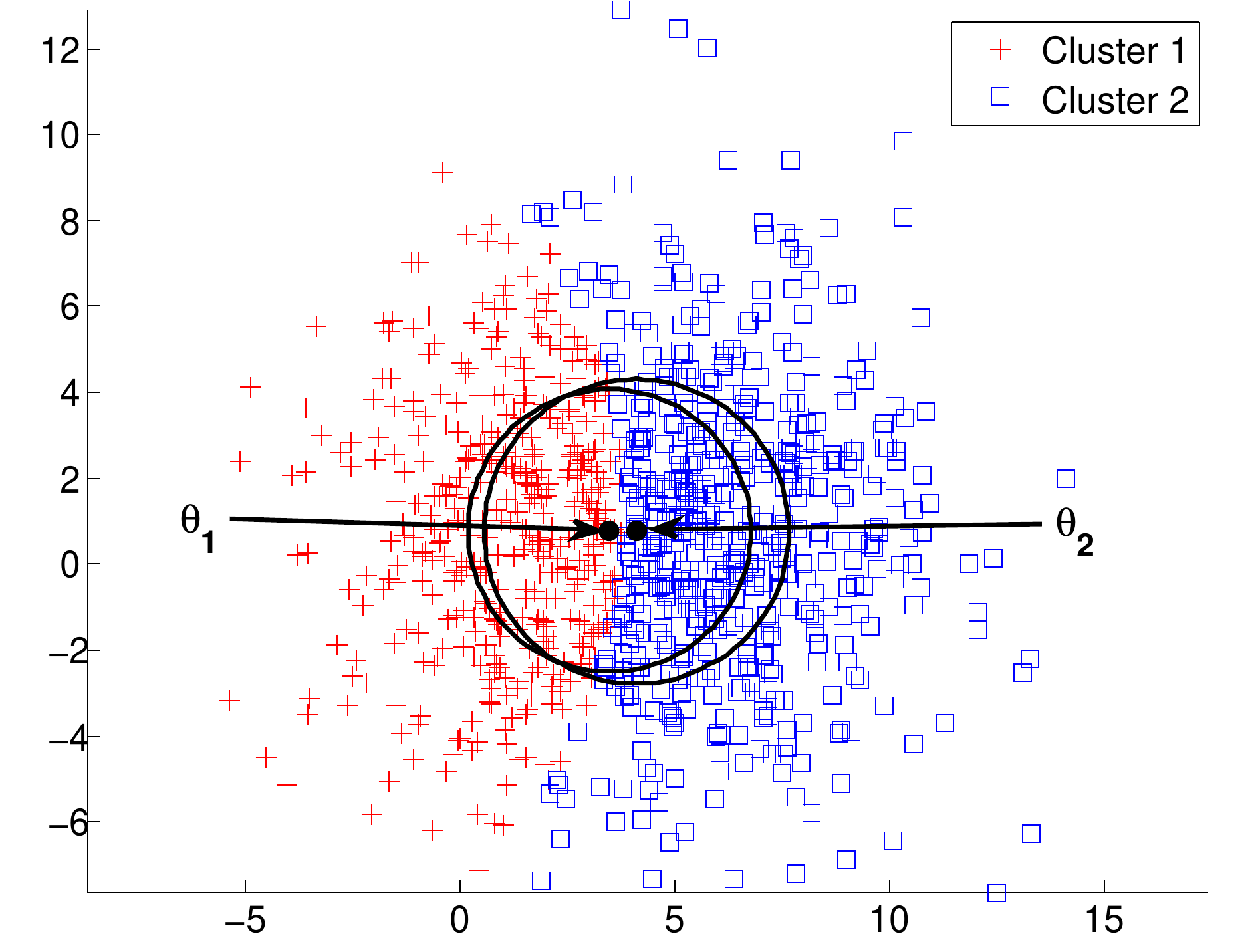} %51
\label{cex2c}}
\hfil
\centering
\subfloat[6th iteration of APCM]{\includegraphics[width=0.32\textwidth]{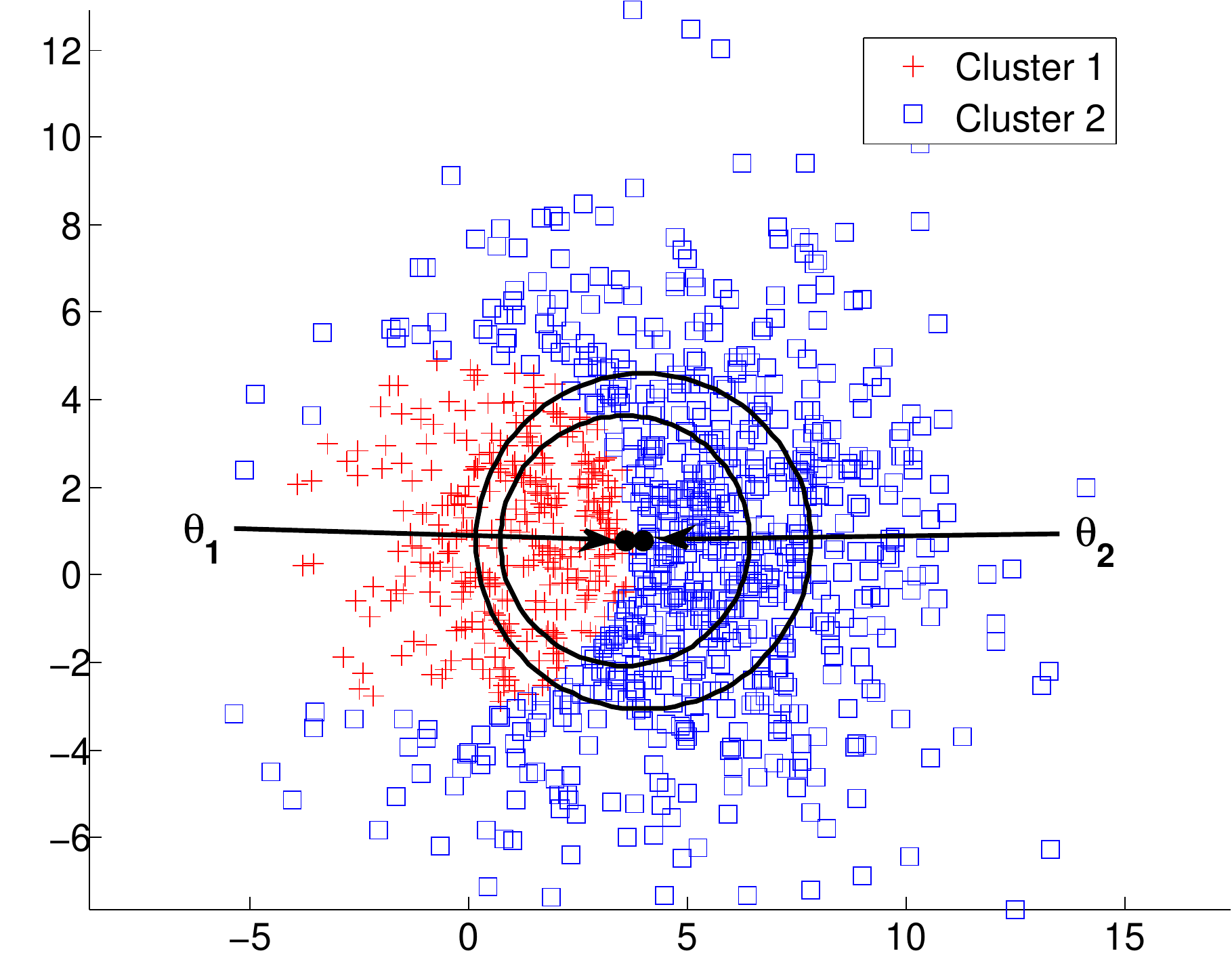} %51
\label{cex2d}}
\centering
\subfloat[7th iteration of APCM]{\includegraphics[width=0.32\textwidth]{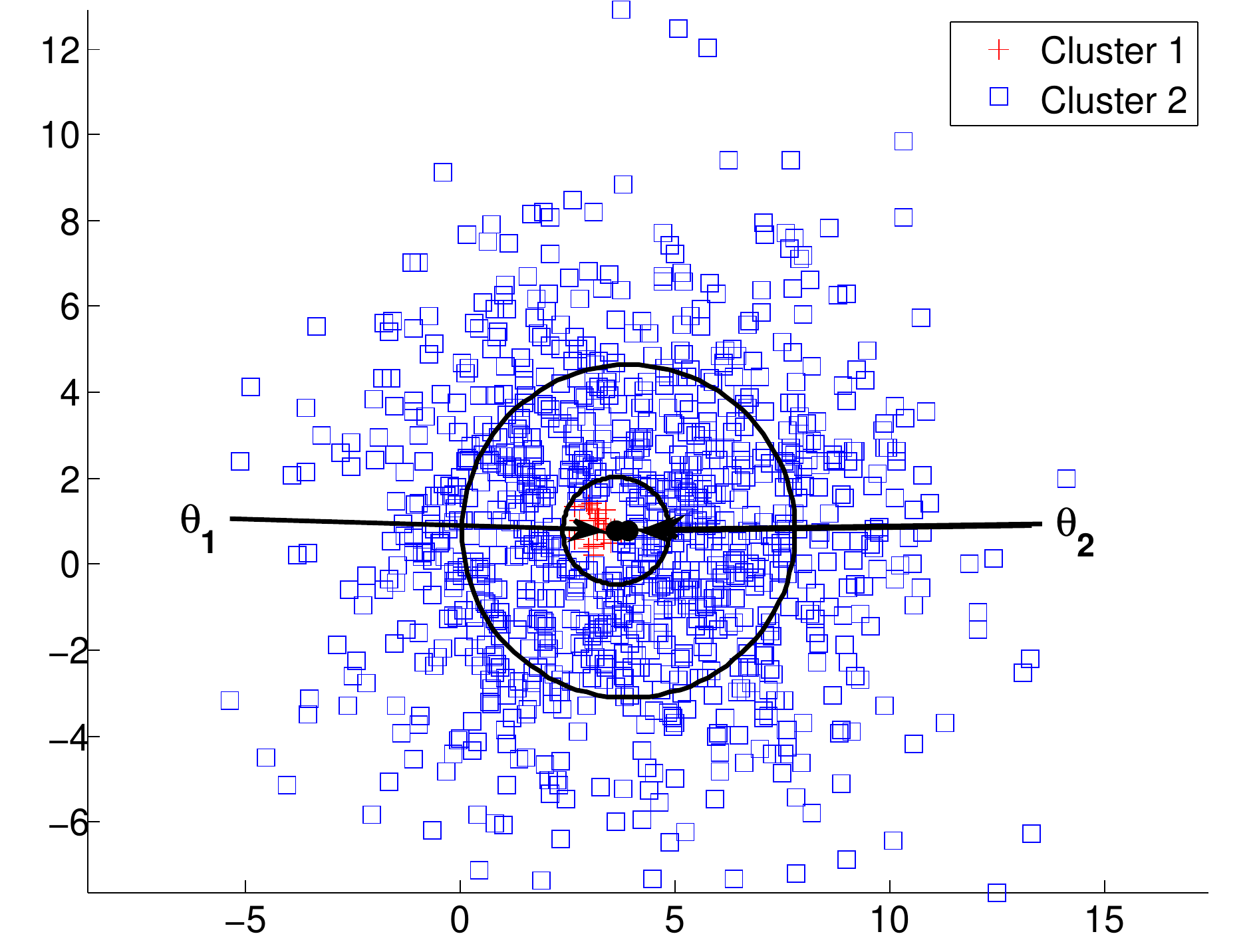} %51
\label{cex2e}}
\hfil
\centering
\subfloat[Final iteration of APCM]{\includegraphics[width=0.32\textwidth]{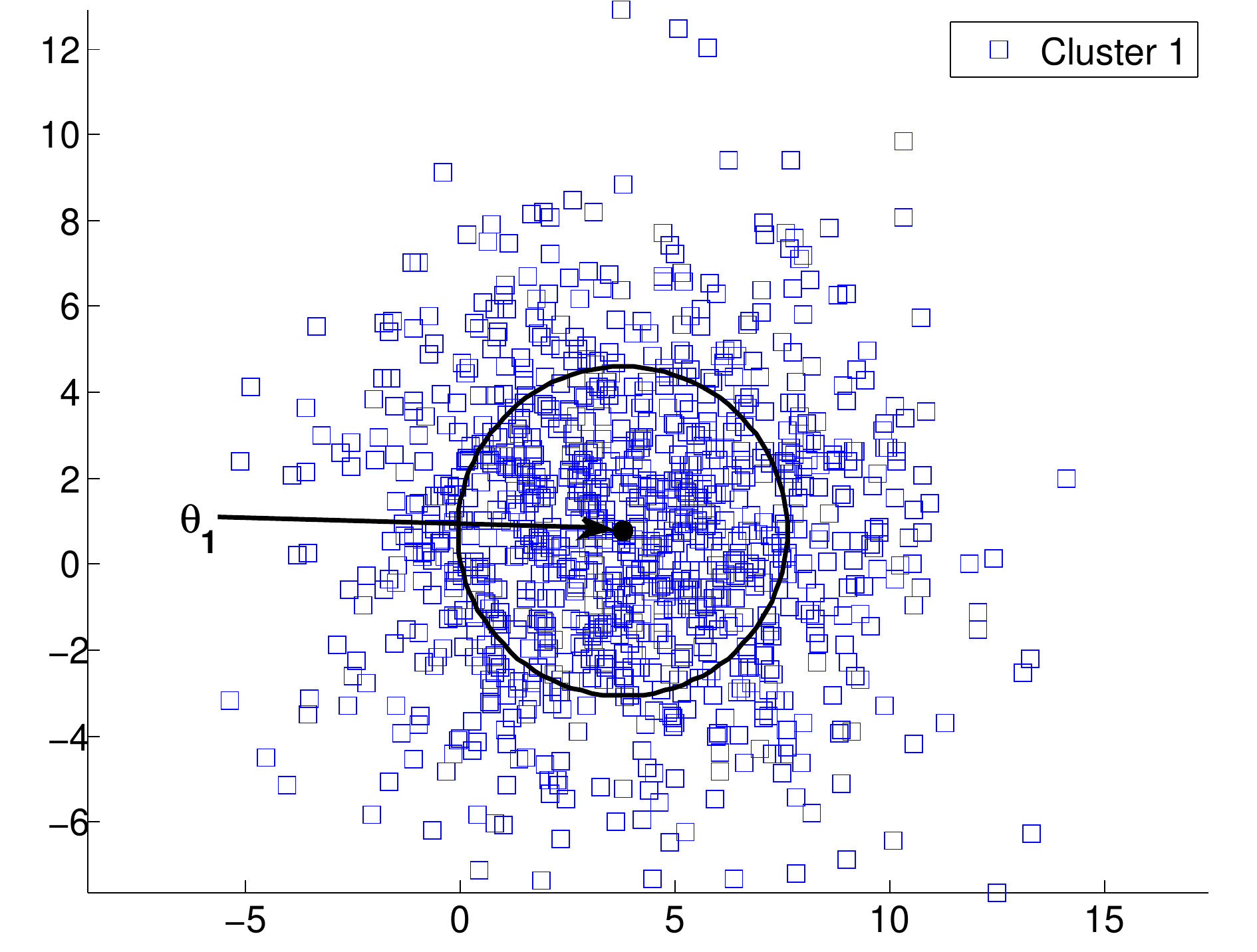} %51
\label{cex2f}}
\hfil
\centering{\caption{A two dimensional data set consisting of a single physical cluster. APCM is intialized with two representatives and the cluster elimination procedure is illustrated at several time instances.}
\label{cex2}}
\end{figure*}

In the sequel, we will focus on the cluster elimination property of the APCM algorithm. To this end, consider the case where a single physical cluster is formed by the data points where $k(>1)$ representatives $\boldsymbol{\theta}_j$'s, $j=1,\dots,k$, are initialized within it (see Fig.~\ref{cex2a} for $k=2$). As eq.~\eqref{theta} suggests, each representative will move towards the center of the dense region (see also propositions \ref{prop3} and \ref{prop4} in Appendix B for a more rigorous justification). As $\boldsymbol{\theta}_j$'s move towards the center of the region, they are getting closer to each other. At a specific iteration $t_0$ ($t_0=6$ in Fig.~\ref{cex2d}) where, say $\gamma_r(t_0)=\max_{j=1,...,k} \gamma_j(t_0)$, the hypersphere centered at $\boldsymbol{\theta}_r(t_0)$ and having radius $\sqrt{\gamma_r(t_0)}$ will enclose all the hyperspheres associated with the other representatives. From this point on, the region of influence ($\gamma_j$) of all the clusters except $C_r$ shrinks to 0 as is shown in Fig.~\ref{cex2}, due to their definition (see eqs.~\eqref{gamma},~\eqref{adapteta}) (a theoretical justification for the two representatives case is given in proposition \ref{prop5} in Appentix B).

%Assume that $\eta_r=\max_{j=1,...,k} \eta_j$ and all representatives $\boldsymbol{\theta}_j$'s, $j=1,\dots,k$ will have been sufficiently close to each other at a given iteration, such that $\eta_r\geq\eta_j+\|\boldsymbol{\theta}_r-\boldsymbol{\theta}_j\|$, $j=1,\ldots,k$, i.e., the hypersphere of the maximum among $\eta_j$'s ($\eta_r$) encloses the hyperspheres of all others $\eta_j$'s, $j=1,\ldots,k$, $j\neq r$ (see fig.~\ref{cex2d}). According to Proposition~\ref{prop2} in the Appendix, $\eta_j(t+1)\leq\eta_j(t)$, $j=1,\ldots,k$, $j\neq r$, that is, at each next iteration fewer and fewer points will be most compatible with clusters $C_j$'s, $j\neq r$ as the algorithm evolves, leading gradually to their final elimination and allowing $\boldsymbol{\theta}_r$ to be the only representative of this dense region (see fig.~\ref{cex2f}).

\subsection{Selection of parameter $\alpha$}
\label{subsec4}

As it was mentioned previously, $\alpha$ is a user-defined parameter that has to be fine-tuned, so that $\hat{\eta}/\alpha$ becomes equal to the mean absolute deviation of the smallest physical cluster. As it is expected, larger values of $m_{ini}$ lead to smaller initial $\eta_j$'s and thus a smaller $\hat{\eta}$. As a consequence, there exists a trade-off between $m_{ini}$ and parameter $\alpha$: large (small) values of $m_{ini}$ require small (large) values of $\alpha$, so that the ratio $\hat{\eta}/\alpha$ approximates the mean absolute deviation of the smallest physical cluster. Note that although the latter quantity is fixed for a given data set, it is unknown in practice.

In the sequel, we discuss how different choices of $\alpha$ affect the behavior of APCM, focusing on the limiting cases $\alpha\rightarrow 0$ and $\alpha\rightarrow +\infty$. Specifically, we consider a single representative $\boldsymbol{\theta}_j$ and we concentrate on its corresponding ``subcost" function\footnote{We write $J_j(\boldsymbol{\theta}_j)$ to explicitly denote the dependence of $J_j$ on $\boldsymbol{\theta}_j$.} $$J_j(\boldsymbol{\theta}_j)=\sum\nolimits_{i=1}^N  u_{ij}\|\mathbf{x}_i-\boldsymbol{\theta}_j\|^2 + \eta_j\frac{\hat{\eta}}{\alpha} \sum\nolimits_{i=1}^N (u_{ij}\ln u_{ij}-u_{ij}),$$ where we assume for the time being that $\eta_j$ is constant, while $u_{ij}$ is given as $u_{ij}=\exp\left(-\frac{\alpha}{\hat{\eta}}\frac{\|\mathbf{x}_i-\boldsymbol{\theta}_j\|^2}{\eta_j}\right)$ (see eq.~\eqref{uij2}). Utilizing the last equation and after some algebra, $J_j(\boldsymbol{\theta}_j)$ can be written as
\begin{equation}
J_j(\boldsymbol{\theta}_j)=-\eta_j\frac{\hat{\eta}}{\alpha}\sum\limits_{i=1}^N \exp\left(-\frac{\alpha}{\hat{\eta}}\frac{\|\mathbf{x}_i-\boldsymbol{\theta}_j\|^2}{\eta_j}\right)
\label{subsecD1}
\end{equation}
Taking the gradient of $J_j$ with respect to $\boldsymbol{\theta}_j$, we have:
\begin{equation}
\frac{\partial J_j(\boldsymbol{\theta}_j)}{\partial \boldsymbol{\theta}_j}=2\sum\limits_{i=1}^N \exp\left(-\frac{\alpha}{\hat{\eta}}\frac{\|\mathbf{x}_i-\boldsymbol{\theta}_j\|^2}{\eta_j}\right) (\mathbf{x}_i-\boldsymbol{\theta}_j)
\label{subsecD2}
\end{equation}

For $\alpha\rightarrow 0$, we have that $\exp\left(-\frac{\alpha}{\hat{\eta}}\frac{\|\mathbf{x}_i-\boldsymbol{\theta}_j\|^2}{\eta_j}\right)\rightarrow 1$. Thus, $\frac{\partial J_j(\boldsymbol{\theta}_j)}{\partial \boldsymbol{\theta}_j}$ tends to $2\sum\nolimits_{i=1}^N (\mathbf{x}_i-\boldsymbol{\theta}_j)$ and equating the latter to zero, we end up with $\boldsymbol{\theta}_j=\frac{1}{N}\sum\nolimits_{i=1}^N \mathbf{x}_i$. Thus, in this case there exists a single minimum; the mean of the data set.

For $\alpha\rightarrow +\infty$, it is clear from eq.~\eqref{subsecD1} that, identically, $J_j(\boldsymbol{\theta}_j)=0$. Thus, all possible choices for $\boldsymbol{\theta}_j$ are (trivially) local minima of $J_j(\boldsymbol{\theta}_j)$. As $\alpha$ gradually increases from 0, the number of minima of $J_j(\boldsymbol{\theta}_j)$ increases and it is expected that, for a specific range of $\alpha$ values, the minima of $J_j(\boldsymbol{\theta}_j)$ will correspond to the centers of the physical clusters. Of course, this cease to hold as we move outside this range towards $+\infty$.

The above are illustrated via a simple clustering example. Specifically, we consider an one-dimensional data set consisting of two Gaussian clusters with 50 points each, shown on the $x$-axis in Fig.~\ref{original_1d}. The centers of the clusters are at locations 28 and 67 and their variances are 100 and 121, respectively. We consider two cases: in the first, the number of initial representatives is $m_{ini}=3$ while in the second, $m_{ini}=10$. We run first the FCM algorithm for each case and we obtain the resulting $u_{ij}^{FCM}$'s and $\boldsymbol{\theta}_j$'s, from which the initial $\gamma_j$'s are computed using eqs.~\eqref{initetaj} and \eqref{gamma}. Note that for $m_{ini}=3$ and $m_{ini}=10$, the corresponding $\hat{\eta}$ values are 7.0094 and 2.3213. 
%Then, for several values of $\alpha$, we execute APCM.

%Due to the independence among the $J_j$'s in eq.~(\ref{Jpcm}), the minimization of $J$ with respect to $\boldsymbol{\theta}_j$'s results from the minimization of each individual $J_j$ with respect to the corresponding $\boldsymbol{\theta}_j$. 
In order to investigate further the relation between $\alpha$ and $m_{ini}$, we focus on $J_j$ that corresponds to the minimum initial $\gamma_j$ and we drop time dependence. Thus, in this case, $\gamma_j$ is fixed to $\hat{\eta}^2/\alpha$. The ``subcost" function $J_j(\theta_j)=\sum\nolimits_{i=1}^N  u_{ij}\|x_i-\theta_j\|^2 + \frac{\hat{\eta}^2}{\alpha} \sum\nolimits_{i=1}^N (u_{ij}\ln u_{ij}-u_{ij})$ is plotted with respect to $\theta_j$, for various values of $\alpha$. We consider first $m_{ini}=3$, i.e., $m_{ini}$ is very close to the number of actual clusters ($m=2$). Thus, in this case, FCM algorithm is more likely to give good initial estimations for $\eta_j$'s (through eq.~\eqref{initetaj}), i.e. the minimum initial $\eta_j(\equiv\hat{\eta})$ approximates the mean absolute deviation of the smallest physical cluster. We consider the following indicative cases:
\begin{itemize}
  \item $\alpha=0.05$: In this case the ratio $\hat{\eta}/\alpha$ becomes much larger than the mean absolute deviation of the smallest physical cluster, leading all data points to have significant $u_{ij}$'s for all representatives (through eq.~\eqref{uij2}). This justifies the plot of Fig.~\ref{costfun_m3a005}, where $J_j$ exhibits just a single valley centered at the mean of the data set. Clearly, the minimization of $J_j$ will lead $\theta_j$ to this position, which means that in this case the algorithm will fail to detect any of the two true clusters. 
  \item $\alpha=1\text{ or }2$: In this case the ratio $\hat{\eta}/\alpha$ approximates the mean absolute deviation of the smallest physical cluster and as we can see in Figs.~\ref{costfun_m3a1},~\ref{costfun_m3a2}, two well formed valleys are centered at the means of the two natural clusters (although a bit disturbed in the $\alpha=2$ case). Thus, minimization of $J_j$ will lead $\theta_j$ to the center of a true cluster. 
\end{itemize}
In conclusion, when $m_{ini}$ is close to actual $m$ and provided that at least one representative is placed at each dense region, the minimum $\eta_j$ value ($\hat{\eta}$) that is derived using the FCM algorithm (eq.~\eqref{initetaj}) is {\it a good estimate} of the mean absolute deviation of the smallest physical cluster, thus values of $\alpha$ around $1$ allow the algorithm to work properly.

%We plot $J_j$ when $\alpha=0.05$ (Fig.~\ref{costfun_m3a005}), $\alpha=1$ (Fig.~\ref{costfun_m3a1}) and $\alpha=2$ (Fig.~\ref{costfun_m3a2}). In Fig.~\ref{costfun_m3a005}, we have just a single valley centered at the mean of the data set. Minimization of $J_j$ will lead $\boldsymbol{\theta}_j$ to this position, which means that in this case the algorithm will fail to detect any of the two true clusters. This case corresponds to the situation where all data points have significant degree of compatibility $u_{ij}$ with all representatives, leading them gradually to the center of the data set. In Fig.~\ref{costfun_m3a1}, we see two well formed valleys centered at the centers of the two natural clusters. Minimization of $J_j$ will lead $\boldsymbol{\theta}_j$ to the center of a true cluster. Thus, for $\alpha=1$, APCM works properly. Observe that in this case APCM will drive $\boldsymbol{\theta}_j$ to the center of an actual cluster, without requiring a strictly {\it good} initialization for it. A similar situation is shown in Fig.~\ref{costfun_m3a2} for $\alpha=2$, although the form of the valleys is a bit disturbed. In conclusion, when $m_{ini}$ is close to actual $m$, in most cases of practice, the minimum $\eta_j$ value that is derived using the FCM algorithm (eq.~\eqref{initetaj}) is {\it a good estimation} of the mean absolute deviation of the minimum physical cluster, thus values of $\alpha$ around 1 allow the algorithm to work properly.

In case where $m_{ini}=10$ (that is $m_{ini}\gg m$) the situation changes. In this case, all initial $\eta_j$'s and thus $\hat{\eta}$ are much smaller than the mean absolute deviation of the smallest physical cluster. We consider the following indicative cases:
\begin{itemize}
  \item $\alpha=0.05$: In this case the ratio $\hat{\eta}/\alpha$ approximates the mean absolute deviation of the smallest physical cluster. Thus, two well formed valleys are centered at the means of the two natural clusters (see in Fig.~\ref{costfun_m10a005}) and the APCM will lead a $\theta_j$ to the center of a true cluster. 
  \item $\alpha=1\text{ or }2$: In this case $J_j$ exhibits many local minima (see Figs.~\ref{costfun_m10a1},~\ref{costfun_m10a2}), as the ratio $\hat{\eta}/\alpha$ is significantly smaller than the mean absolute deviation of the smallest physical cluster, leading all data points to have negligible $u_{ij}$'s values, even with $\theta_j$'s that are placed very close to them (through eq.~\eqref{uij2}). As a consequence, $J_j$ exhibits several local minima that do not correspond to any of the two true clusters and APCM is most likely to end up with clusters that do not correspond to the underlying data set structure.
\end{itemize}

%In this case, two well-formed valleys appear for the very small value $\alpha=0.05$ (Fig.~\ref{costfun_m10a005}), whereas for values of $\alpha$ around 1, $J_j(\boldsymbol{\theta}_j)$ exhibits many local minima. As explained before, this is due to the fact that very few data points around $\boldsymbol{\theta}_j$ are allowed to affect significantly the determination of its new position, in cases where $m_{ini}\gg m$. Thus, $\boldsymbol{\theta}_j$ is likely to get stuck to local minima that do not correspond to any of the two true clusters (it is worth noting that similar arguments hold also for all the other $J_j$ functions). In this case, APCM is most likely to end up with more than 2 clusters, which does not correspond to the underlying structure of the data set.

\begin{figure*}[htb!]
\centering
\subfloat[The data set]{\hspace{55pt}\includegraphics[width=0.50\textwidth]{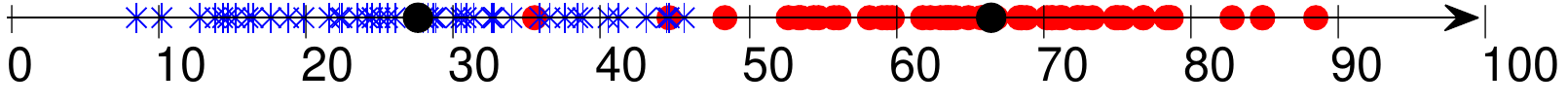}\vspace{20pt} %47
\label{original_1d}\hspace{40pt}}
\hfil
\centering
\subfloat[$m_{ini}=3$, $\alpha=0.05$]{\includegraphics[width=0.32\textwidth]{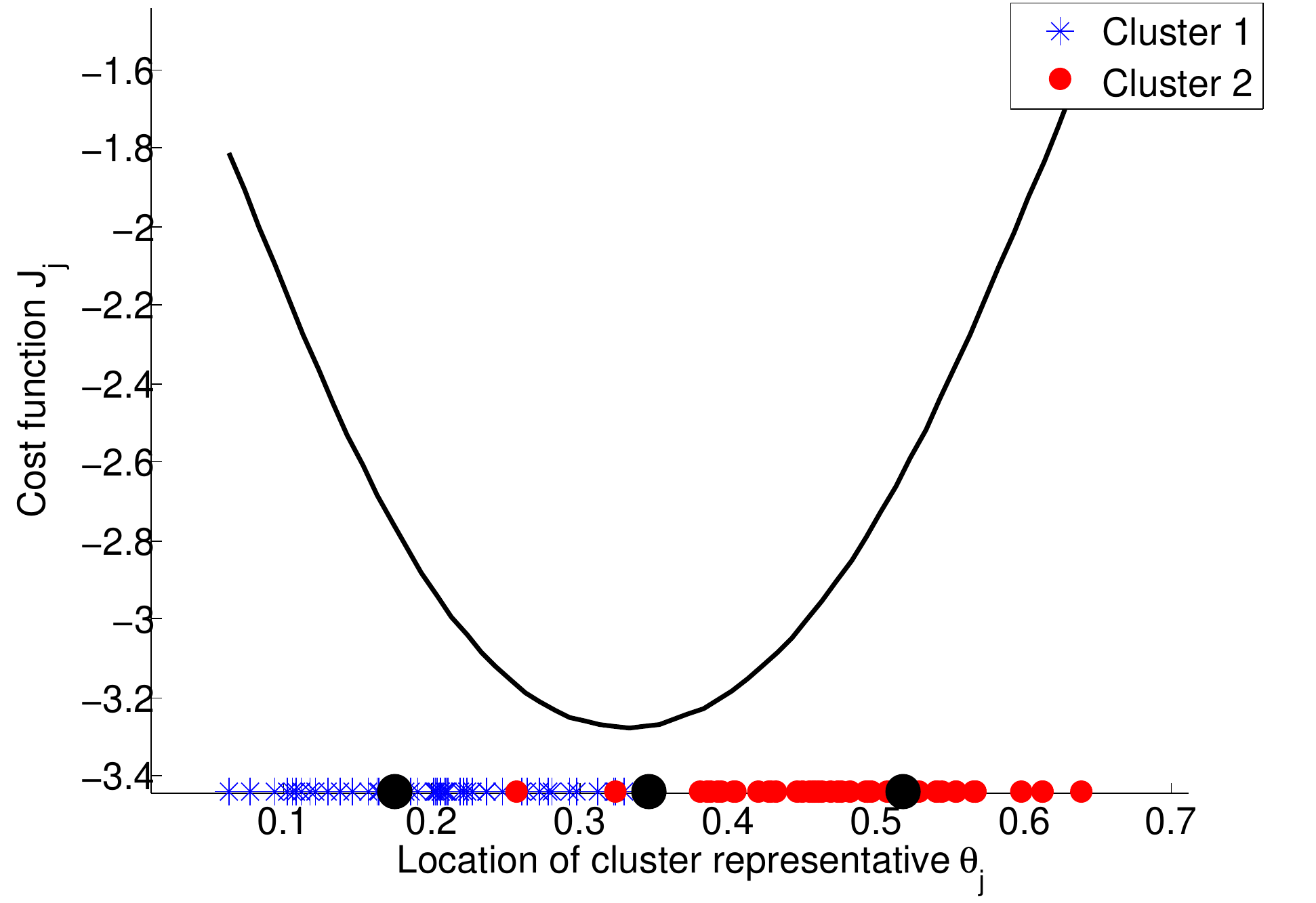}\vspace{20pt} %47
\label{costfun_m3a005}}
\hfil
\centering
\subfloat[$m_{ini}=3$, $\alpha=1$]{\includegraphics[width=0.32\textwidth]{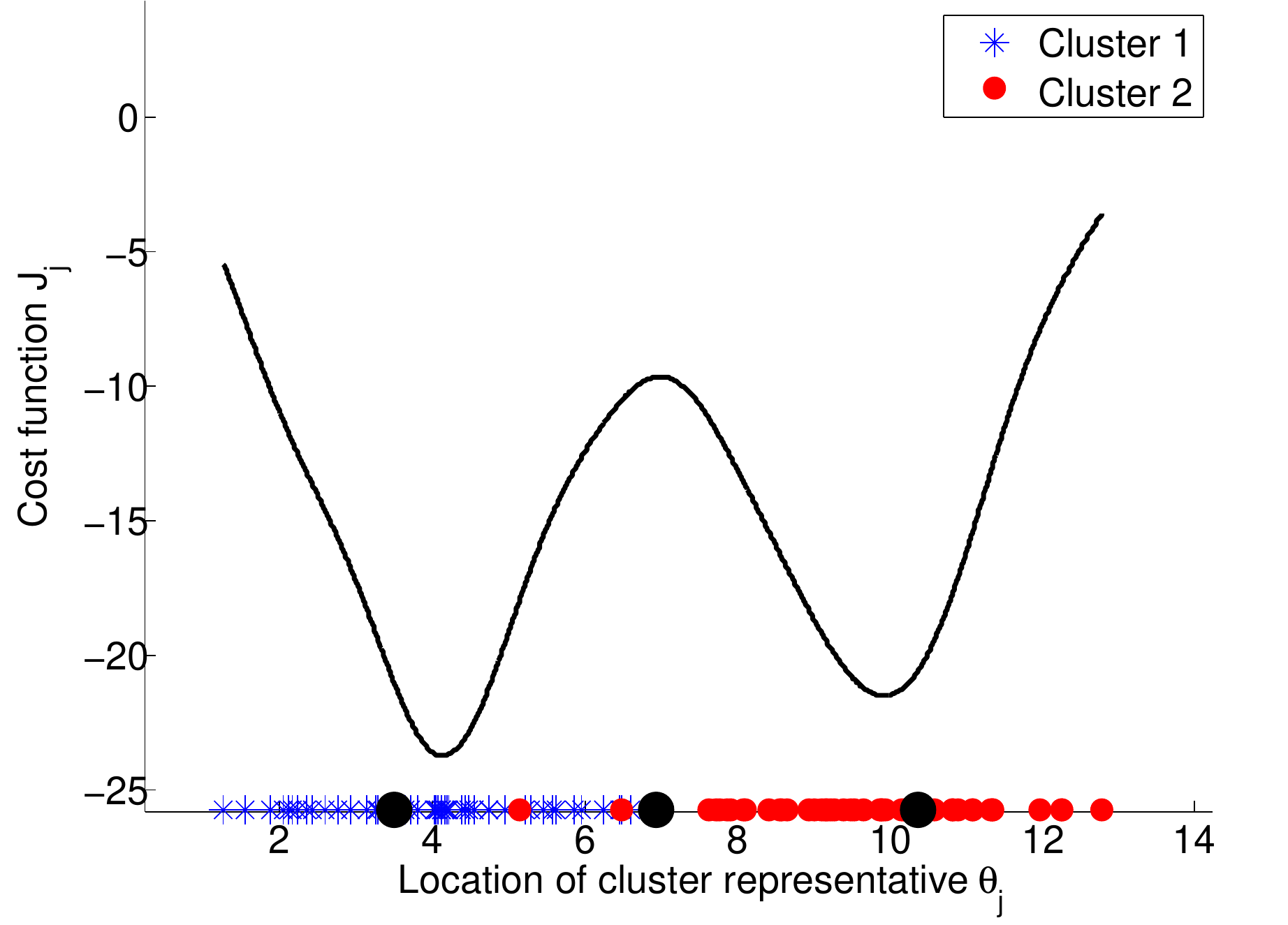} %51
\label{costfun_m3a1}}
\hfil
\centering
\subfloat[$m_{ini}=3$, $\alpha=2$]{\includegraphics[width=0.32\textwidth]{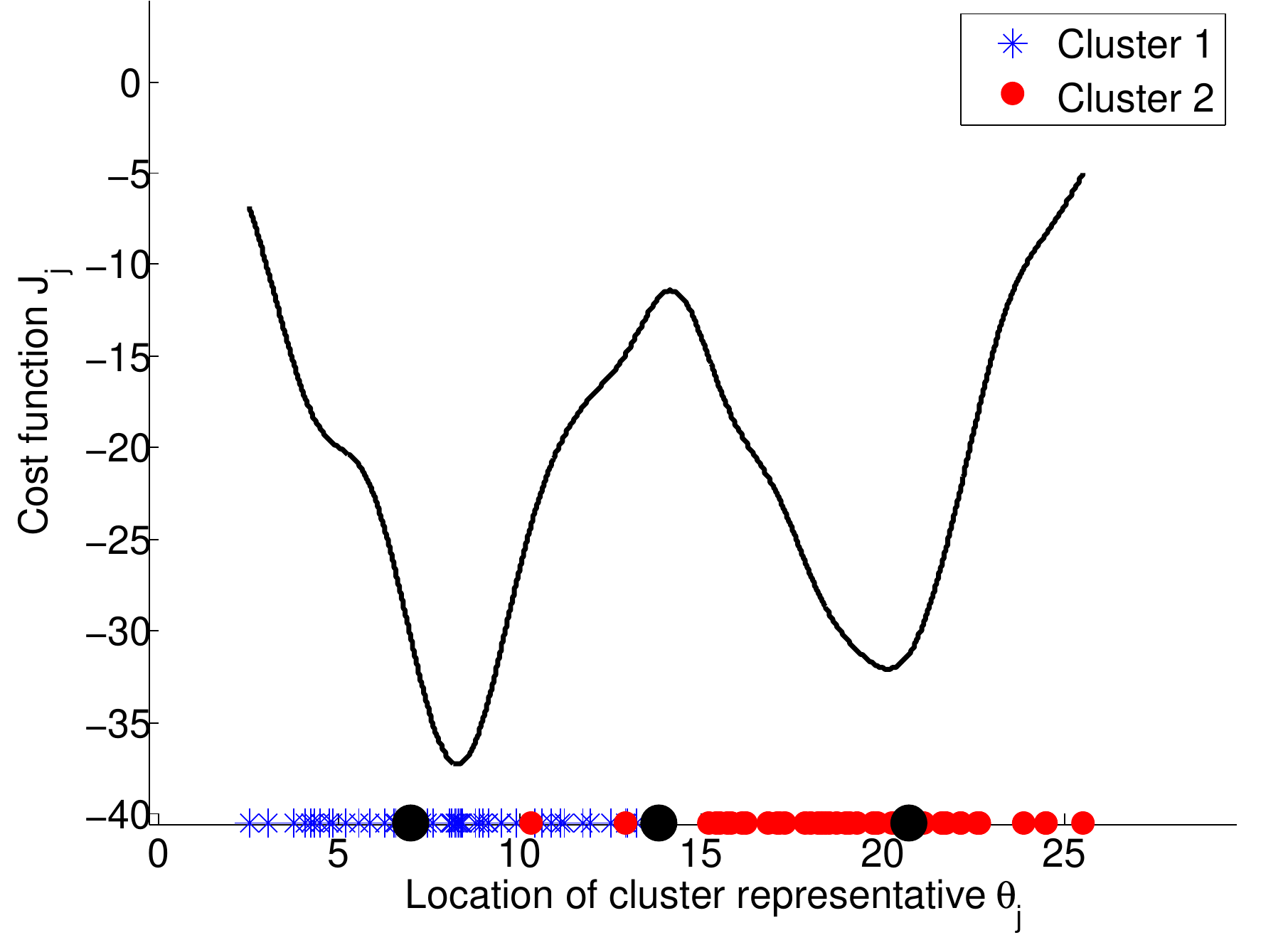}\vspace{20pt} %47
\label{costfun_m3a2}}
\hfil
\centering
\subfloat[$m_{ini}=10$, $\alpha=0.05$]{\includegraphics[width=0.32\textwidth]{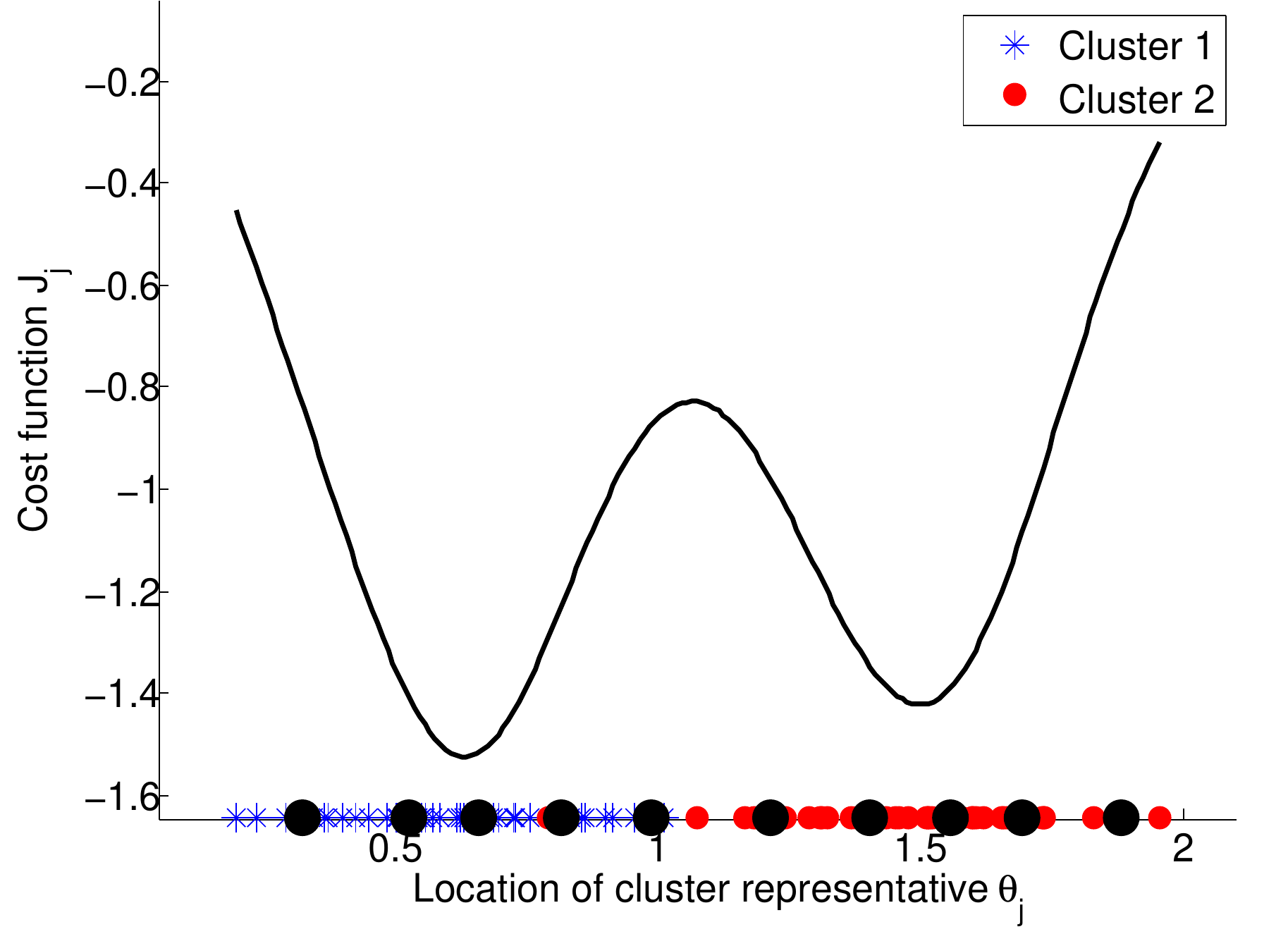} %51
\label{costfun_m10a005}}
\hfil
\centering
\subfloat[$m_{ini}=10$, $\alpha=1$]{\includegraphics[width=0.32\textwidth]{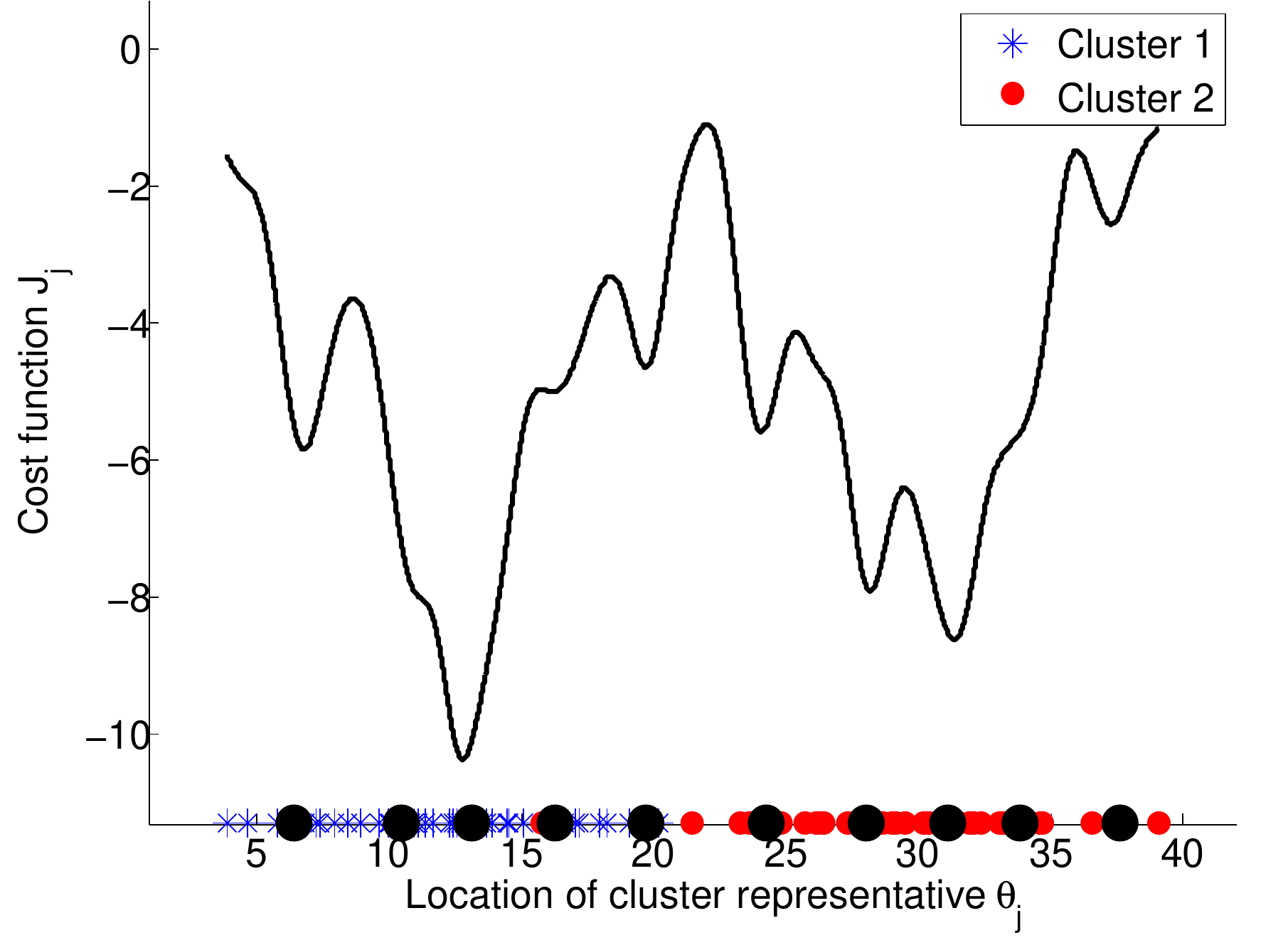} %51
\label{costfun_m10a1}}
\hfil
\centering
\subfloat[$m_{ini}=10$, $\alpha=2$]{\includegraphics[width=0.32\textwidth]{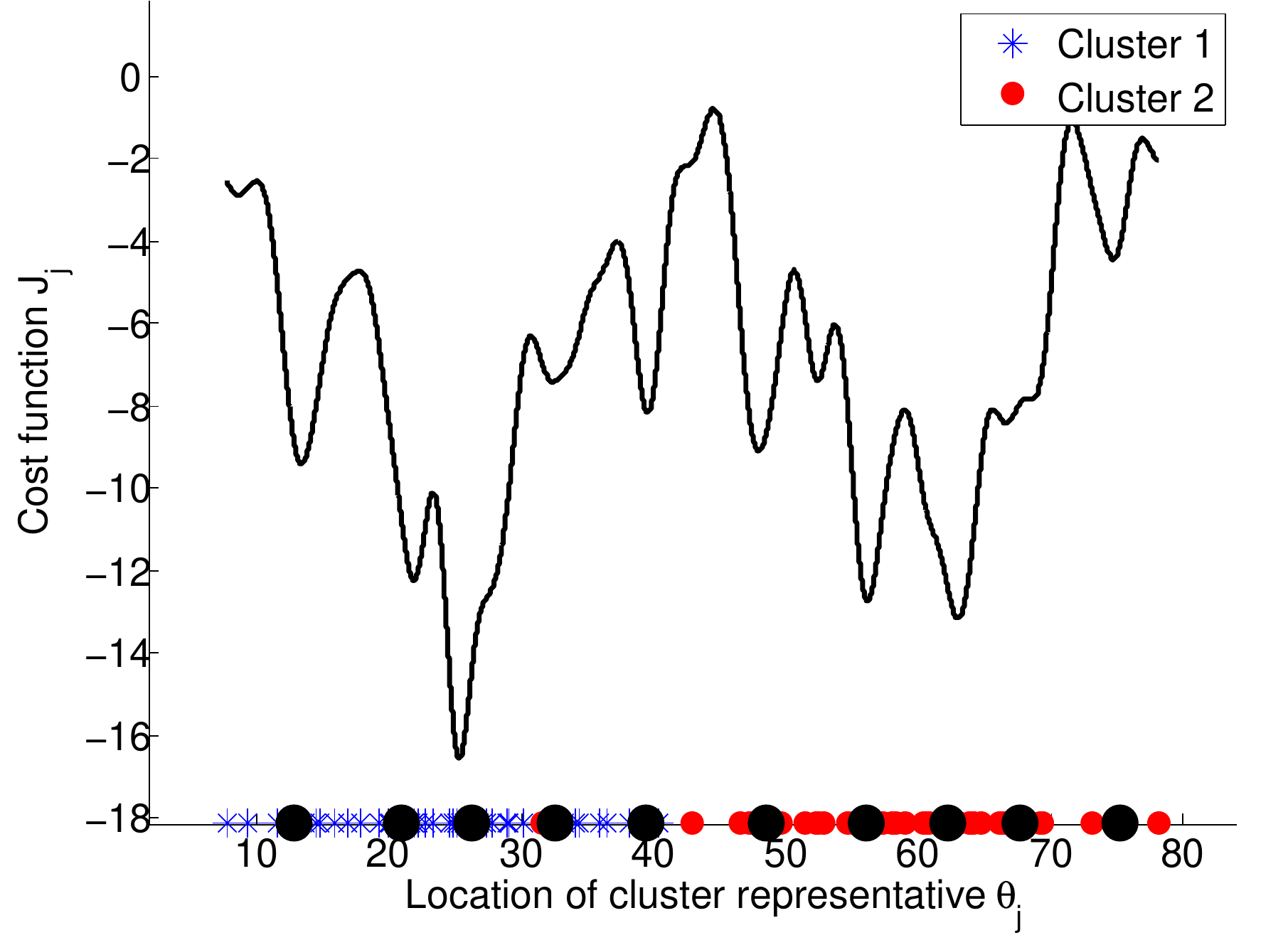} %51
\label{costfun_m10a2}}
\hfil
\centering{\caption{Plot of the APCM cost function in case of a two-class 1-dim data set. Plot shows (a) the data set. Data points are denoted by stars on the $x$-axis and representatives by black dots. Results for (b) $m_{ini}=3$, $\alpha=0.05$, (c) $m_{ini}=3$, $\alpha=1$, (d) $m_{ini}=3$, $\alpha=2$, (e) $m_{ini}=10$, $\alpha=0.05$, (f) $m_{ini}=10$, $\alpha=1$ and (g) $m_{ini}=10$, $\alpha=2$.
}\label{costfun}}
\end{figure*}

This example indicates that in cases where $m_{ini}$ is chosen not to be very larger than the actual number of clusters $m$, appropriate values for the parameter $\alpha$ are around 1. On the other hand, when $m_{ini}$ is chosen much larger than $m$, parameter $\alpha$ should be taken much less than $1$. However, in more demanding data sets, which contain very closely located natural clusters and for a fixed value of $m_{ini}$, larger values for the parameter $\alpha$ should be chosen, compared to cases of less closely located clusters, in order to discourage the movement of a representative from one dense region to another. Experiment showed that values of $\alpha$ around 1 and up to 3 suffice for almost any data set, provided that $m_{ini}$ is not extremely larger than $m$ (about 3-4 times larger).

\section{Experimental results}
\label{sec4}
In this section, we assess the performance of the proposed method in several experimental settings and illustrate the obtained results. More specifically, we consider two series of experiments. In the first one, we use two-dimensional simulated data sets in order to exhibit more clearly certain aspects of the behavior of the APCM itself. In the second one, we use both simulated and real-world data sets (Iris \cite{UCILib}, New Thyroid \cite{UCILib}, and a hyperspectral image data set \cite{HsiSal}) of both low and high dimensionality to evaluate the performance of APCM in comparison with several other related algorithms.

\subsection{Behavior of the APCM}
\label{subsec41}

{\bf Experiment 1}: Let us consider a two-dimensional data set consisting of $N=17$ points, which form two natural clusters $C_1$ and $C_2$ with 12 and 5 data points, respectively (see Fig.~\ref{example1}). The means of the clusters are $\mathbf{c}_1=[1.75, 2.75]$ and $\mathbf{c}_2=[4.25, 2.75]$. In this experiment, we consider only the PCM (with $m=2$) and the APCM (with $m_{ini}=2$, $\alpha=1$) algorithms. Figs.~\ref{ex1initPCM} and \ref{ex1initAPCM} show the initial positions of the cluster representatives that are taken from the FCM clustering algorithm and the circles with radius equal to $\sqrt{\gamma_j}$'s resulting from eq.~(\ref{Ketaj}) (for $K=1$) for PCM and from eq.~(\ref{initetaj}) for APCM. Similarly, Figs.~\ref{ex1PCM1st} and \ref{ex1APCM1st} show the new locations of $\boldsymbol{\theta}_j$'s after the first iteration of the algorithms and Figs.~\ref{ex1PCM13th},~\ref{ex1APCM8th} show the locations of $\boldsymbol{\theta}_j$'s at a later iteration of them. Table~\ref{table:synth1} shows the degrees of compatibility $u_{ij}$'s of all data points $\mathbf{x}_i$ with the cluster representatives $\boldsymbol{\theta}_j$'s at the three specific iterations depicted in Fig.~\ref{example1} (initial, $1^{st}$ for both algorithms, $13^{th}$ for PCM and $10^{th}$ (final) for APCM). 
%First, the $u_{ij}^{FCM}$'s at the initialization step of both PCM and APCM, resulting from the FCM algorithm, are shown. Moreover, the corresponding $u_{ij}$'s at the first and eighth iteration of PCM and APCM, resulting from eq.~(\ref{uij}) and eq.~(\ref{adapteta}), respectively, are also presented. 

\begin{figure*}[htb!]
%\captionsetup{width=0.55\textwidth}
\centering
\subfloat[Initial step of PCM]{\includegraphics[width=0.32\textwidth]{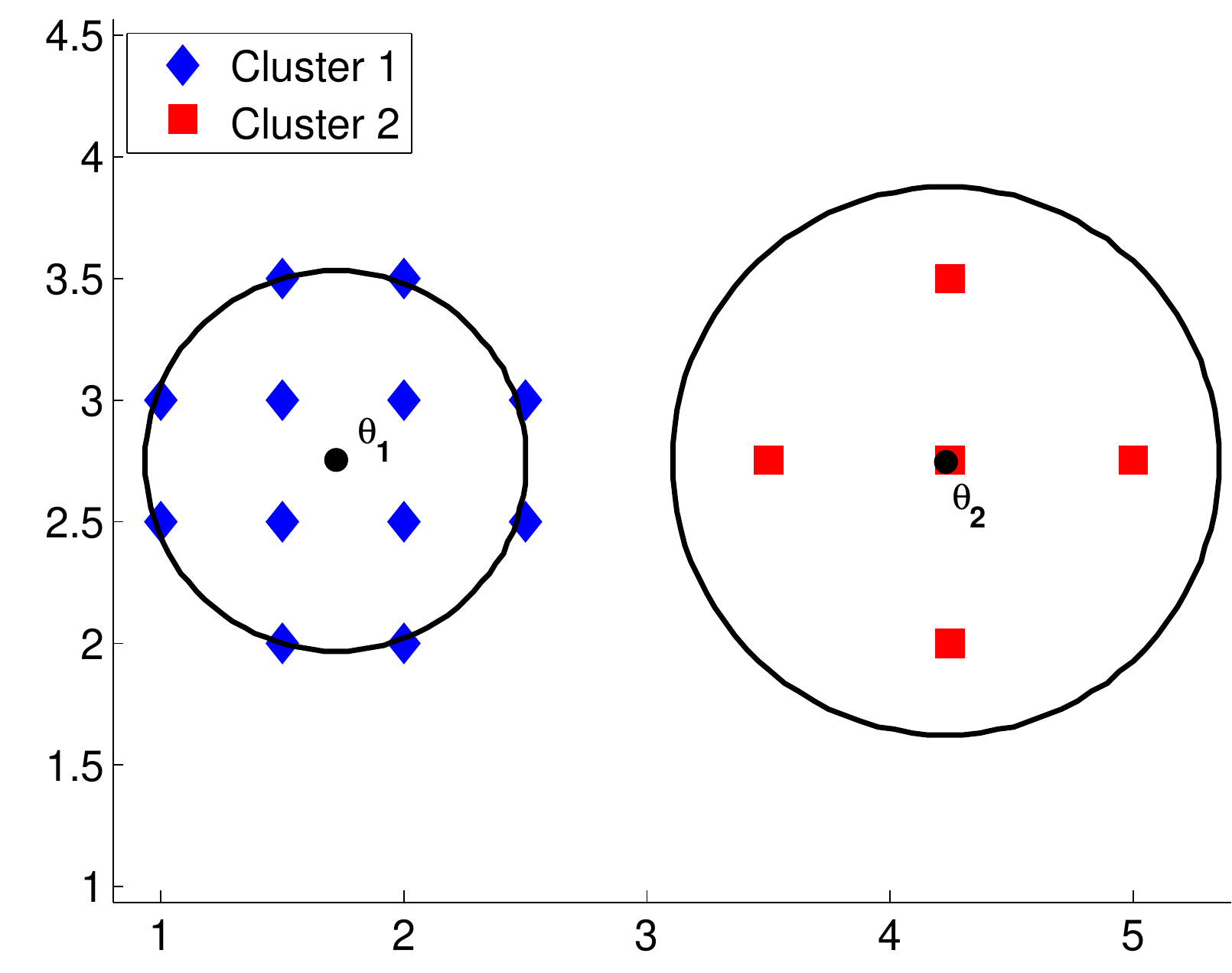}\hspace{6pt}\label{ex1initPCM}}
\hfil
\centering
\subfloat[$1^{st}$ iteration of PCM]{\includegraphics[width=0.32\textwidth]{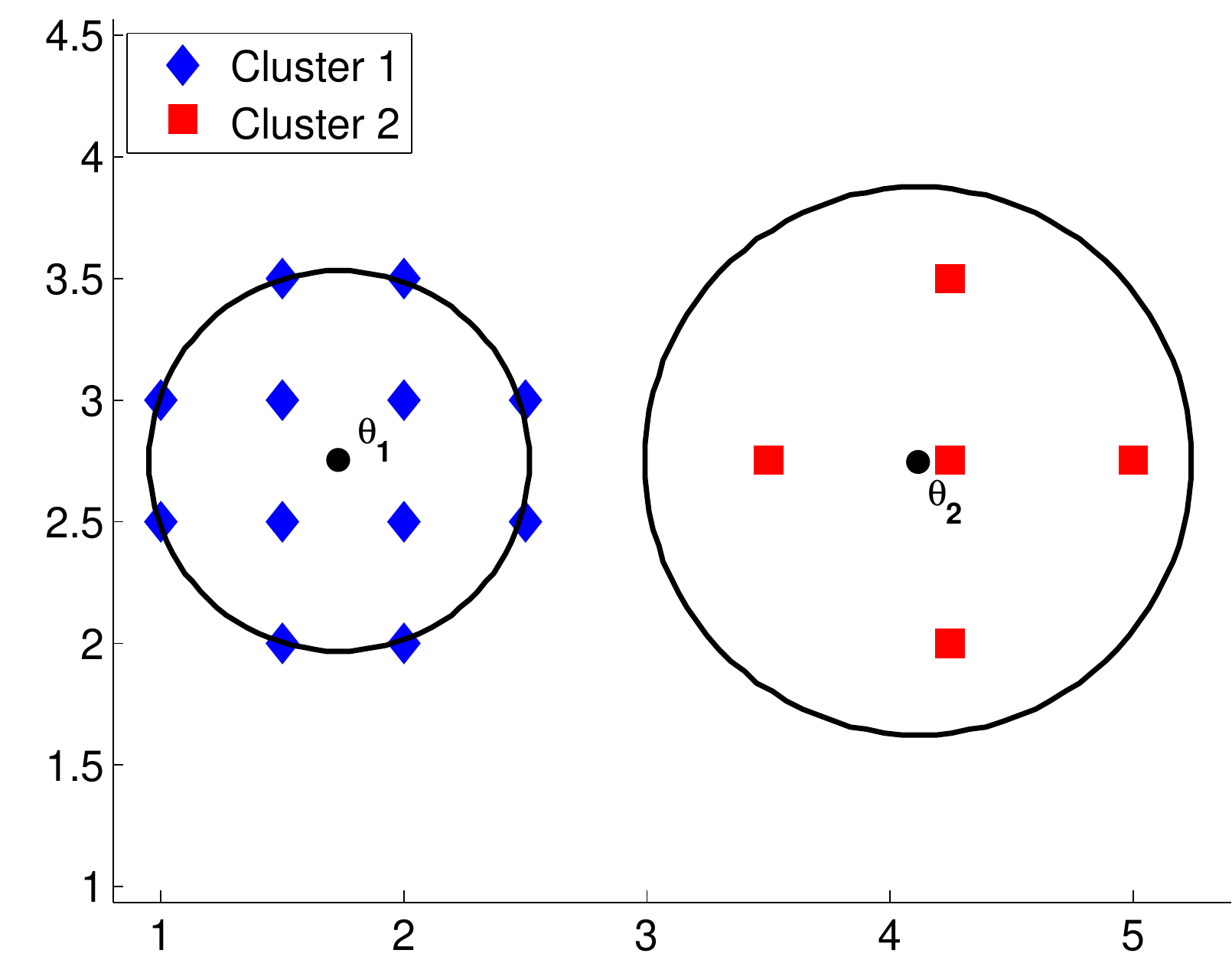}\hspace{6pt}\label{ex1PCM1st}}
\hfil
\centering
\subfloat[$13^{th}$ iteration of PCM]{\includegraphics[width=0.32\textwidth]{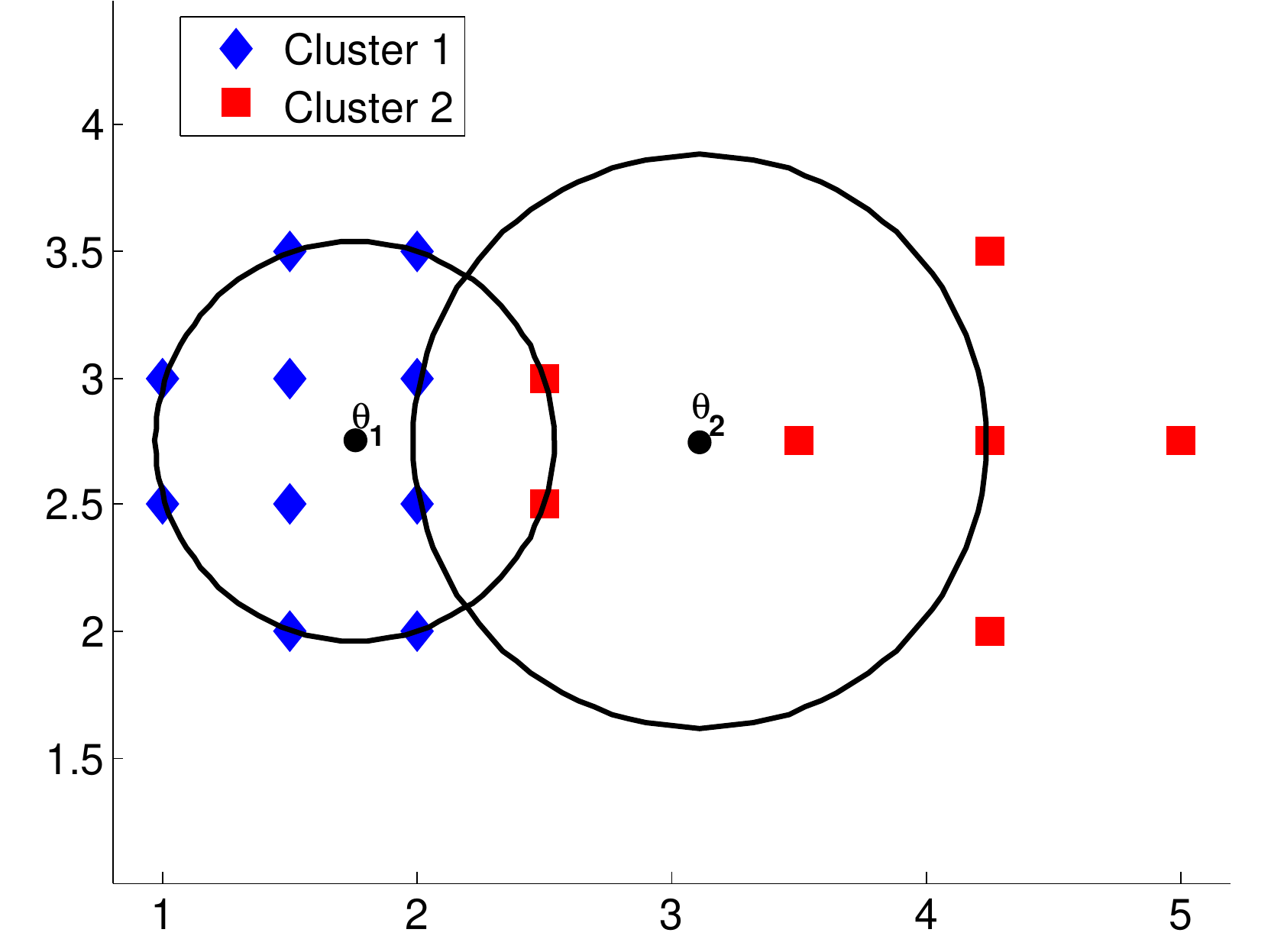}\hspace{6pt}\label{ex1PCM13th}}
\hfil
\centering
\subfloat[Initial step of APCM]{\includegraphics[width=0.33\textwidth]{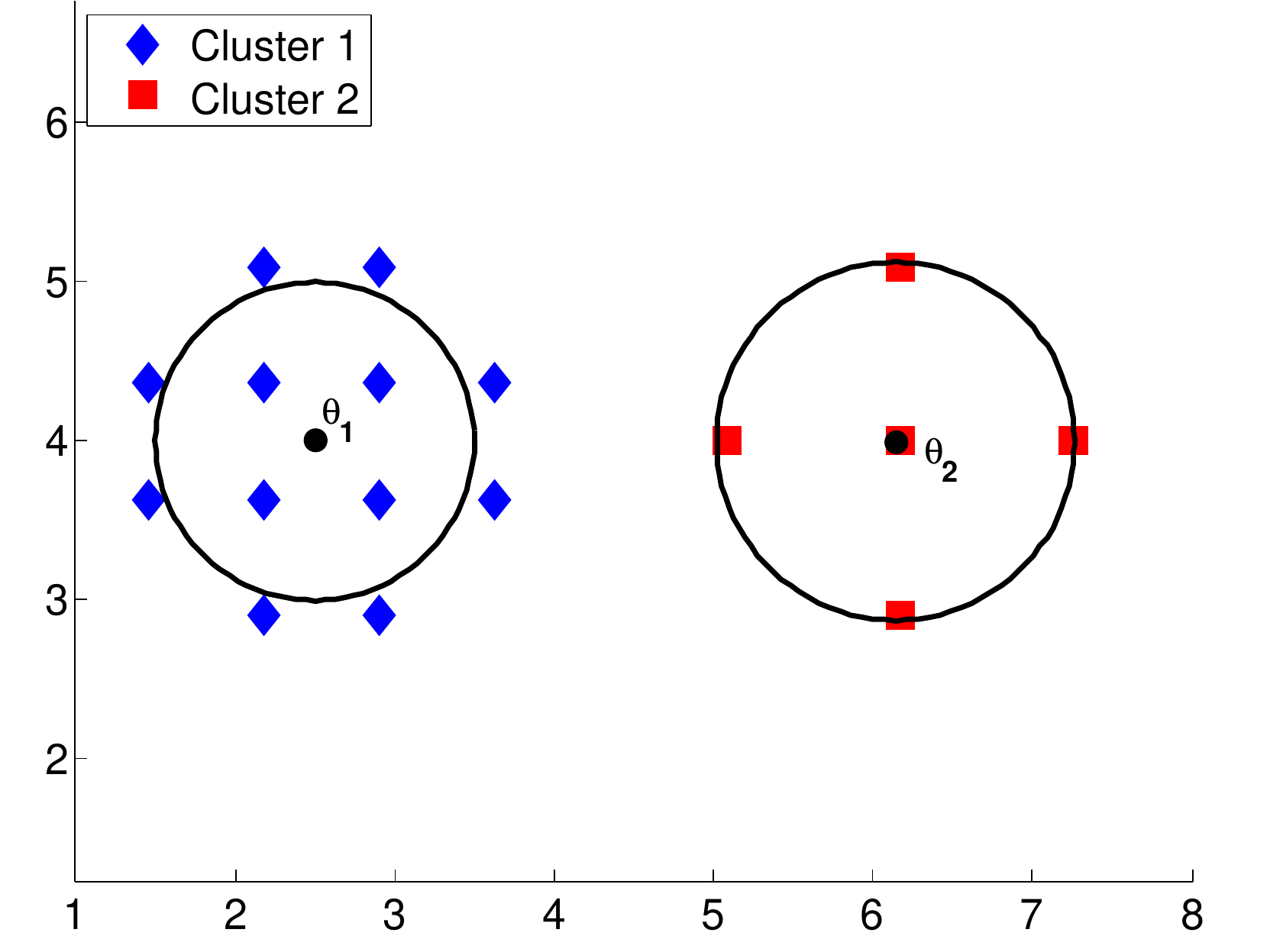}\hspace{-6pt}\label{ex1initAPCM}}
\hfil
\centering
\subfloat[$1^{st}$ iteration of APCM]{\includegraphics[width=0.33\textwidth]{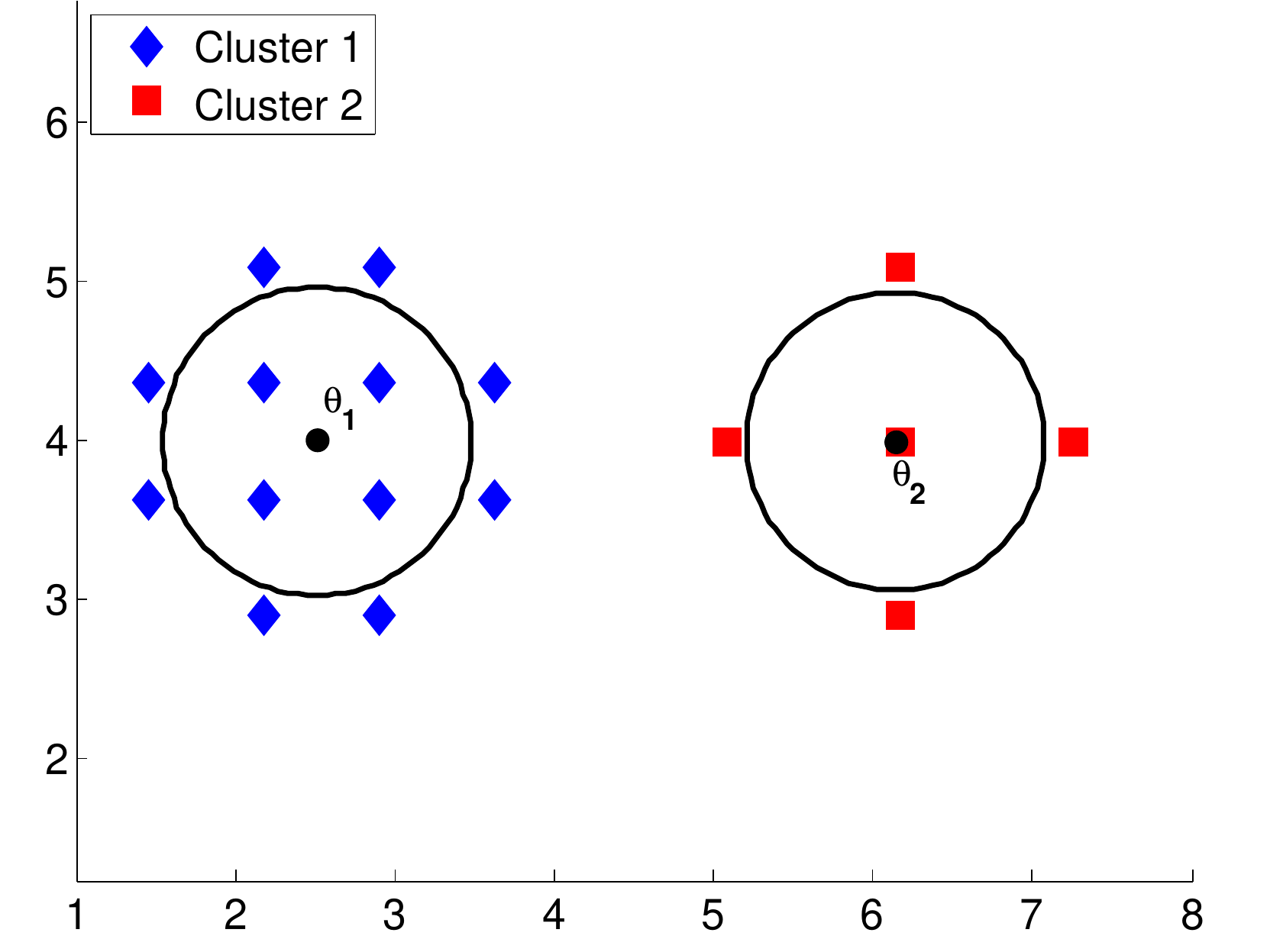}\hspace{-7pt}\label{ex1APCM1st}}
\hfil
\centering
\subfloat[$10^{th}$ (final) iter. of APCM]{\includegraphics[width=0.33\textwidth]{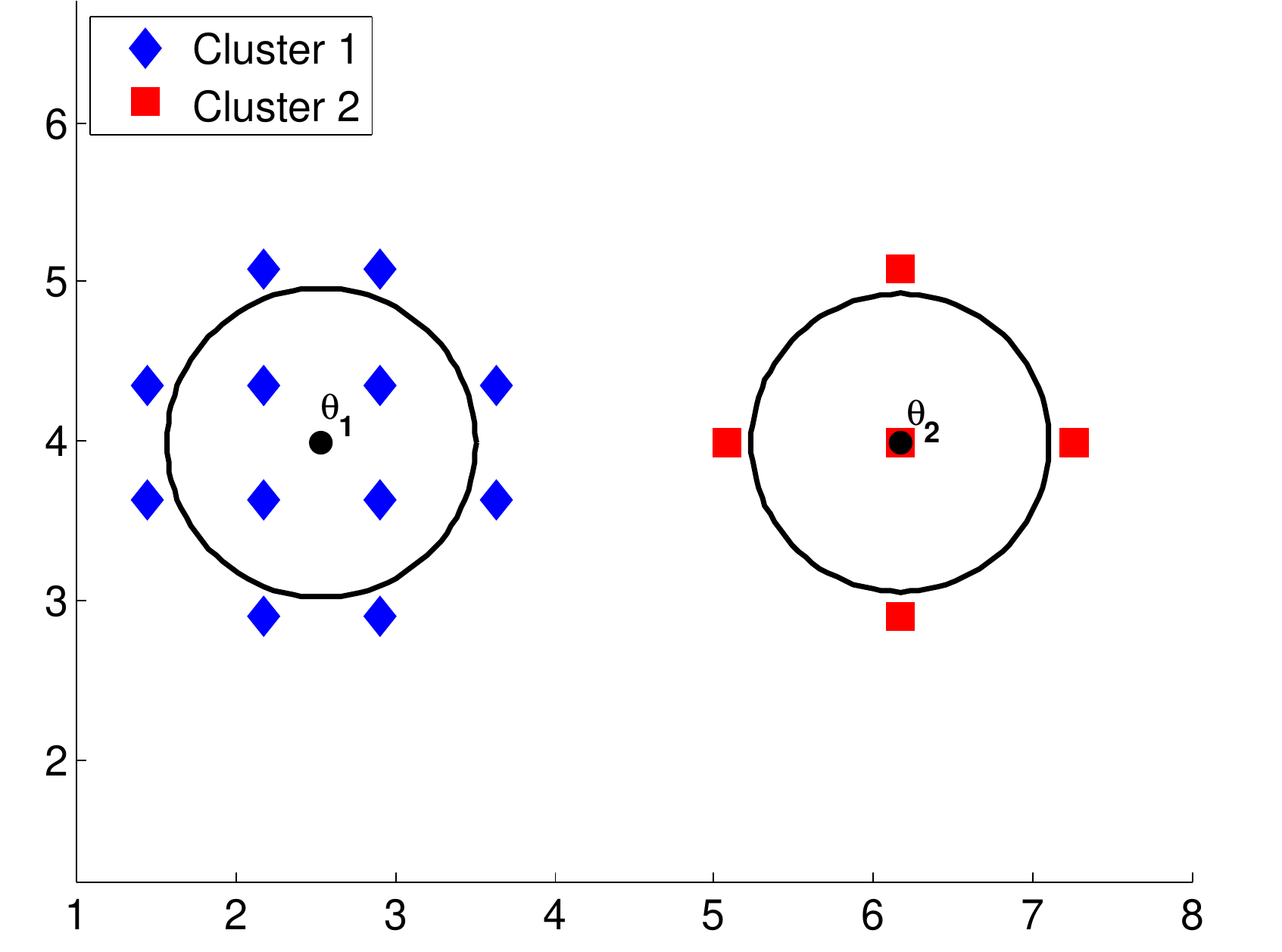}\hspace{-8pt}\label{ex1APCM8th}}
\hfil
\centering{\caption{PCM and APCM snapshots at their initialization step, 1st iteration and 13th iteration for PCM and 10th (final) iteration for APCM (experiment 1).}\label{example1}}
\end{figure*}

As it can be deduced from Table~\ref{table:synth1} and Fig.~\ref{example1}, the degrees of compatibility of the data points of $C_1$ with the cluster representative $\boldsymbol{\theta}_2$ increase as PCM evolves, leading gradually $\boldsymbol{\theta}_2$ towards the region of the cluster $C_1$ and thus, ending up with two coincident clusters, although $\boldsymbol{\theta}_1$ and $\boldsymbol{\theta}_2$ are initialized properly through the FCM algorithm (see Fig.~\ref{ex1initPCM}). However, this is not the case in APCM algorithm, as both the cluster representatives remain in the centers of the actual clusters. Obviously, this differentation on the behavior of the two algorithms is due to the different definition of the parameters $\gamma_j$'s, which affect the degrees of compatibility of the data points with each cluster (see eqs.~(\ref{Ketaj}),~(\ref{adapteta}) and~(\ref{uij})). This experiment indicates that, in principle, APCM can handle successfully cases where relatively closely located clusters with different densities are involved.

\begin{table*}[htb!]
\centering
\caption{The degrees of compatibility of the data points of experiment 1 for PCM and APCM algorithms, after: (a) initialization (common to both algorithms), (b) first iteration and (c) 13th iteration for PCM and 10th (final) iteration for APCM.}
{\footnotesize
\begin{tabular}{>{\centering\arraybackslash}m{0.09\textwidth} | >{\centering\arraybackslash}m{0.065\textwidth} | >{\centering\arraybackslash}m{0.05\textwidth} ||>{\centering\arraybackslash}m{0.065\textwidth} | >{\centering\arraybackslash}m{0.065\textwidth} || >{\centering\arraybackslash}m{0.065\textwidth} |>{\centering\arraybackslash}m{0.065\textwidth}  || >{\centering\arraybackslash}m{0.065\textwidth} |>{\centering\arraybackslash}m{0.05\textwidth}  || >{\centering\arraybackslash}m{0.065\textwidth} |>{\centering\arraybackslash}m{0.065\textwidth}}
\hline  
	& \multicolumn{2}{c||}{Initialization}  & \multicolumn{4}{c||}{$1^{st}$ iteration} & \multicolumn{2}{c||}{$13^{th}$ iteration} & \multicolumn{2}{c}{$10^{th}$ iteration}\\
\cline{2-11}
	& \multicolumn{2}{c||}{PCM/APCM} &  \multicolumn{2}{c||}{PCM} & \multicolumn{2}{c||}{APCM} & \multicolumn{2}{c||}{PCM} & \multicolumn{2}{c}{APCM} \\
\cline{2-11}
	\centering $\mathbf{x}_i$ & \centering $C_1$ & \centering $C_2$ & \centering $C_1$ & \centering $C_2$ & \centering $C_1$ & {\centering $C_2$} & \centering $C_1$ & \centering $C_2$ & \centering $C_1$ & {\centering $C_2$}\\
\hline
{\footnotesize$(1.5, 3.5)$}  	 & \centering 0.9292 & \centering 0.0708 & \centering 0.3701 & \centering 0.0018 & \centering 0.2757 & {\centering 1.6e-06} & \centering 0.3604 & \centering 0.0831 & \centering 0.2449 & {\centering 3.0e-09} \\
{\footnotesize$(2.0, 3.5)$}    & \centering 0.8963 & \centering 0.1037 & \centering 0.3526 & \centering 0.0127   & \centering 0.2590 & {\centering 9.6e-05} & \centering 0.3632 & \centering 0.2428 & \centering 0.2447 & {\centering 1.3e-06} \\
{\footnotesize$(1.0, 3.0)$}    & \centering 0.9475 & \centering 0.0525 & \centering 0.3884 & \centering 2.6e-04 & \centering 0.2936 & {\centering 2.4e-08} & \centering 0.3575 & \centering 0.0284 & \centering 0.2451 & {\centering 7.2e-12} \\
{\footnotesize$(1.5, 3.0)$}    & \centering 0.9854 & \centering 0.0146 & \centering 0.8348 & \centering 0.0027   & \centering 0.7913 & {\centering 3.4e-06} & \centering 0.8178 & \centering 0.1232 & \centering 0.7550 & {\centering 1.0e-08}\\
{\footnotesize$(2.0, 3.0)$}    & \centering 0.9728 & \centering 0.0272 & \centering 0.7954 & \centering 0.0188   & \centering 0.7432 & {\centering 2.2e-04} & \centering 0.8192 & \centering 0.3602 & \centering 0.7544 & {\centering 4.3e-06}\\
{\footnotesize$(2.5, 3.0)$}    & \centering 0.8201 & \centering 0.1799 & \centering 0.3360 & \centering 0.0897   & \centering 0.2433 & {\centering 0.0060} & \centering 0.3661 & \centering 0.7098 & \centering 0.2445 & {\centering 5.4e-04}\\
{\footnotesize$(1.0, 2.5)$}    & \centering 0.9475 & \centering 0.0525 & \centering 0.3884 & \centering 2.6e-04 & \centering 0.2936 & {\centering 2.4e-08} & \centering 0.3575 & \centering 0.0284 & \centering 0.2451 & {\centering 7.2e-12}\\
{\footnotesize$(1.5, 2.5)$}    & \centering 0.9854 & \centering 0.0146 & \centering 0.8348 & \centering 0.0027   & \centering 0.7913 & {\centering 3.4e-06} & \centering 0.8128 & \centering 0.1232 & \centering 0.7550 & {\centering 1.0e-08}\\
{\footnotesize$(2.0, 2.5)$}    & \centering 0.9728 & \centering 0.0272 & \centering 0.7954 & \centering 0.0188   & \centering 0.7432 & {\centering 2.2e-04} & \centering 0.8192 & \centering 0.3602 & \centering 0.7544 & {\centering 4.3e-06}\\
{\footnotesize$(2.5, 2.5)$} & \centering 0.8201 & \centering 0.1799 & \centering 0.3360 & \centering 0.0897   & \centering 0.2433 & {\centering 0.0060} & \centering 0.3661 & \centering 0.7098 & \centering 0.2445 & {\centering 5.4e-04}\\
{\footnotesize$(1.5, 2.0)$} & \centering 0.9292 & \centering 0.0708 & \centering 0.3701 & \centering 0.0018   & \centering 0.2757 & {\centering 1.6e-06} & \centering 0.3604 & \centering 0.0831 & \centering 0.2449 & {\centering 3.0e-09}\\
{\footnotesize$(2.0, 2.0)$} & \centering 0.8963 & \centering 0.1037 & \centering 0.3526 & \centering 0.0127   & \centering 0.2590 & {\centering 9.6e-05} & \centering 0.3632 & \centering 0.2428 & \centering 0.2447 & {\centering 1.3e-06}\\
\hline
{\footnotesize$(4.25, 3.5)$} & \centering 0.0748   & \centering 0.9252 & \centering 1.2e-05 & \centering 0.6415 & \centering 4.2e-07 & {\centering 0.3903} & \centering 1.6e-05 & \centering 0.2302 & \centering 2.2e-07 & {\centering 0.2563}\\
{\footnotesize$(3.5, 2.75)$} & \centering 0.1441   & \centering 0.8559 & \centering 0.0058   & \centering 0.6566 & \centering 0.0013   & {\centering 0.4101} & \centering 0.0071 & \centering 0.8869 & \centering 0.0010 & {\centering 0.2600}\\
{\footnotesize$(4.25, 2.75)$} & \centering 6.0e-05 & \centering 0.9999 & \centering 3.0e-05 & \centering 0.9997 & \centering 1.3e-06 & {\centering 0.9994} & \centering 4.0e-05 & \centering 0.3587 & \centering 7.7e-07 & {\centering 1.0000}\\
{\footnotesize$(5.0, 2.75)$} & \centering 0.0522   & \centering 0.9478 & \centering 2.6e-08 & \centering 0.6267 & \centering 1.4e-10 & {\centering 0.3715} & \centering 3.6e-08 & \centering 0.0597 & \centering 4.7e-11 & {\centering 0.2527}\\
{\footnotesize$(4.25, 2.0)$} & \centering 0.0748   & \centering 0.9252 & \centering 1.2e-05 & \centering 0.6415 & \centering 4.2e-07 & {\centering 0.3903} & \centering 1.6e-05 & \centering 0.2302 & \centering 2.2e-07 & {\centering 0.2563}\\
\hline
\end{tabular}}
\label{table:synth1}
\end{table*}

In the next experiment, we investigate on the relation between $m_{ini}$ and parameter $\alpha$.

\begin{figure*}[htb!]
%\captionsetup{width=0.55\textwidth}
\centering
\subfloat[$\alpha=0.5$]{\includegraphics[width=0.33\textwidth]{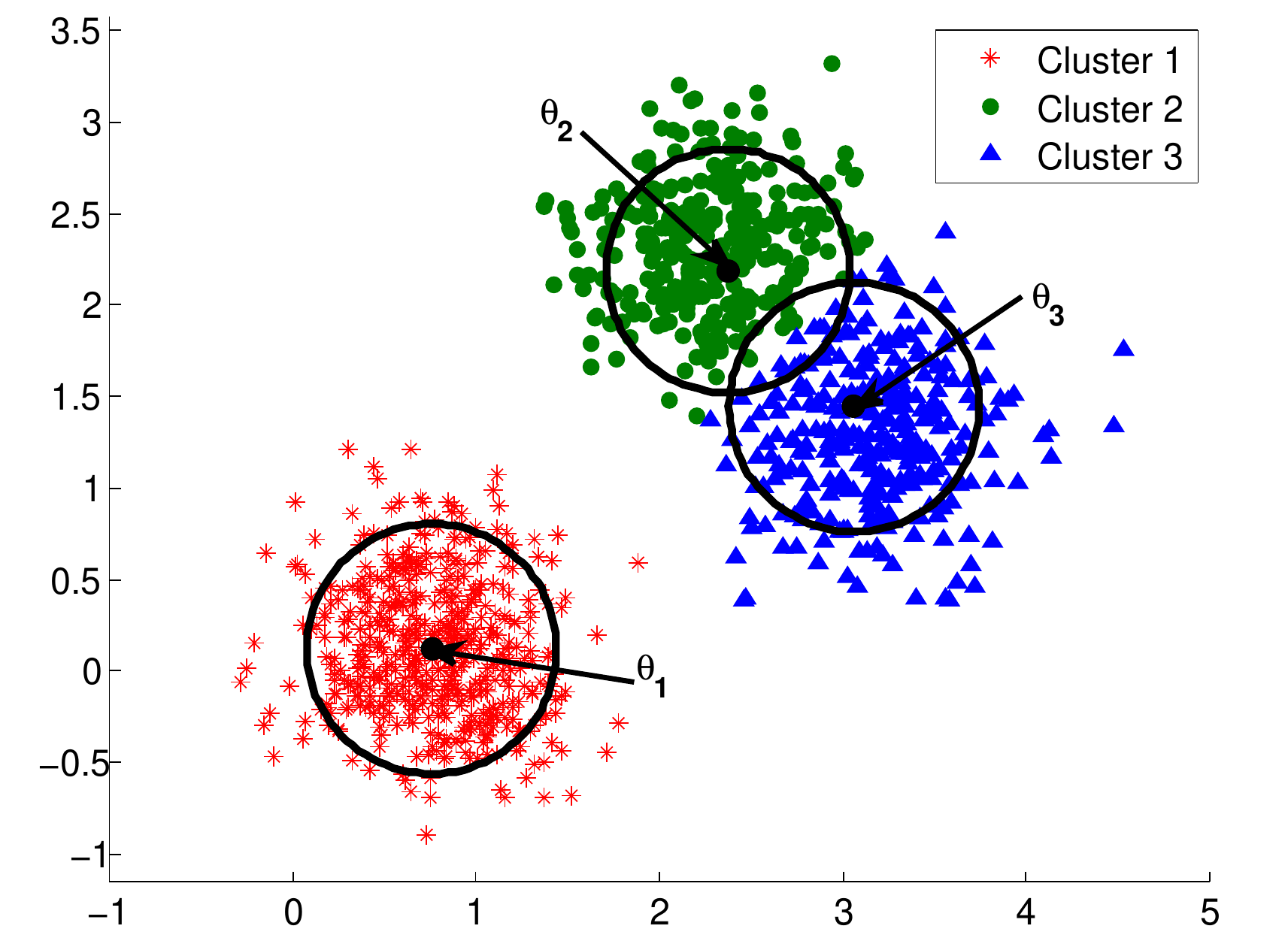}\vspace{20pt}}
\hfil
\centering
\subfloat[$\alpha=1.0$]{\includegraphics[width=0.33\textwidth]{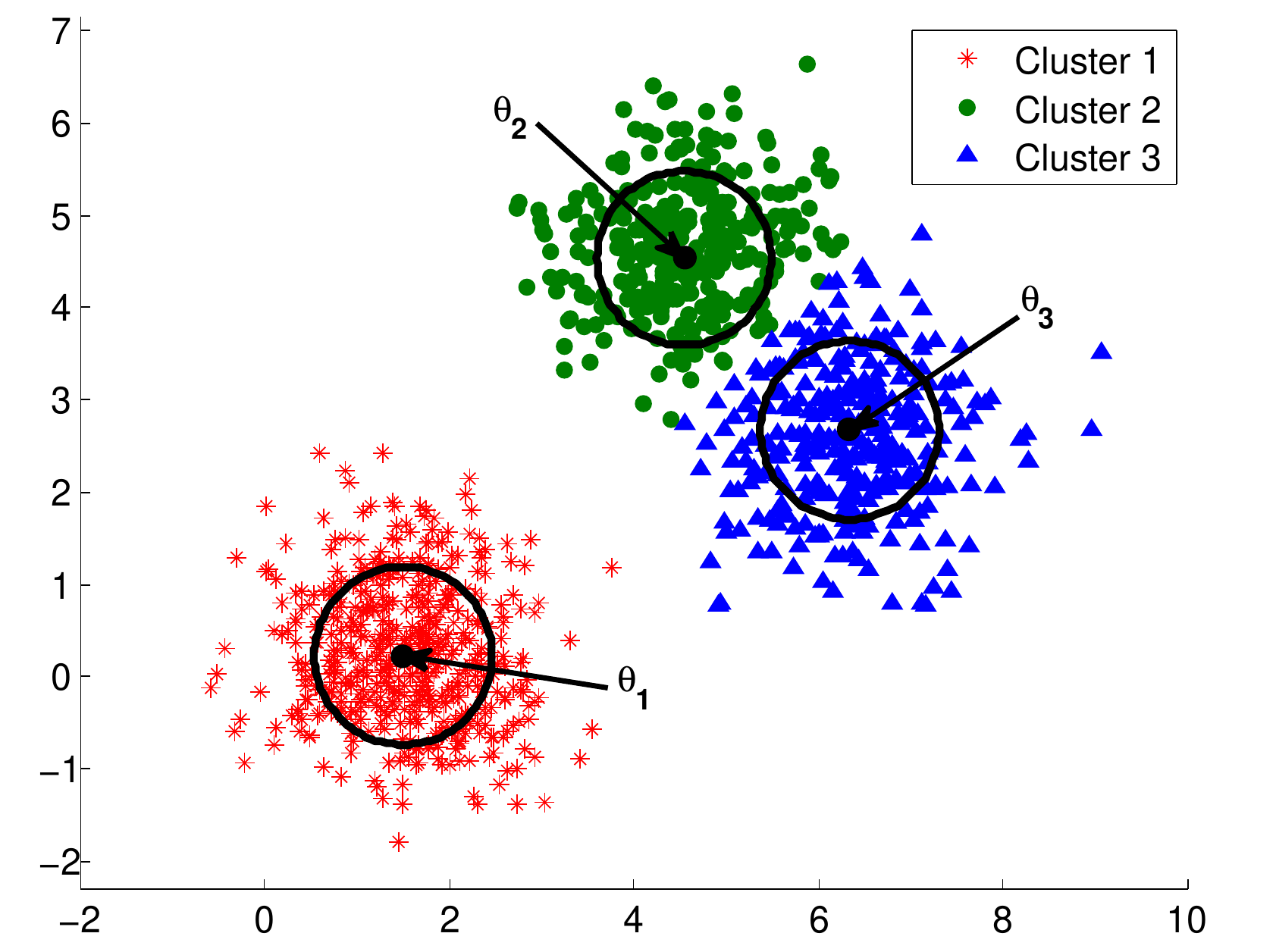}}
\hfil
\centering
\subfloat[$\alpha=3.0$]{\includegraphics[width=0.33\textwidth]{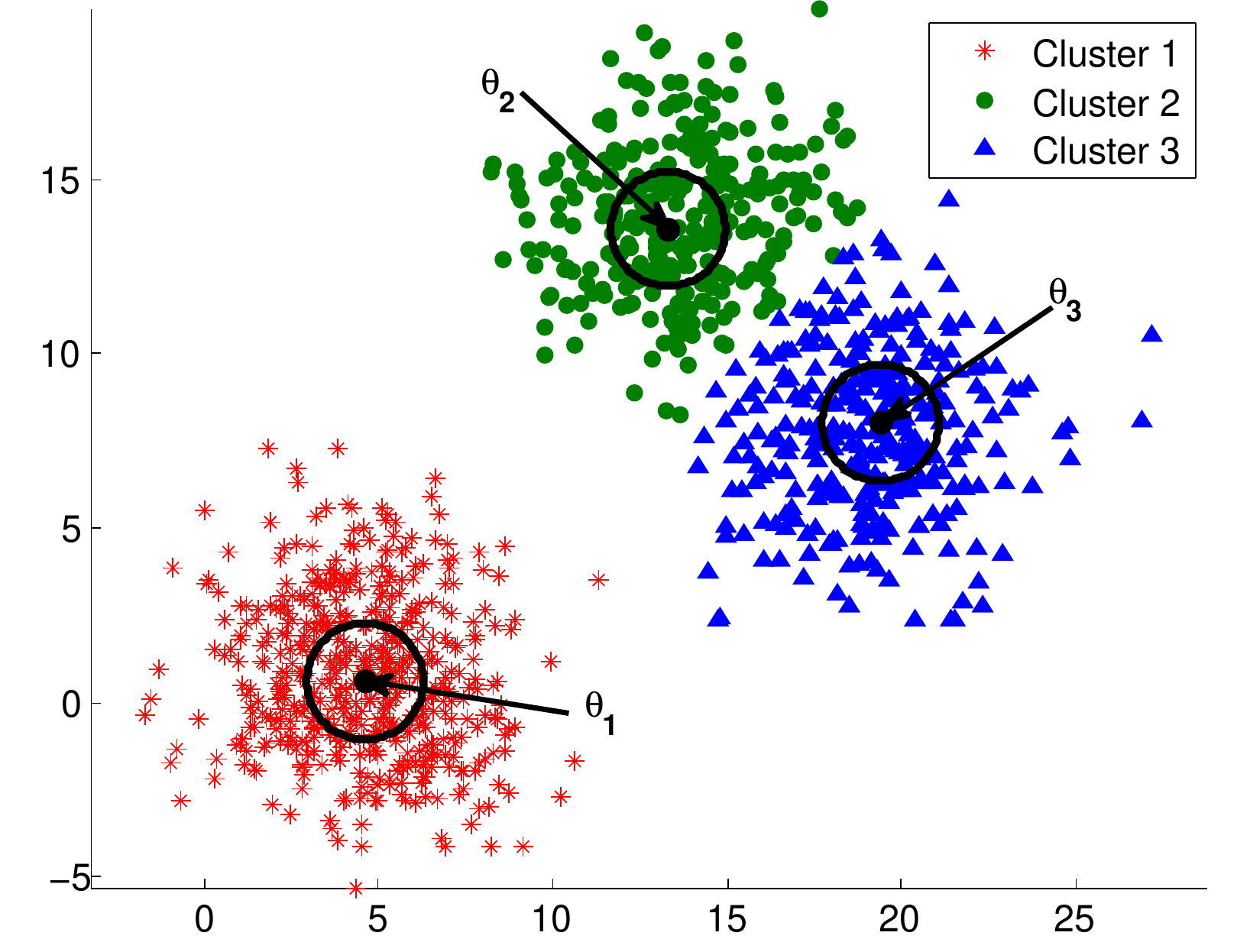}}
\hfil
\centering{\caption{The clustering results of APCM for experiment 2, when it is initialized with $m_{ini}=5$, for several values of parameter $\alpha$.}\label{example2}}
\end{figure*}

{\bf Experiment 2}: Let us consider now a two-dimensional data set consisting of $N=1100$ points, which form three natural clusters $C_1$, $C_2$ and $C_3$ (see Fig.~\ref{example2}). Each such cluster is modelled by a normal distribution. The (randomly generated) means of the distributions are $\mathbf{c}_1=[1.35, 0.23]^T$, $\mathbf{c}_2=[4.03, 4.09]^T$ and $\mathbf{c}_3=[5.64, 2.28]^T$, respectively, while their (common) covariance matrix is set equal to $0.4\cdot I_2$, where $I_2$ is the $2\times 2$ identity matrix. A number of 500 points is generated by the first distribution and 300 points are generated by each one of the other two distributions. Note that clusters $C_2$ and $C_3$ lie very close to each other and, therefore, their discrimination is considered as a difficult task for a clustering algorithm. Table~\ref{table:synth2} shows the ranges of values of the parameter $\alpha$, for which APCM manages to identify correctly the naturally formed $m=3$ clusters, for various values of $m_{ini}$. Fig.~\ref{example2} shows the clustering results of the APCM algorithm, when it is initialized with $m_{ini}=5$, in cases where (a) $\alpha=0.5$, (b) $\alpha=1.0$ and (c) $\alpha=3.0$, respectively. Note from Table~\ref{table:synth2}, that these values of parameter $\alpha$ belong to the range where APCM identifies correctly the actual clusters, when $m_{ini}=5$. Also, in Fig.~\ref{example2}, it is shown how $\gamma_j$'s are affected when varying the parameter $\alpha$, after APCM is initialized with $m_{ini}=5$.

\begin{table}[htb!]
\centering
\caption{Range of values of the parameter $\alpha$, in which APCM concludes correctly to $m_{final}=3$ clusters, for specific values of $m_{ini}$ for experiment 2.}
{\small
\begin{tabular}{>{\centering\arraybackslash}m{0.10\linewidth} | >{\centering\arraybackslash}m{0.10\linewidth} |>{\centering\arraybackslash}m{0.10\linewidth}}
\hline  
	\centering $m_{ini}$  & \centering $\alpha_{min}$ & {\centering $\alpha_{max}$}\\
\hline
\centering 3 	& \centering 0.35 & {\centering 5.00} \\
\centering 5 	& \centering 0.33 & {\centering 3.08} \\
\centering 10  & \centering 0.28 & {\centering 1.38} \\
\centering 20  & \centering 0.23 & {\centering 0.90} \\
\centering 50  & \centering 0.17 & {\centering 0.36} \\
\centering 100 & \centering 0.15 & {\centering 0.29} \\
\hline
\end{tabular}}
\label{table:synth2}
\end{table}

\begin{figure}[htb!]
\centering
{\includegraphics[width=0.48\textwidth]{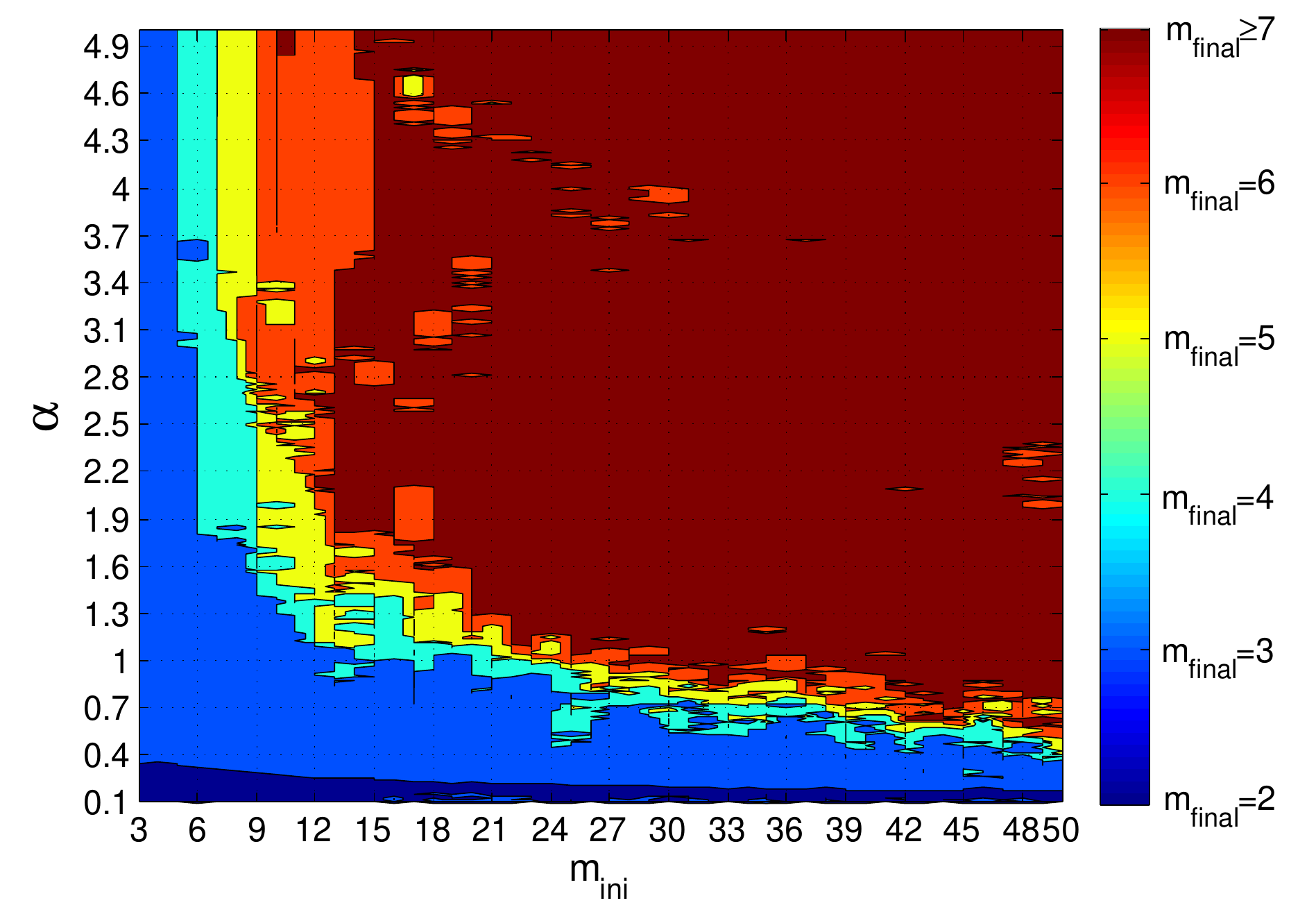}}
\hfil
\centering{\caption[]{Graphical representation of the number of final clusters, $m_{final}$, returned by APCM for experiment 2, for various combinations of $\alpha$ and $m_{ini}$ \footnotemark.}\label{example2b}}
\end{figure}

Executing APCM on the previous data set, for various values of $m_{ini}$ and $\alpha$, we end up with the figure shown in Fig.~\ref{example2b}, where regions in the $\alpha-m_{ini}$ plot are drawn with different colors, each one corresponding to a different number of final clusters, $m_{final}$. The light-blue colored region corresponds to the case where $m_{final}=3$, i.e., when APCM identifies correctly the underlying clusters. From the shape of this region, we can verify the ``rule of thumb" stated already in Section~\ref{subsec4}, that is, $\alpha$ is inversely related to $m_{ini}$. Moreover, from Fig.~\ref{example2b}, we deduce that by fixing $\alpha$ to a value arround $1$ and taking $m_{ini}$ $3-4$ times greater than the actual number of clusters, APCM will identify correctly the underlying physical clusters. Interestingly, the situation depicted in Fig.~\ref{example2b} has also been observed for several other data sets. Thus, the above rule of thumb seems to hold more generally.
\footnotetext{Note that for each value of $m_{ini}$ the same initial representatives (produced by FCM) have been used, for all values of $\alpha$. Results may differ slightly for different initializations of APCM.}

\subsection{Comparison of APCM with other algorithms}

\begin{table*}[htb!]
\centering
\caption{Performance of clustering algorithms for the experiment 3 data set.}
{\small
\begin{tabular}{>{\arraybackslash}m{0.41\linewidth} | >{\centering\arraybackslash}m{0.03\linewidth} |>{\centering\arraybackslash}m{0.05\linewidth}| >{\centering\arraybackslash}m{0.06\linewidth} |>{\centering\arraybackslash}m{0.06\linewidth} |>{\centering\arraybackslash}m{0.07\linewidth} |>{\centering\arraybackslash}m{0.05\linewidth}
|>{\centering\arraybackslash}m{0.05\linewidth}}
\hline  
	& \centering $m_{ini}$ & \centering $m_{final}$ & \centering $RM$ & \centering $SR$ & \centering $MD$ & {\centering $Iter$} & {\centering $Time$}\\
\hline
k-means & 3  & 3  & 91.02 & 86.74 & 6.8509 & 45 & 0.13\\
k-means & 8  & 8  & 73.83 & 42.22 & 2.5267 & 60 & 0.33\\
k-means & 10 & 10 & 71.41 & 34.52 & 2.3544 & 48 & 0.41\\
k-means & 15 & 15 & 68.35 & 27.00 & 0.8074 & 31 & 0.46\\
\hline
FCM & 3  & 3  & 82.05 & 65.39 & 4.2089 & 66 & 0.04\\
FCM & 8  & 8  & 71.88 & 36.91 & 2.5468 & 100 & 0.34\\
FCM & 10 & 10 & 69.67 & 28.74 & 2.3466 & 100 & 0.48\\
FCM & 15 & 15 & 67.18 & 21.96 & 0.8593 & 100 & 0.50\\
\hline
FCM \& XB & -  & 2  & 87.62 & 86.74 & 0.6346 & - & 3.60\\
\hline
PCM & 3  & 2 & 87.62 & 86.78 & 0.4778 & 10 & 0.14\\
PCM & 8  & 3 & 75.14 & 67.87 & 0.2138 & 26 & 0.59\\
PCM & 10 & 3 & 75.64 & 68.35 & 0.1918 & 23 & 0.74\\
PCM & 15 & 3 & 78.64 & 70.04 & 0.1877 & 41 & 1.06\\
\hline
APCM ($\alpha=1$)   & 3  & 2 & 87.73 & 86.83 & 0.0655 & 12 & 0.07\\
APCM ($\alpha=1.5$) & 8  & 3 & 90.83 & 90.04 & 0.2268 & 38 & 0.40\\
APCM ($\alpha=1$)   & 10 & 3 & 90.80 & 90.00 & 0.2131 & 28 & 0.52\\
APCM ($\alpha=1$)   & 15 & 3 & 90.83 & 90.04 & 0.2157 & 35 & 0.55\\
\hline
UPC ($q=2$) & 3  & 2 & 87.69 & 86.78 & 0.1331 & 20  & 0.07\\
UPC ($q=3$) & 8  & 4 & 90.04 & 85.96 & 0.5517 & 76  & 0.54\\
UPC ($q=3$) & 10 & 4 & 89.92 & 85.78 & 0.5829 & 89  & 0.57\\
UPC ($q=3$) & 15 & 4 & 89.79 & 85.61 & 0.6618 & 111 & 0.80\\
\hline
PFCM ($K=1$, $a=1$, $b=1$, $q=2$,   $n=2$) & 3  & 2 & 87.62 & 86.78 & 1.2927 & 25 & 0.07 \\
PFCM ($K=1$, $a=1$, $b=1$, $q=4$,   $n=2$) & 8  & 3 & 83.11 & 84.65 & 0.5595 & 55 & 0.47 \\
PFCM ($K=1$, $a=1$, $b=2$, $q=3$,   $n=2$) & 10 & 3 & 84.30 & 85.78 & 0.7517 & 119 & 0.74 \\
PFCM ($K=1$, $a=1$, $b=3$, $q=2.5$, $n=2$) & 15 & 3 & 86.74 & 87.70 & 0.8414 & 201 & 1.83\\
\hline
UPFC ($a=1$, $b=1$,   $q=4$, $n=2$) & 3  & 2 & 87.76 & 86.83 & 0.4588 & 20 & 0.08 \\
UPFC ($a=1$, $b=3$,   $q=3$, $n=2$) & 8  & 3 & 87.39 & 85.43 & 0.7260 & 85 & 0.49 \\
UPFC ($a=1$, $b=3$,   $q=3$, $n=2$) & 10 & 3 & 87.40 & 85.43 & 0.7364 & 101 & 0.68 \\
UPFC ($a=1$, $b=1.5$, $q=3$, $n=2$) & 15 & 3 & 87.64 & 85.91 & 0.5555 & 94 & 0.83 \\
\hline
GRPCM & - & 2 & 87.54 & 86.74 & 0.3611 & 90 & 148.03 \\
\hline
AMPCM & - & 2 & 87.54 & 86.74 & 0.3189 & 87 & 151.64 \\
\hline
\end{tabular}}
\label{table:synth3}
\end{table*}

In the sequel, we compare the clustering performance of APCM with that of the k-means, the FCM, the FCM with the XB validity index \cite{Xie91}, the PCM, the UPC \cite{Yang06}, the PFCM \cite{Pal05}, the UPFC \cite{Wu10}, the GRPCM \cite{Liao11} and the AMPCM \cite{Yang11} algorithms, which all result from cost optimization schemes. For a fair comparison, the representatives $\boldsymbol{\theta}_j$'s of all algorithms, except for GRPCM and AMPCM, are initialized based on the FCM scheme and the parameters of each algorithm are first fine-tuned. In order to compare a clustering with the true data label information, we use (a) the Rand Measure (RM) (e.g. \cite{Theo09}), which measures the degree of agreement between the obtained clustering and the true data classification and can handle clusterings whose number of clusters may differ from the number of true data labels, and (b) the Success Rate (SR), which measures the percentage of the points that have been correctly labeled by each algorithm. Moreover, the mean of the Euclidean distances (MD) between the true mean of each physical cluster $c_j$ and its closest cluster representative ($\boldsymbol{\theta}_j$) obtained by each algorithm, is given. In cases where a clustering algorithm ends up with a higher number of clusters than the actual one ($m_{final}>m$), only the $m$ cluster representatives that are closest to the true $m$ centers of the physical clusters, are taken into account in the determination of MD. On the other hand, in cases where $m_{final}<m$, the MD measure refers to the distances of all cluster representatives from their nearest actual center; thus some actual centers are ignored. It is noted that lower MD values indicate more accurate determination of the cluster center locations. Finally, the number of iterations and the time (in seconds) required for the convergence of each algorithm, are provided\footnote{In the FCM \& XB validity index only the total time required for the execution of FCM 19-times (for $m_{ini}=2,\ldots,20$) is given.}. Note that in all reported results for the UPC, the PFCM and the UPFC algorithms, clusters that coincide are considered as a single one. Moreover, for the FCM with XB validity index case, only the clustering obtained for the $m_{ini}$ that minimizes the XB index is given and discussed.

We begin with a demanding simulated data set with classes exhibiting significant differences with respect to their variance.

{\bf Experiment 3}: 
Consider a two-dimensional data set consisting of $N=2100$ points, where three natural clusters $C_1$, $C_2$ and $C_3$ are formed. Each such cluster is modelled by a normal distribution. The means of the distributions are $\mathbf{c}_1=[6.53, 1.39]^T$, $\mathbf{c}_2=[20.32, 20.39]^T$ and $\mathbf{c}_3=[28.09, 11.38]^T$, respectively, while their covariance matrices are set to $10\cdot I_2$, $20\cdot I_2$ and $1\cdot I_2$, respectively. A number of 1000 points are generated by each one of the first two distributions and 100 points are generated by the last one. Moreover, 200 data points are added randomly as noise in the region where data live (see Fig.~\ref{ex3dataset}). 

\begin{figure*}[htb!]
%\captionsetup{width=0.55\textwidth}
\centering
\subfloat[The data set]{\includegraphics[width=0.32\textwidth]{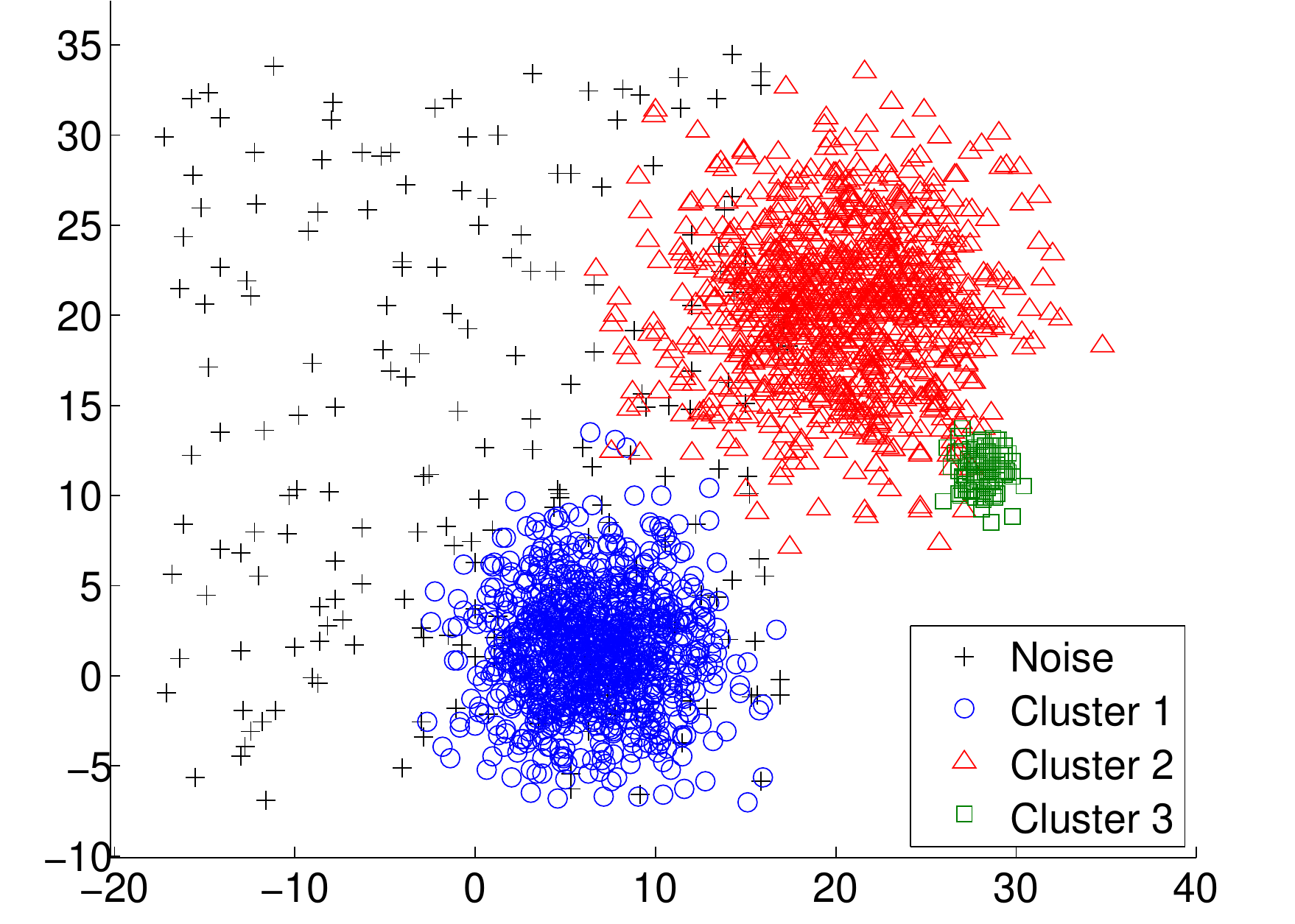}\vspace{20pt}\label{ex3dataset}}
\hfil
\centering
\subfloat[k-means]{\includegraphics[width=0.32\textwidth]{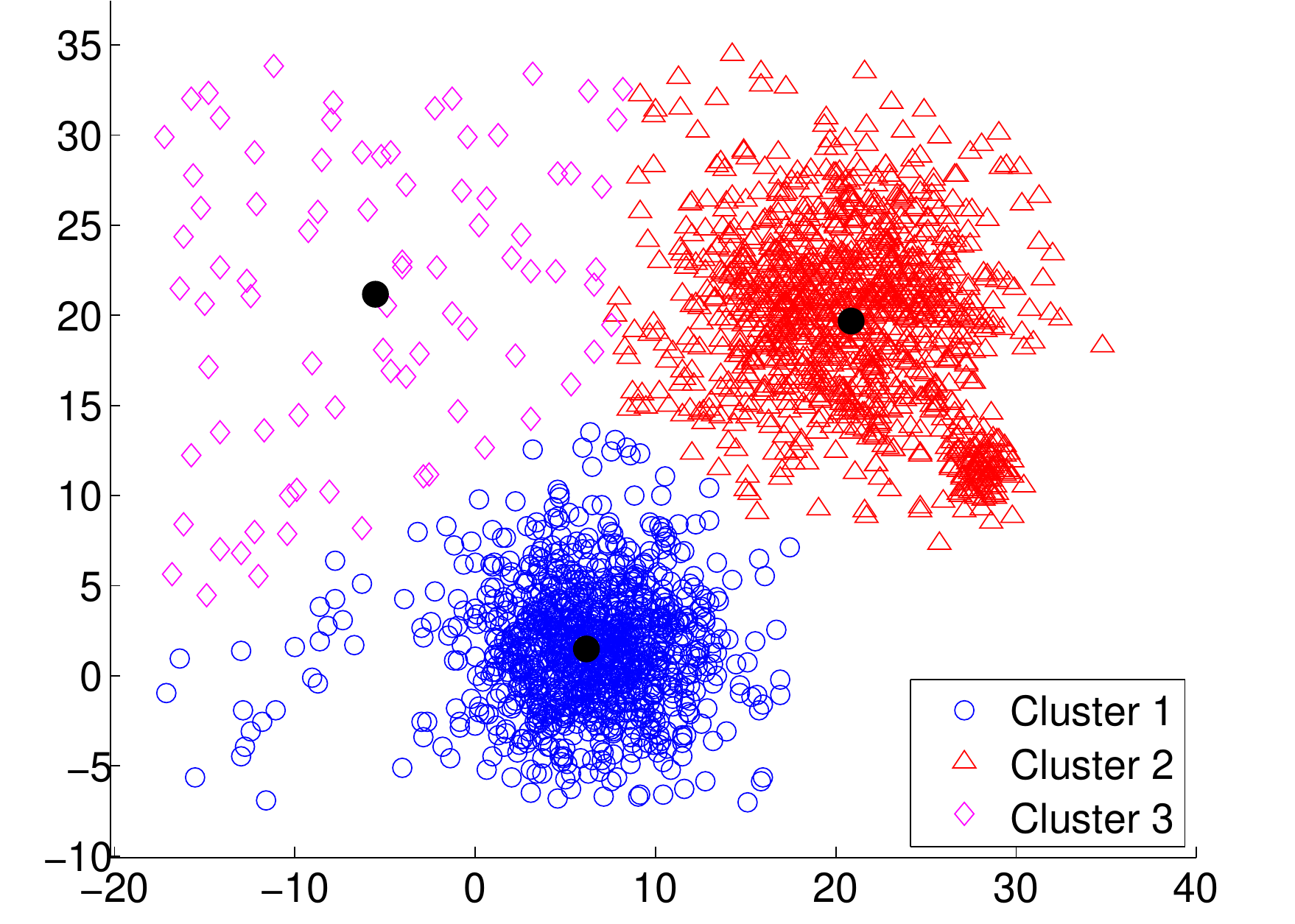}\hspace{6pt}\label{ex3kmeans}}
\hfil
\centering
\subfloat[FCM]{\includegraphics[width=0.32\textwidth]{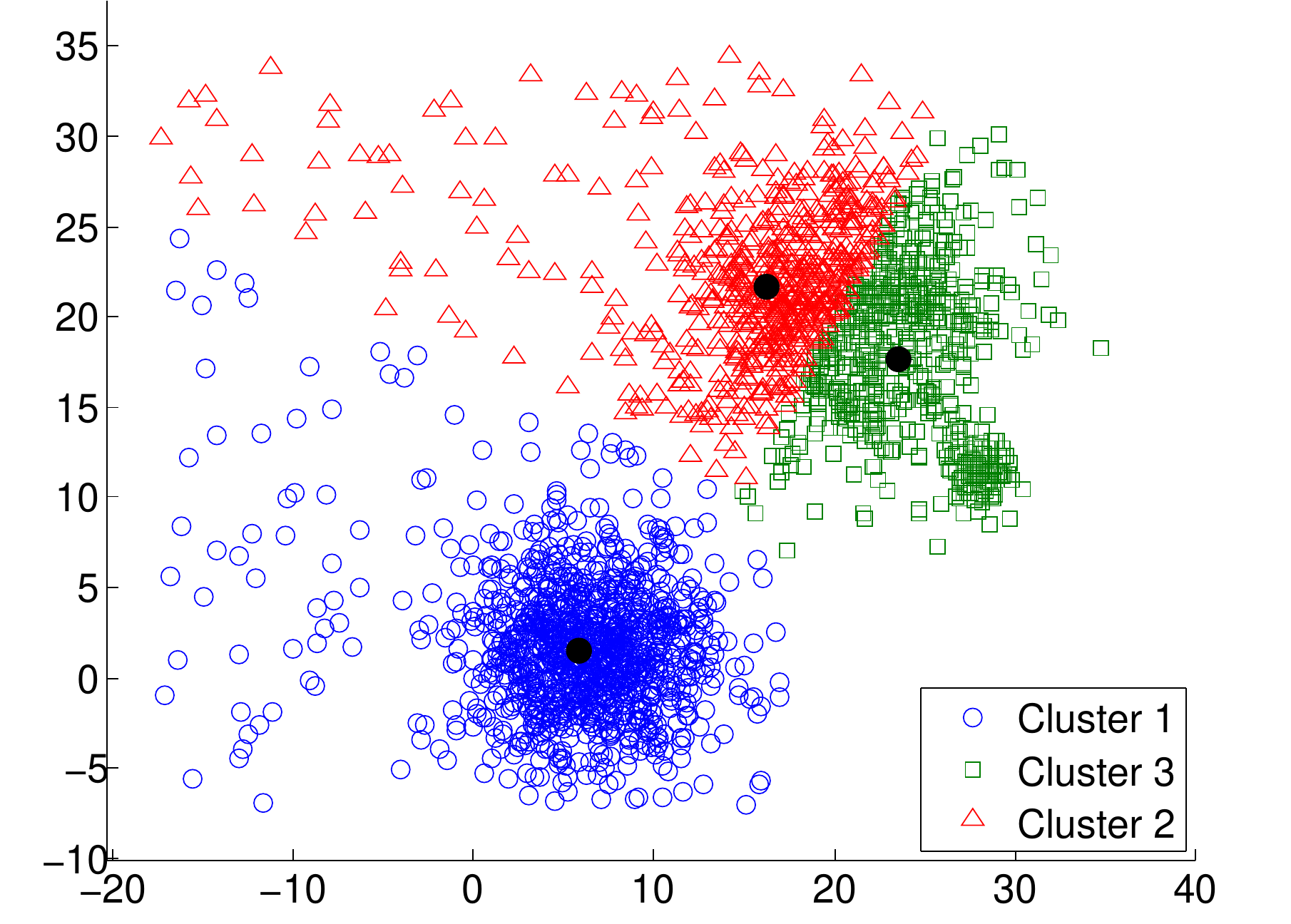}\hspace{6pt}\label{ex3fcm}}
\hfil
\centering
\subfloat[FCM \& XB]{\includegraphics[width=0.32\textwidth]{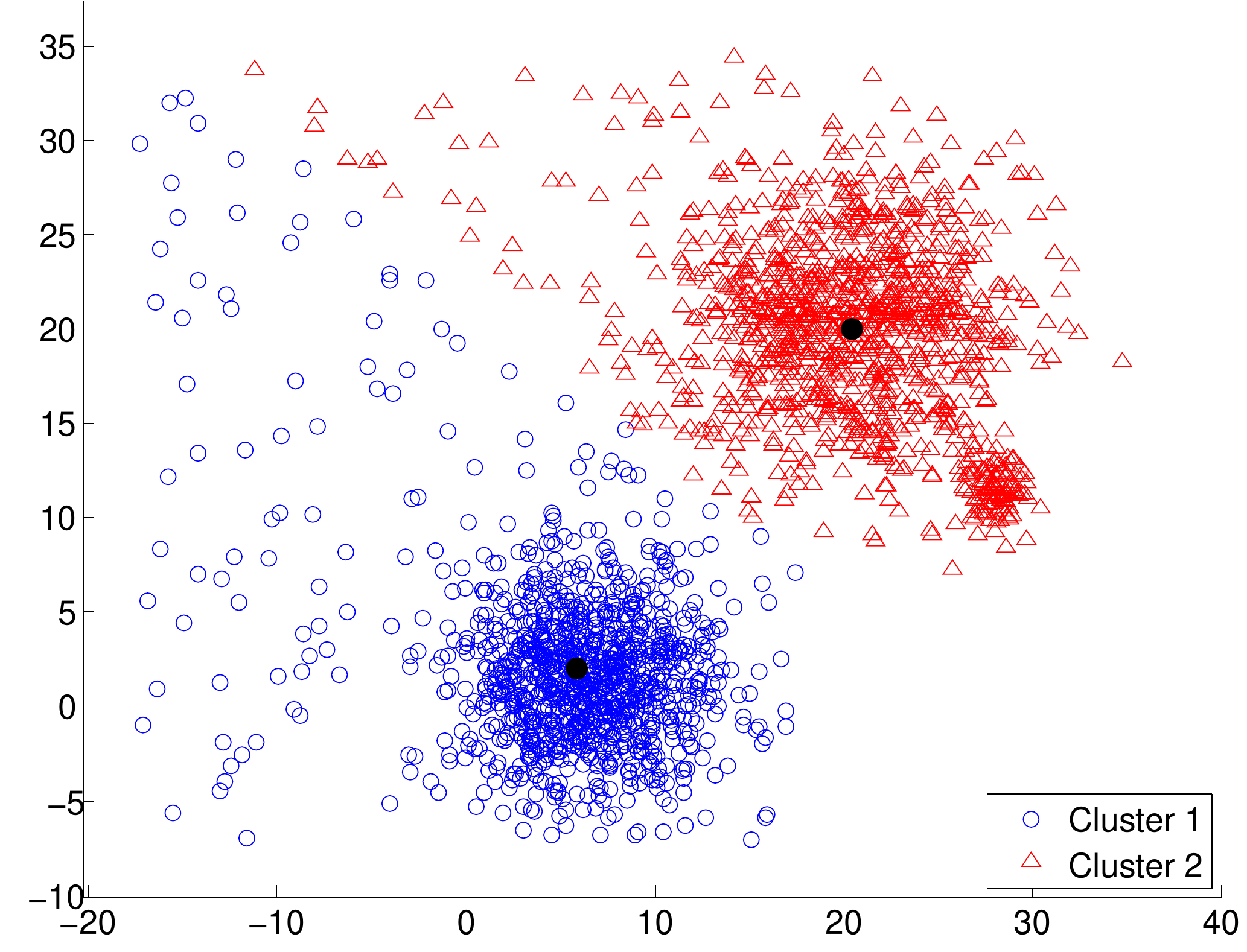}\hspace{6pt}\label{ex3fcmXB}}
\hfil
\centering
\subfloat[PCM]{\includegraphics[width=0.33\textwidth]{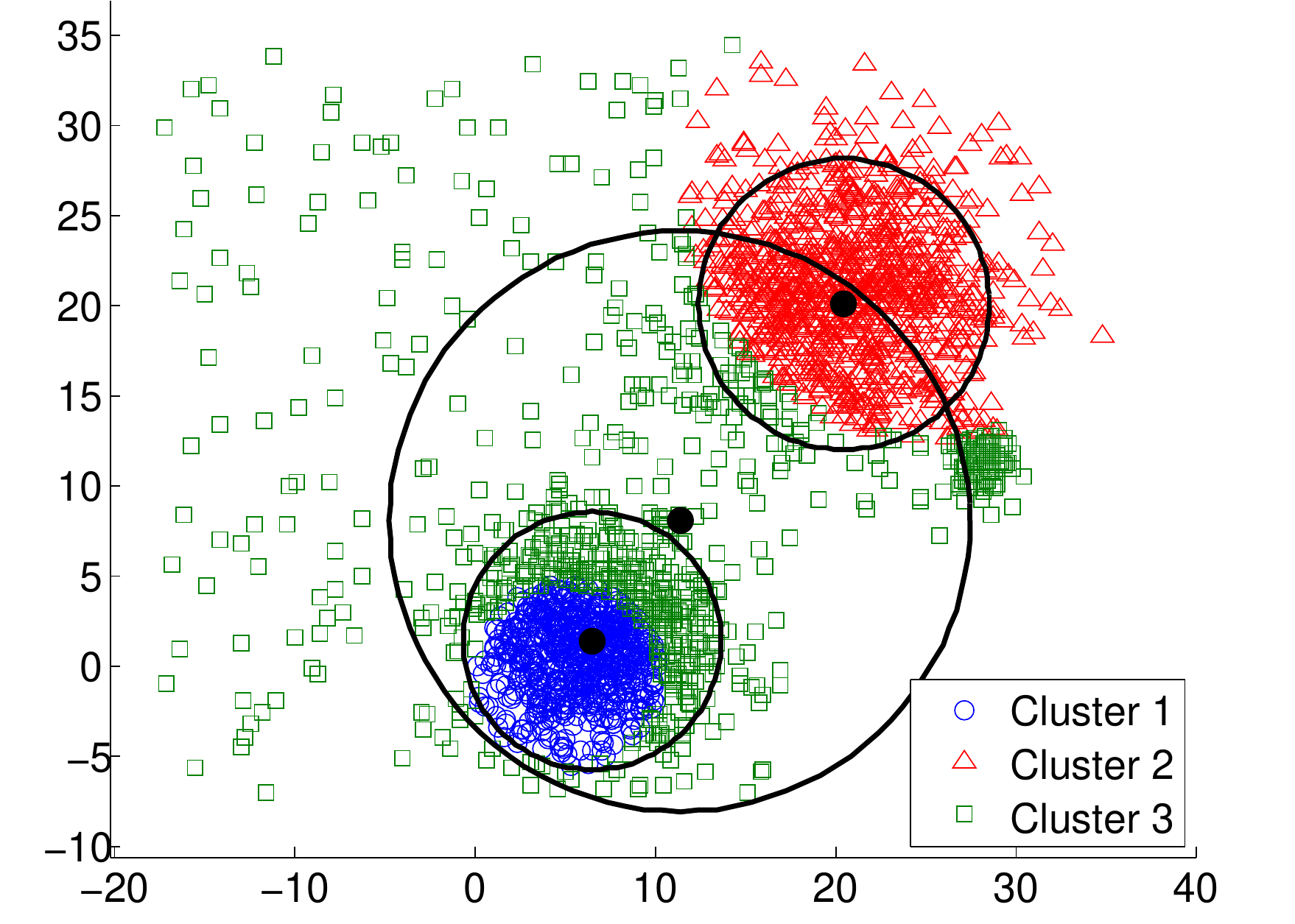}\hspace{-8pt}\label{ex3pcm}}
\hfil
\centering
\subfloat[APCM]{\includegraphics[width=0.33\textwidth]{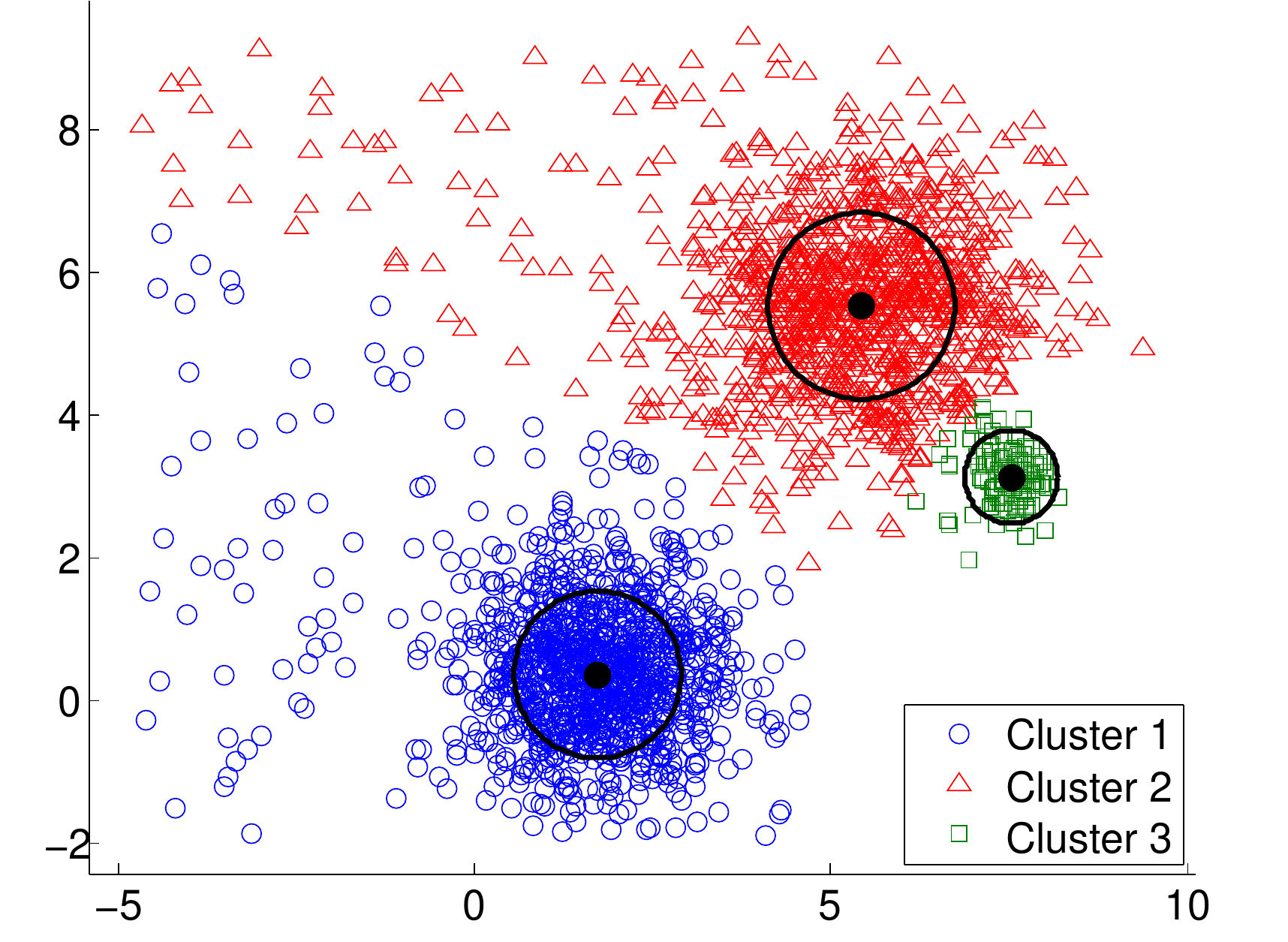}\hspace{6pt}\label{ex3apcm}}
\hfil
\centering
\subfloat[UPC]{\includegraphics[width=0.33\textwidth]{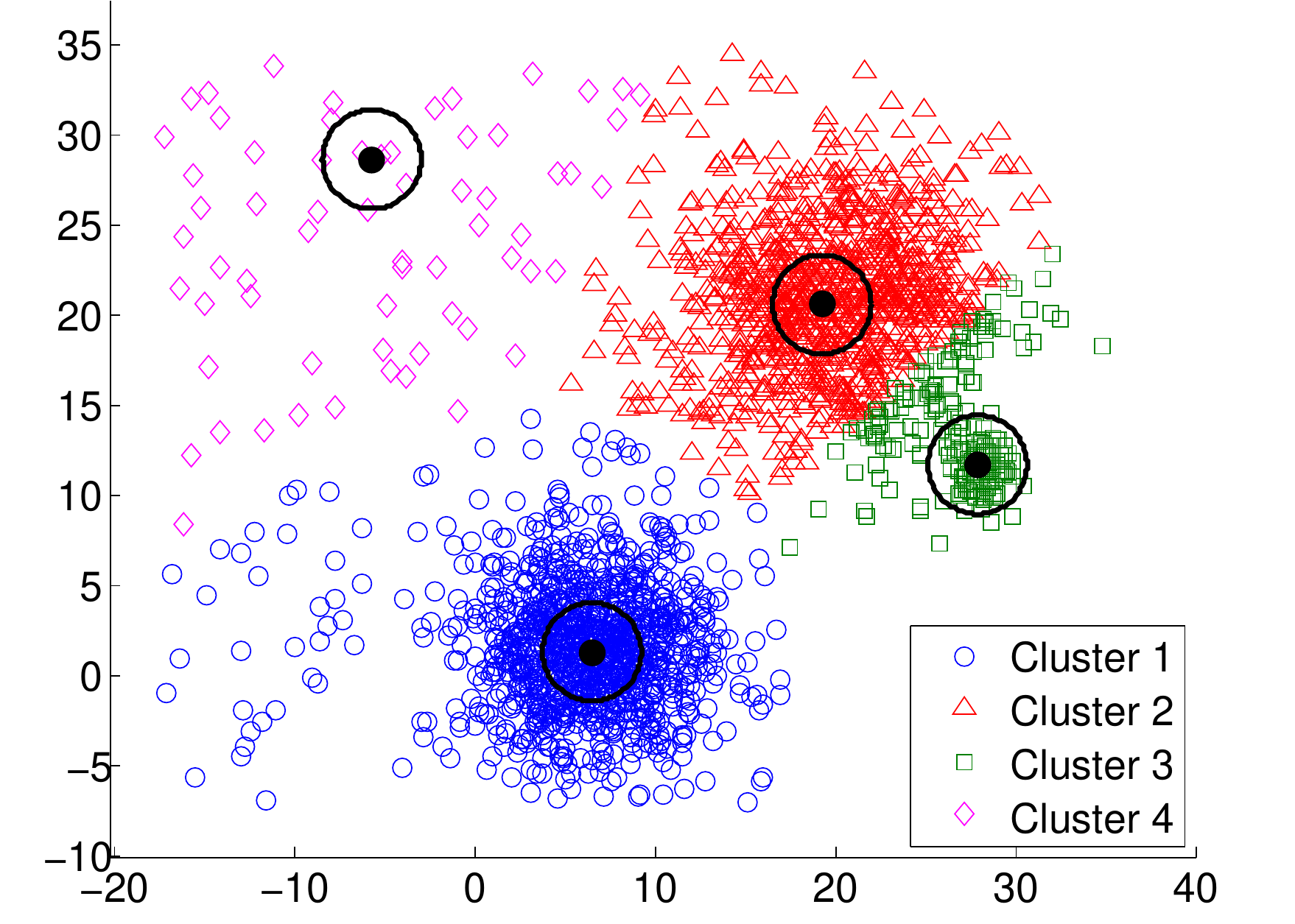}\hspace{-6pt}\label{ex3upc}}
\hfil
\centering
\subfloat[PFCM]{\includegraphics[width=0.33\textwidth]{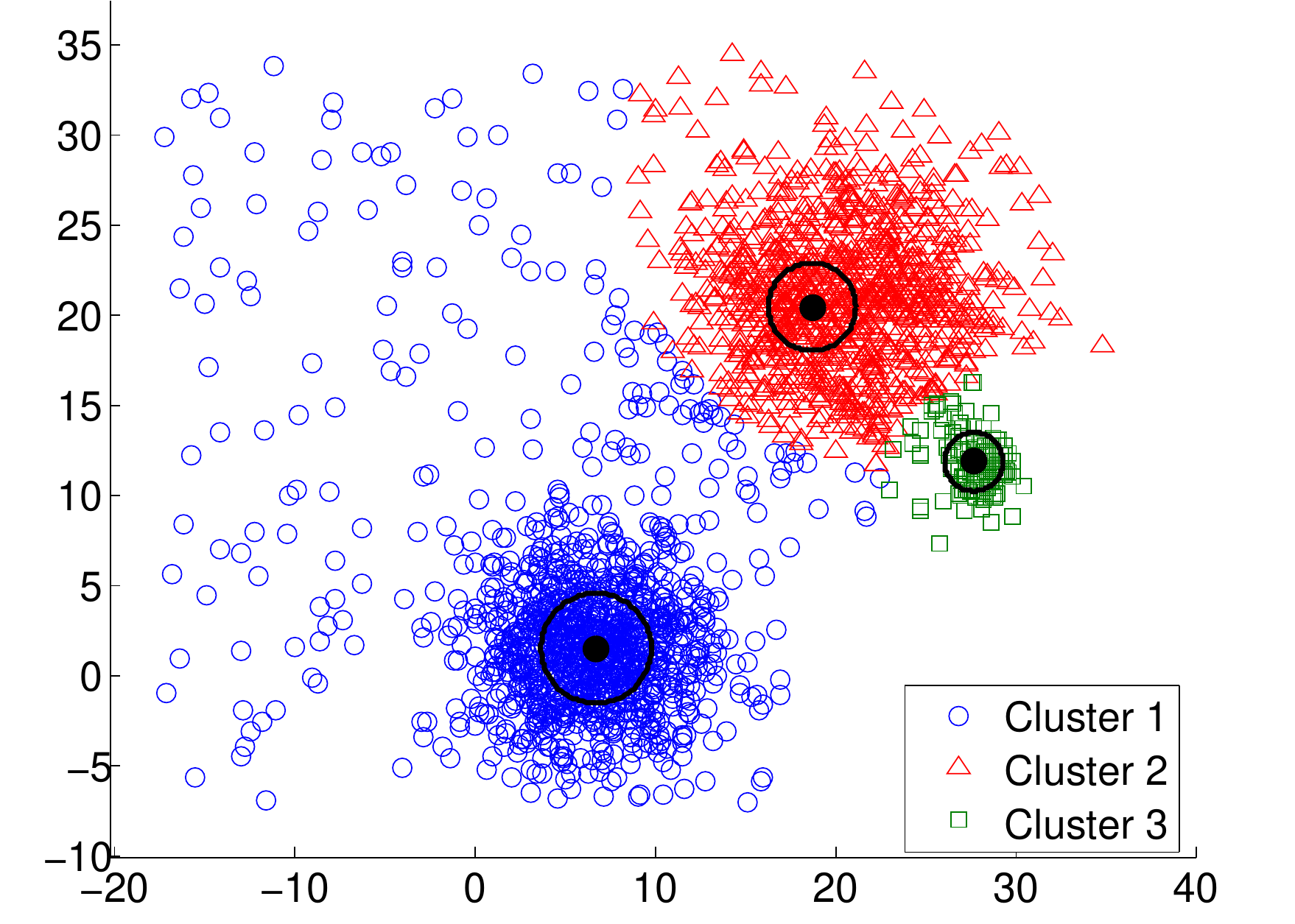}\hspace{-7pt}\label{ex3pfcm}}
\hfil
\centering
\subfloat[UPFC]{\includegraphics[width=0.33\textwidth]{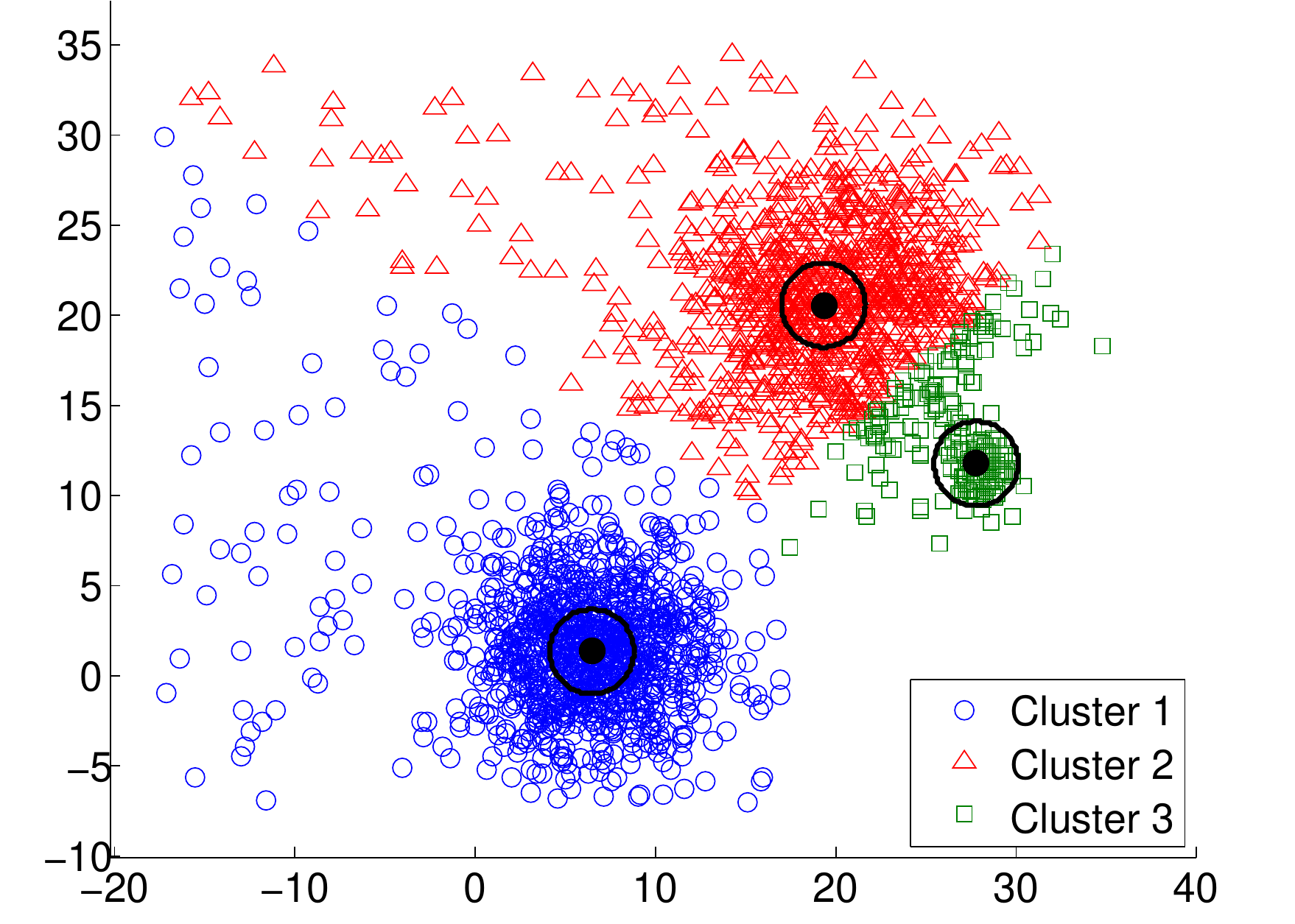}\hspace{-8pt}\label{ex3upfc}}
\hfil
\centering
\subfloat[GRPCM]{\includegraphics[width=0.33\textwidth]{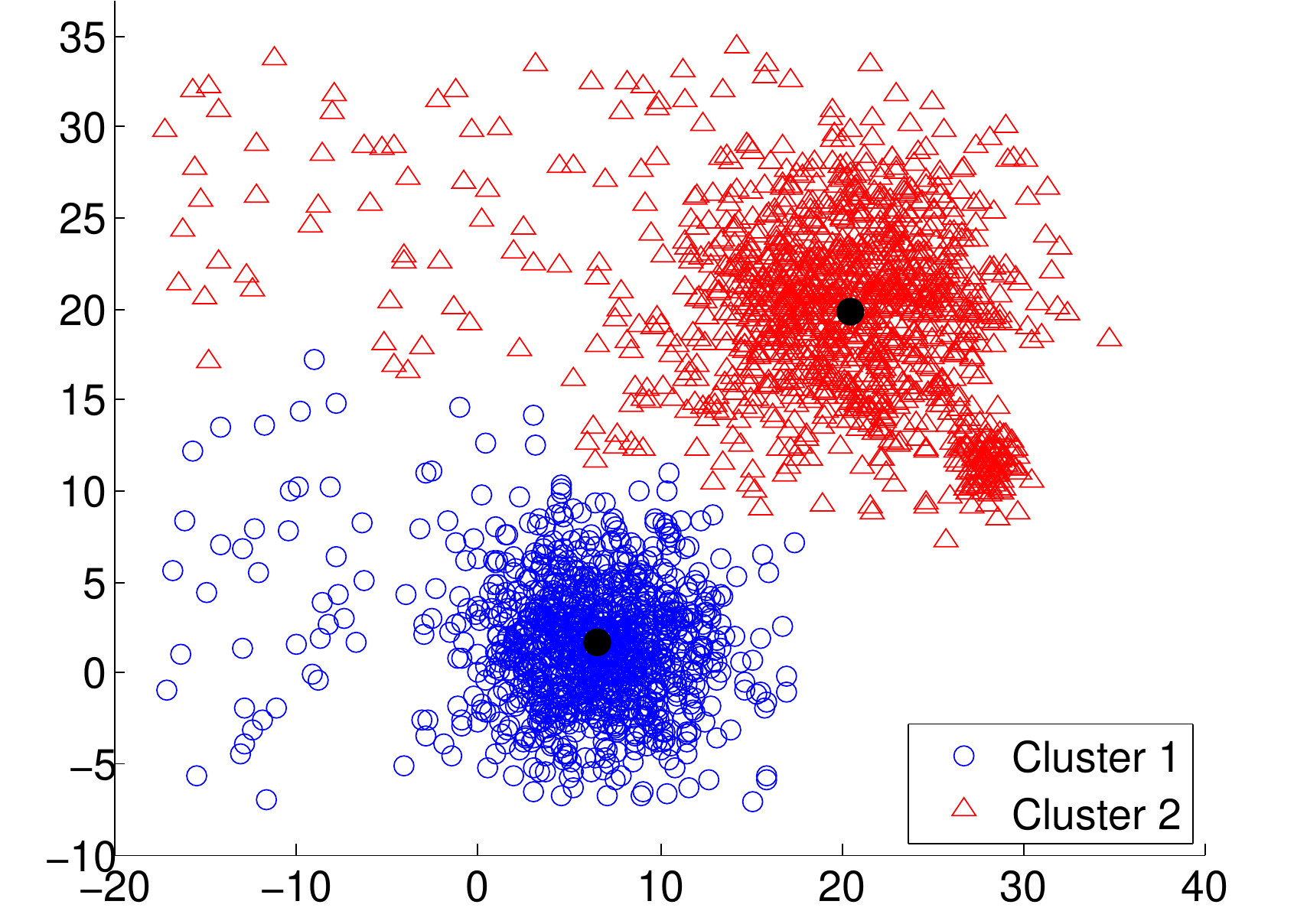}\hspace{-8pt}\label{ex3grpcm}}
\hfil
\centering
\subfloat[AMPCM]{\includegraphics[width=0.33\textwidth]{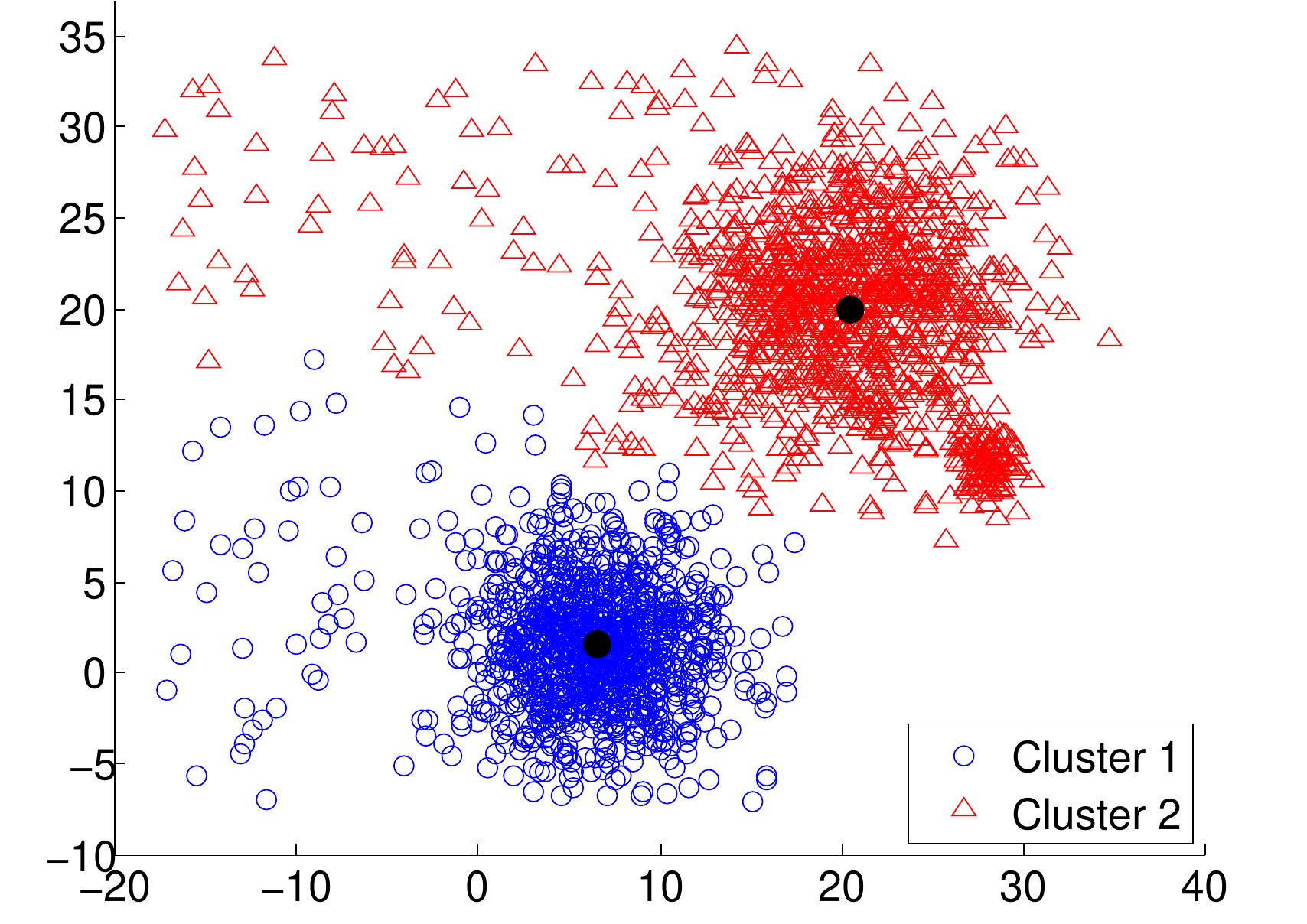}\hspace{-8pt}\label{ex3ampcm}}
\hfil
\centering{\caption{ (a) Data set of experiment 3. Clustering results for (b) k-means, $m_{ini}=3$, (c) FCM, $m_{ini}=3$, (d) FCM \& XB, (e) PCM, $m_{ini}=15$, (f) APCM, $m_{ini}=15$ and $\alpha=1$, (g) UPC, $m_{ini}=8$ and $q=3$, (h) PFCM, $m_{ini}=15$, $K=1$, $\alpha=1$, $\beta=3$, $q=2.5$ and $n=2$, (i) UPFC, $m_{ini}=15$, $\alpha=1$, $\beta=1.5$, $q=3$ and $n=2$, (j) GRPCM and (k) AMPCM.}\label{example3}}
\end{figure*}

Table~\ref{table:synth3} shows the clustering results of all algorithms, where $m_{ini}$ and $m_{final}$ denote the initial and the final number of the obtained clusters, respectively. Fig.~\ref{ex3kmeans} and Fig.~\ref{ex3fcm} show the clustering result obtained using the k-means and FCM algorithms, respectively, for $m_{ini}=3$. Figs.~\ref{ex3fcmXB},~\ref{ex3pcm},~\ref{ex3apcm},~\ref{ex3upc},~\ref{ex3pfcm},~\ref{ex3upfc},~\ref{ex3grpcm} and \ref{ex3ampcm} depict the performance of FCM \& XB, PCM, APCM, UPC, PFCM, UPFC, GRPCM and AMPCM respectively, with their parameters chosen as stated in the figure caption. In addition, the circles, centered at each $\boldsymbol{\theta}_j$ and having radius $\sqrt{\gamma_j}$ (as they have been computed after the convergence of the algorithms), are also drawn.

\begin{table*}[htb!]
\centering
\caption{Performance of clustering algorithms for the Iris data set.}
{\small
\begin{tabular}{>{\arraybackslash}m{0.41\linewidth} | >{\centering\arraybackslash}m{0.03\linewidth} |>{\centering\arraybackslash}m{0.05\linewidth}| >{\centering\arraybackslash}m{0.06\linewidth} |>{\centering\arraybackslash}m{0.06\linewidth} |>{\centering\arraybackslash}m{0.07\linewidth} |>{\centering\arraybackslash}m{0.05\linewidth}
|>{\centering\arraybackslash}m{0.05\linewidth}}
\hline  
	& \centering $m_{ini}$ & \centering $m_{final}$ & \centering $RM$ & \centering $SR$ & \centering $MD$ & {\centering $Iter$} & {\centering $Time$}\\
\hline
k-means & 3   & 3   & 87.97 & 89.33 & 0.1271 & 3 & 0.30\\
k-means & 10  & 10  & 76.64 & 40.00 & 0.7785 & 4 & 0.13\\
\hline
FCM & 3  & 3   & 87.97 & 89.33 & 0.1287 & 19 & 0.02\\
FCM & 10 & 10  & 76.16 & 36.00 & 0.7793 & 35 & 0.02\\
\hline
FCM \& XB & -  & 2  & 76.37 & 66.67 & 0.3986 & - & 0.16\\
\hline
PCM & 3  & 2 & 77.19 & 66.67 & 0.3563 & 19 & 0.11\\
PCM & 10 & 2 & 77.63 & 66.67 & 0.3488 & 28 & 0.11\\
\hline
APCM ($\alpha=3$) & 3   & 3 & 91.24 & 92.67 & 0.1406 & 26 & 0.06 \\
APCM ($\alpha=1$) & 10  & 3 & 84.15 & 84.67 & 0.4030 & 67 & 0.09\\
\hline
UPC ($q=4$)   & 3   & 3 & 91.24 & 92.67 & 0.1438 & 26 & 0.03 \\
UPC ($q=2.4$) & 10  & 3 & 81.96 & 81.33 & 0.5569 & 150 & 0.11\\
\hline
PFCM ($K=1$, $a=1$, $b=10$,  $q=7$, $n=2$) & 3  & 3 & 90.55 & 92.00 & 0.1833 & 17 & 0.03\\
PFCM ($K=1$, $a=1$, $b=1.5$, $q=2$, $n=2$) & 10 & 3 & 84.64 & 85.33 & 0.5411 & 92 & 0.05 \\
\hline
UPFC ($a=1$, $b=5$,  $q=4$,   $n=2$) & 3  & 3 & 91.24 & 92.67 & 0.1642 & 32 & 0.03 \\
UPFC ($a=1$, $b=1.5$, $q=2.5$, $n=2$) & 10 & 3 & 81.96 & 81.33 & 0.5566 & 180 & 0.16\\
\hline
GRPCM  & -  & 2 & 77.63 & 66.67 & 0.3675 & 26 & 0.47\\
\hline
AMPCM  & -  & 2 & 77.63 & 66.67 & 0.3643 & 28 & 0.47\\
\hline
\end{tabular}}
\label{table:real1}
\end{table*}

\begin{table*}[htb!]
\centering
\caption{Performance of clustering algorithms for the New Thyroid data set.}
{\small
\begin{tabular}{>{\arraybackslash}m{0.41\linewidth} | >{\centering\arraybackslash}m{0.03\linewidth} |>{\centering\arraybackslash}m{0.05\linewidth}| >{\centering\arraybackslash}m{0.06\linewidth} |>{\centering\arraybackslash}m{0.06\linewidth} |>{\centering\arraybackslash}m{0.07\linewidth} |>{\centering\arraybackslash}m{0.05\linewidth}
|>{\centering\arraybackslash}m{0.05\linewidth}}
\hline  
	& \centering $m_{ini}$ & \centering $m_{final}$ & \centering $RM$ & \centering $SR$ & \centering $MD$ & {\centering $Iter$} & {\centering $Time$}\\
\hline
k-means & 3  & 3  & 79.65 & 87.44 & 0.8949 & 3 & 0.16\\
k-means & 5  & 5  & 70.78 & 63.72 & 0.8548 & 12 & 0.14\\
k-means & 15 & 15 & 55.01 & 25.12 & 0.7159 & 16 & 0.17\\
\hline
FCM & 3  & 3  & 83.29 & 89.77 & 0.4385 & 53 & 0.02\\
FCM & 5  & 5  & 60.32 & 46.98 & 1.0785 & 55 & 0.02\\
FCM & 15 & 15 & 52.83 & 21.86 & 0.8816 & 91 & 0.11\\
\hline
FCM \& XB & -  & 3  & 83.29 & 89.77 & 0.4385 & - & 0.44\\
\hline
PCM & 3  & 1 & 53.05 & 69.77 & 0.1177 & 7 & 0.06\\
PCM & 5  & 1 & 53.05 & 69.77 & 0.0559 & 7 & 0.06\\
PCM & 15 & 1 & 53.05 & 69.77 & 0.0577 & 8 & 0.16\\
\hline
APCM ($\alpha=8$)   & 3  & 3 & 94.58 & 96.74 & 0.7231 & 30 & 0.08\\
APCM ($\alpha=3$)   & 5  & 3 & 87.59 & 92.56 & 1.0026 & 21 & 0.06\\
APCM ($\alpha=1.2$) & 15 & 3 & 73.73 & 83.72 & 2.7123 & 54 & 0.16\\
\hline
UPC ($q=3$) & 3  & 3 & 83.85 & 90.23 & 0.6982 & 41 & 0.03 \\
UPC ($q=2$) & 5  & 3 & 77.94 & 86.51 & 1.0739 & 16 & 0.02\\
UPC ($q=1$) & 15 & 3 & 67.21 & 79.53 & 2.7617 & 34 & 0.05\\
\hline
PFCM ($K=1$, $a=1$, $b=5$, $q=8$, $n=2$) & 3  & 1 & 53.05 & 69.77 & 0.0507 & 15 & 0.03\\
PFCM ($K=1$, $a=1$, $b=5$, $q=8$, $n=2$) & 5  & 2 & 64.95 & 77.21 & 1.3855 & 41  & 0.05\\
PFCM ($K=1$, $a=1$, $b=8$, $q=2$, $n=2$) & 15 & 3 & 66.64 & 79.07 & 1.8381 & 28 & 0.09\\
\hline
UPFC ($a=1$, $b=5$,   $q=8$,   $n=2$) & 3  & 2 & 68.21 & 79.53 & 0.4108 & 21 & 0.05 \\
UPFC ($a=1$, $b=3$,   $q=6$,   $n=2$) & 5  & 3 & 78.76 & 86.98 & 0.9682 & 27 & 0.05 \\
UPFC ($a=1$, $b=0.1$, $q=1.5$, $n=2$) & 15 & 3 & 72.85 & 83.26 & 1.5909 & 34 & 0.09\\
\hline
GRPCM  & -  & 1 & 53.05 & 69.77 & 0.2732 & 63 & 2.04\\
\hline
AMPCM  & - & 1 & 53.05 & 69.77 & 0.2667 & 64 & 1.98\\
\hline
\end{tabular}}
\label{table:real2}
\end{table*}

As it can be deduced from Fig.~\ref{example3} and Table.~\ref{table:synth3}, even when the k-means and the FCM are initialized with the (unknown in practice) true number of clusters ($m=3$), they fail to unravel the underlying clustering structure, most probably due to the noise encountered in the data set and the big difference in the variances between nearby clusters. The FCM \& XB validity index and the classical PCM also fail to detect the cluster with the smallest variance. On the other hand, the proposed APCM algorithm produces very accurate results for various initial values of $m_{ini}$, detecting with high accuracy the center of the actual clusters (see MD measure in Table~\ref{table:synth3}). The UPC algorithm has been exhaustively fine tuned so that the parameters $\gamma_j$'s, which remain fixed during its execution and are the same for all clusters, get small enough values, in order to identify the cluster with the smallest variance ($C_3$). However, under these circumstances, a representative that is initially placed at the region where only noisy points exist (due to bad initialization from FCM), is trapped there and cannot be moved towards a dense region (due to the small value of its $\gamma_j$). Thus, UPC concludes to 4 clusters when $q=3$, but if we set $q=2$, UPC will conclude to 2 clusters, identifying $C_1$ and $C_2$ and missing $C_3$. The PFCM and UPFC algorithms constantly produce 3 clusters, at the cost of a computationally demanding fine tuning of the (several) parameters they involve. However, even when their parameters are fine tuned, the final estimates of $\boldsymbol{\theta}_j$'s are not closely located to the true cluster centers (see MD measure in Table~\ref{table:synth3}). The GRPCM and AMPCM algorithms conclude to two clusters, failing to unravel the underlying clustering structure. It is worth noting that these two algorithms require too much time to converge, mainly due to the way they perform cluster elimination. Finally, as it is deduced from Table~\ref{table:synth3}, the APCM algorithm achieves the best RM and SR results, detecting more accurately the true centers of the clusters (minimum MD), while, in addition, it requires the fewest iterations for convergence. It is worth noting that the operation time of APCM is less than that of PCM, even when APCM requires more iterations than PCM to converge. This is because the APCM iterations become ``lighter" as the algorithm evolves, since several clusters are eliminated.

%In the sequel, we consider real data sets, i.e., Iris, New Thyroid
The last three experiments are conducted on the basis of real world data sets.

\begin{table*}[htb!]
\centering
\caption{Performance of clustering algorithms for the Salinas HSI data set.}
{\small
\begin{tabular}{>{\arraybackslash}m{0.41\linewidth} | >{\centering\arraybackslash}m{0.03\linewidth} |>{\centering\arraybackslash}m{0.05\linewidth}| >{\centering\arraybackslash}m{0.04\linewidth} |>{\centering\arraybackslash}m{0.04\linewidth} |>{\centering\arraybackslash}m{0.075\linewidth} |>{\centering\arraybackslash}m{0.025\linewidth}
|>{\centering\arraybackslash}m{0.07\linewidth}}
\hline  
	& \centering $m_{ini}$ & \centering $m_{final}$ & \centering $RM$ & \centering $SR$ & \centering $MD$ & {\centering $Iter$} & {\centering $Time$} \\
\hline
k-means & 8  & 8  & 93.07 & 77.12 & 0.54e+03 & 11 & 0.11e+02\\
k-means & 20 & 20 & 96.46 & 69.89 & 1.03e+03 & 23 & 1.61e+02\\
k-means & 30 & 30 & 95.94 & 63.29 & 1.17e+03 & 44 & 6.43e+02 \\
\hline
FCM & 8  & 8  & 97.39 & 85.96 & 0.56e+03 & 47 & 0.11e+02\\
FCM & 20 & 20 & 96.33 & 67.28 & 0.89e+03 & 312 & 1.88e+02\\
FCM & 30 & 30 & 95.80 & 61.14 & 0.94e+03 & 827 & 7.41e+02\\
\hline
FCM \& XB & -  & 9  & 91.95 & 70.17 & 2.47e+03 & - & 1.03e+03\\
\hline
PCM & 8  & 4 & 91.51 & 67.22 & 0.94e+03 & 90 & 0.48e+02\\
PCM & 20 & 6 & 94.70 & 74.21 & 0.71e+03 & 98 & 2.49e+02\\
PCM & 30 & 6 & 94.65 & 74.02 & 0.74e+03 & 65 & 8.04e+02\\
\hline
APCM ($\alpha=4$)   & 8  & 8 & 97.35 & 85.42 & 0.76e+03 & 108 & 0.50e+02\\
APCM ($\alpha=2$)   & 20 & 9 & 97.64 & 88.17 & 0.59e+03 & 121 & 2.45e+02\\
APCM ($\alpha=2$) & 30 & 9 & 97.64 & 88.17 & 0.60e+03 & 137 & 7.91e+02\\
\hline
UPC ($q=3$) & 8  & 5 & 95.01 & 77.83 & 0.56e+03 & 48  & 0.23e+02\\
UPC ($q=3$) & 20 & 6 & 96.32 & 81.35 & 0.45e+03 & 44 &  2.50e+02\\
UPC ($q=3$) & 30 & 6 & 96.32 & 81.34 & 0.37e+03 & 47  & 7.18e+02\\
\hline
PFCM ($K=1$, $a=1$, $b=6$, $q=2$, $n=2$) & 8  & 6  & 96.60 & 81.15 & 0.62e+03 & 193 & 0.88e+02\\
PFCM ($K=1$, $a=1$, $b=1$, $q=3$, $n=2$) & 20 & 7  & 97.69 & 88.97 & 0.49e+03 & 151 & 3.24e+02\\
PFCM ($K=1$, $a=1$, $b=1$, $q=4$, $n=2$) & 30 & 7 & 97.59 & 89.44 & 0.56e+03 & 201 & 1.02e+03\\
\hline
UPFC ($a=1$, $b=8$, $q=4$, $n=2$) & 8  & 6 & 96.31 & 81.33 & 0.34e+03 & 61 & 0.37e+02\\
UPFC ($a=1$, $b=5$, $q=5$, $n=2$) & 20 & 6 & 96.31 & 81.33 & 0.40e+03 & 61 & 2.29e+02 \\
UPFC ($a=1$, $b=5$, $q=5$, $n=2$) & 30 & 6 & 96.31 & 81.33 & 0.32e+03 & 136 & 9.18e+02\\
\hline
GRPCM & - & 6 & 90.03 & 70.97 & 0.48e+03 & 142 & 2.79e+04\\
\hline
AMPCM & - & 6 & 90.03 & 70.97 & 0.48e+02 & 145 & 2.85e+04\\
\hline
\end{tabular}}
\label{table:HSI}
\end{table*}

{\bf Experiment 4}:
Let us consider the Iris data set (\cite{UCILib}) consisting of $N=150$, $4$-dimensional data points that form three classes, each one having 50 points. In this data set, two classes are overlapped, thus one can argue whether the true number of clusters $m$ is 2 or 3. As it is shown in Table~\ref{table:real1}, k-means and FCM work well, only if they are initialized with the true number of clusters ($m_{ini}=3$). The FCM \& XB and the classical PCM fail to end up with $m_{final}=3$ clusters, independently of the initial number of clusters. The same result holds for the GRPCM and the AMPCM algorithms. On the contrary, the APCM, the UPC, the PFCM and the UPFC algorithms, after appropriate fine tuning of their parameters, produce very accurate results in terms of RM, SR and MD. However, the APCM algorithm detects more accurately the centers of the true clusters (in most cases), compared to the other algorithms. It is noted again that the main drawback of the PFCM and the UPFC algorithms is the requirement for fine tuning of several parameters, which increases excessively the computational load required for detecting the appropriate combination of parameters that achieves the best clustering performance.

{\bf Experiment 5}:
Let us consider now the so-called New Thyroid three-class data set (\cite{UCILib}) consisting of $N=215$, $5$-dimensional data points. The experimental results for all algorithms are shown in Table~\ref{table:real2}. It can be seen that both k-means and FCM provide satisfactory results, only if they are initialized with the true number of clusters ($m_{ini}=3$), and the XB validity index is correctly minimized for $m_{ini}=3$, thus FCM \& XB concludes to the same results as FCM for $m_{ini}=3$, however at the cost of increased computational time. The classical PCM exhibits degraded performance, for all choices of $m_{ini}$. Similar to PCM behavior is observed for the GRPCM and the AMPCM algorithms, which fail to distinguish any clustering structure. On the contrary, the APCM and UPC algorithms detect the actual number of clusters independently of $m_{ini}$ after appropriate fine tuning of their parameters. However, again the APCM algorithm constantly produces higher RM and SR values. Finally, the PFCM and UPFC exhibit (a) inferior performance compared to APCM and UPC and (b) superior performance with respect to k-means and FCM provided that the latter are not intialized with the correct number of clusters.

In the next experiment we assess the performance of APCM and all the algorithms considered before in a high-dimensional data set.

{\bf Experiment 6}:
In this experiment a hyperspectral image (HSI) data set is considered, which depicts a subscene of the flightline acquired by the AVIRIS sensor over Salinas Valley, California \cite{HsiSal}. The AVIRIS sensor generates 224 bands across the spectral range from 0.2 to 2.4 $\mu m$. The number of bands is reduced to 204 by removing 20 water absorption bands. The aim in this experiment is to identify homogeneous regions in the Salinas HSI. Thus, the dimensionality of the problem is 204. Fig.~\ref{salinas_5thPCA} shows the 5th principal component (PC) of this HSI. Also, for lighten the required computational load, we select a spatial region of size 150x150 from the whole image. Thus, a total size of $N=22500$ samples-pixels are used, stemming from 8 ground-truth classes: ``Corn", two types of ``Broccoli", four types of ``Lettuce" and ``Grapes", denoted by different colors in Fig.~\ref{salinas_gt}.  Note that there is no available ground truth information for the dark blue pixels in Fig.~\ref{salinas_gt}. It is also noted that Fig.~\ref{example6} depicts the best mapping obtained by each algorithm taking into account not only the ``dry" performance indices but also its physical interpretation (see \cite{Xen14}).

\begin{figure*}[htb!]
%\begin{FPfigure}
%\centering
%{\includegraphics[width=0.95\textwidth]{salinas_legend}\label{salinas_legend}}
%\hfil
\centering
\subfloat[The {\small 5th PC} component]{\includegraphics[width=0.29\textwidth]{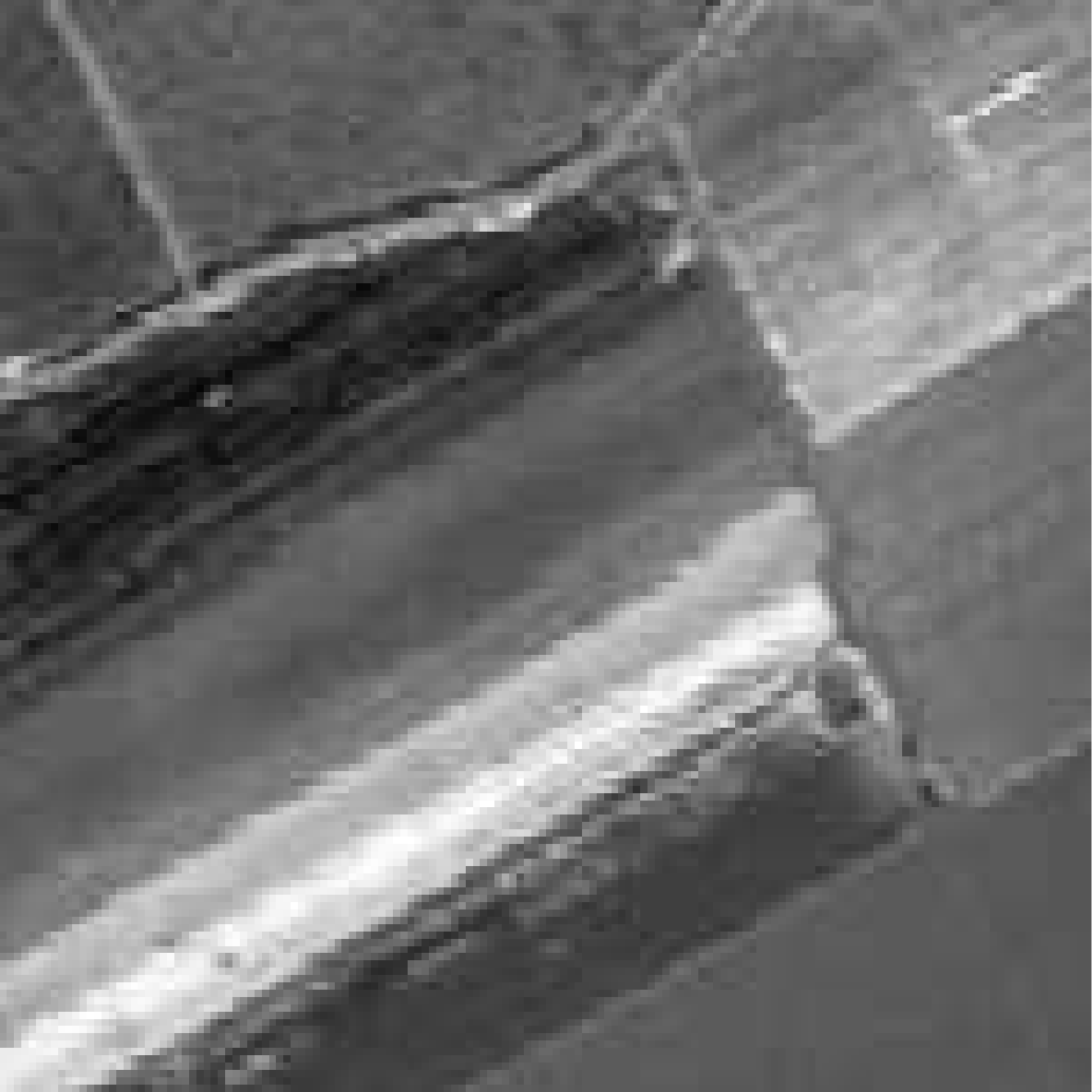}\label{salinas_5thPCA}}
\hfil
\centering
\subfloat[The ground truth]{\includegraphics[width=0.29\textwidth]{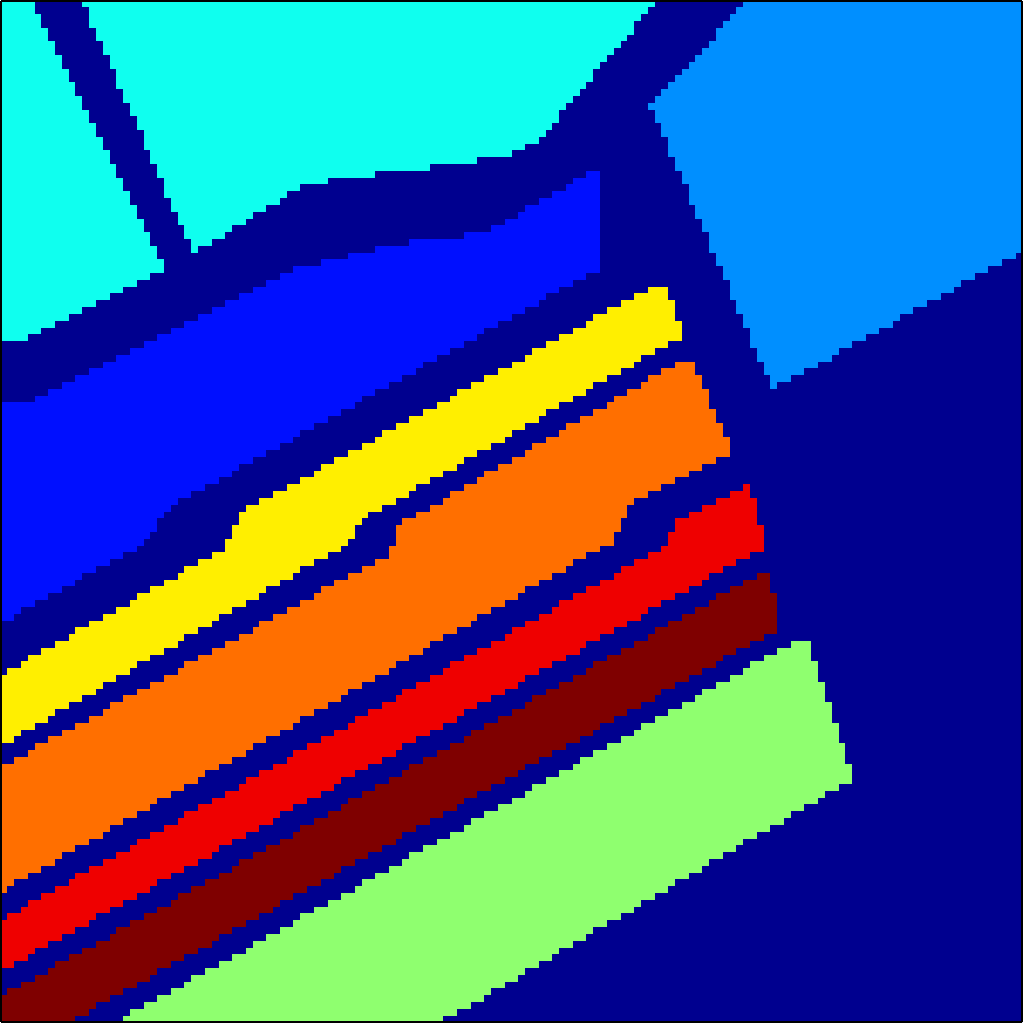}\label{salinas_gt}}
\hfil
\centering
\subfloat[k-means]{\includegraphics[width=0.29\textwidth]{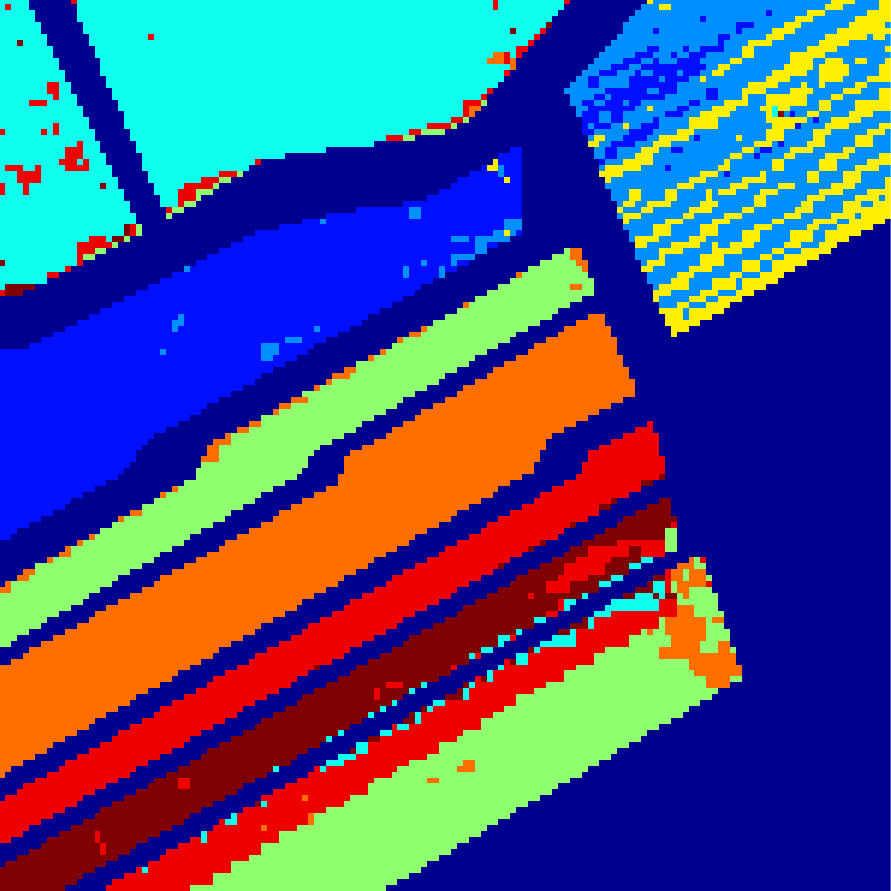}\label{salinas_kmeans_m8}}
\hfil
\centering
\subfloat[FCM]{\includegraphics[width=0.29\textwidth]{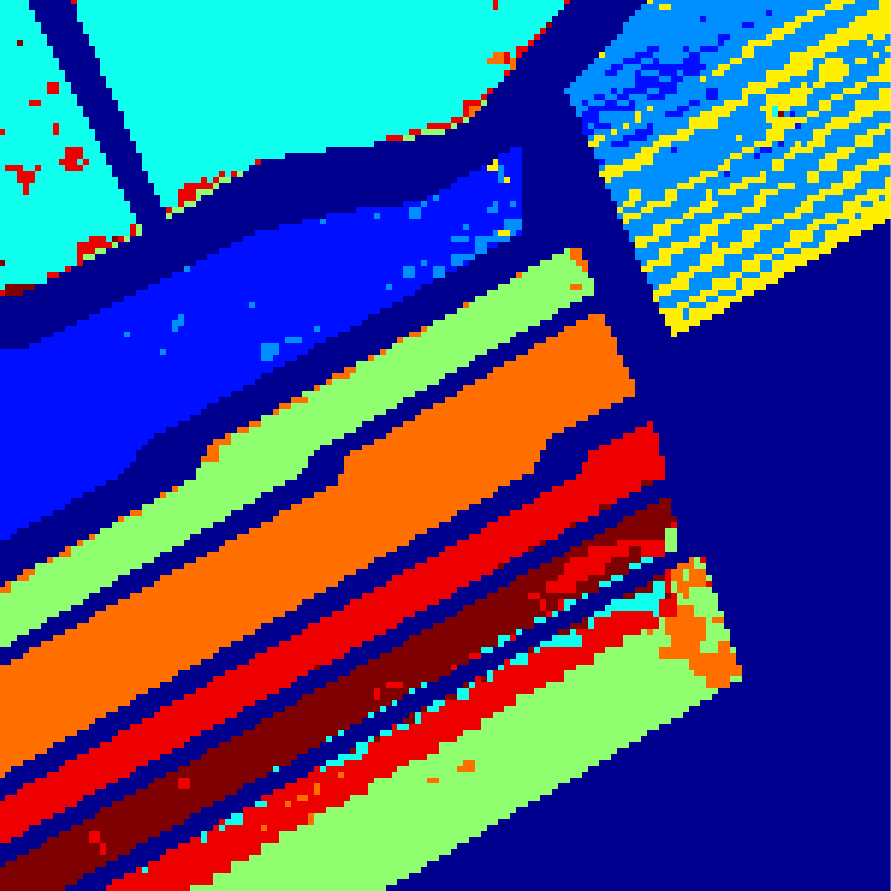}\label{salinas_FCM_m8}}
\hfil
\centering
\subfloat[FCM \& XB]{\includegraphics[width=0.29\textwidth]{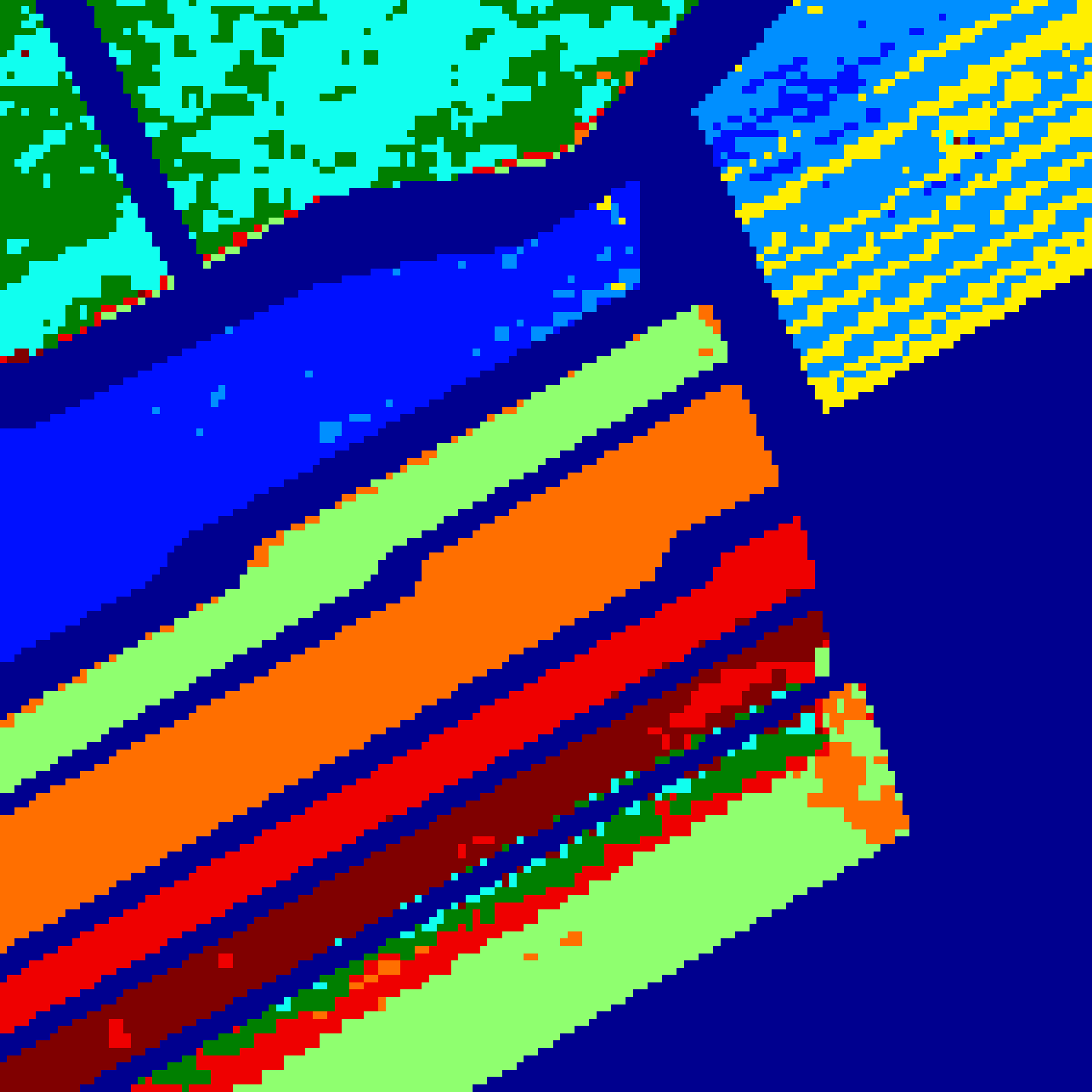}\label{salinas_FCM_XB_m9}}
\hfil
\centering
\subfloat[PCM]{\includegraphics[width=0.29\textwidth]{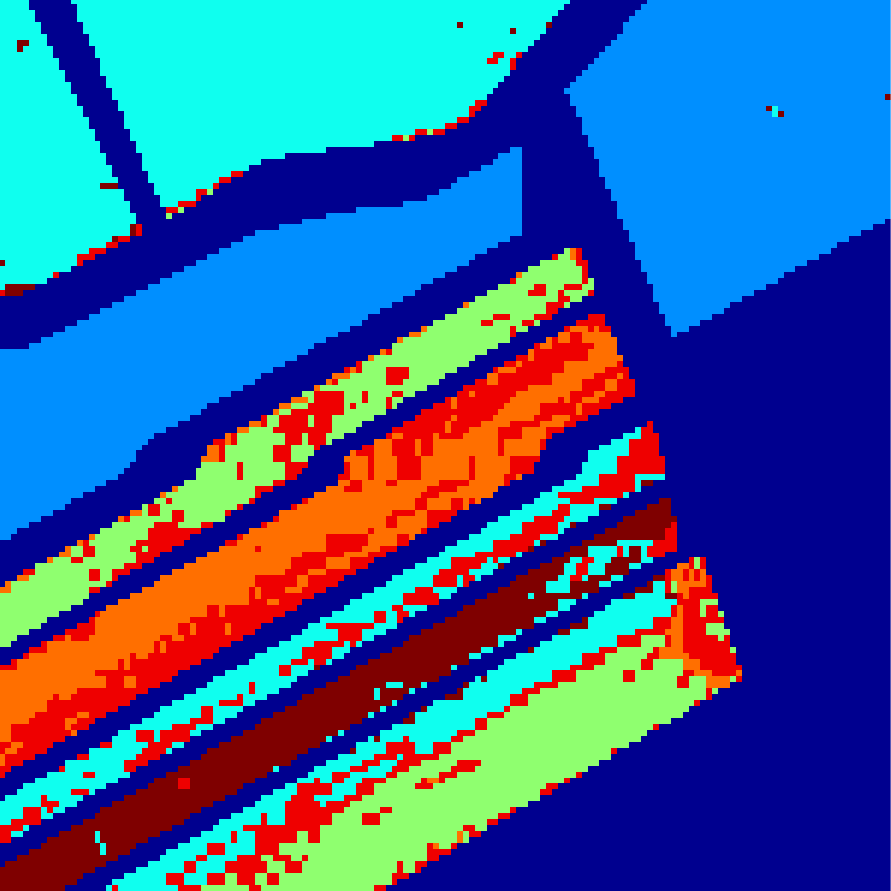}\label{salinas_PCM_m35}}
\hfil
\centering
\subfloat[APCM]{\includegraphics[width=0.29\textwidth]{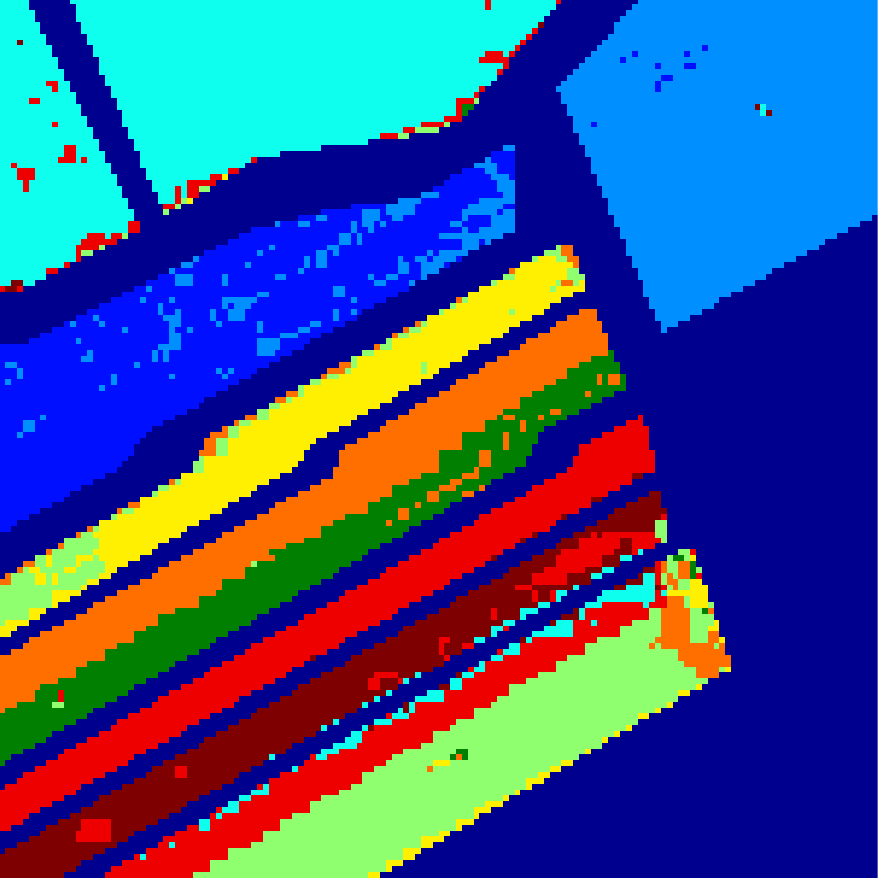}\label{salinas_APCM_m15}}
\hfil
\centering
\subfloat[UPC]{\includegraphics[width=0.29\textwidth]{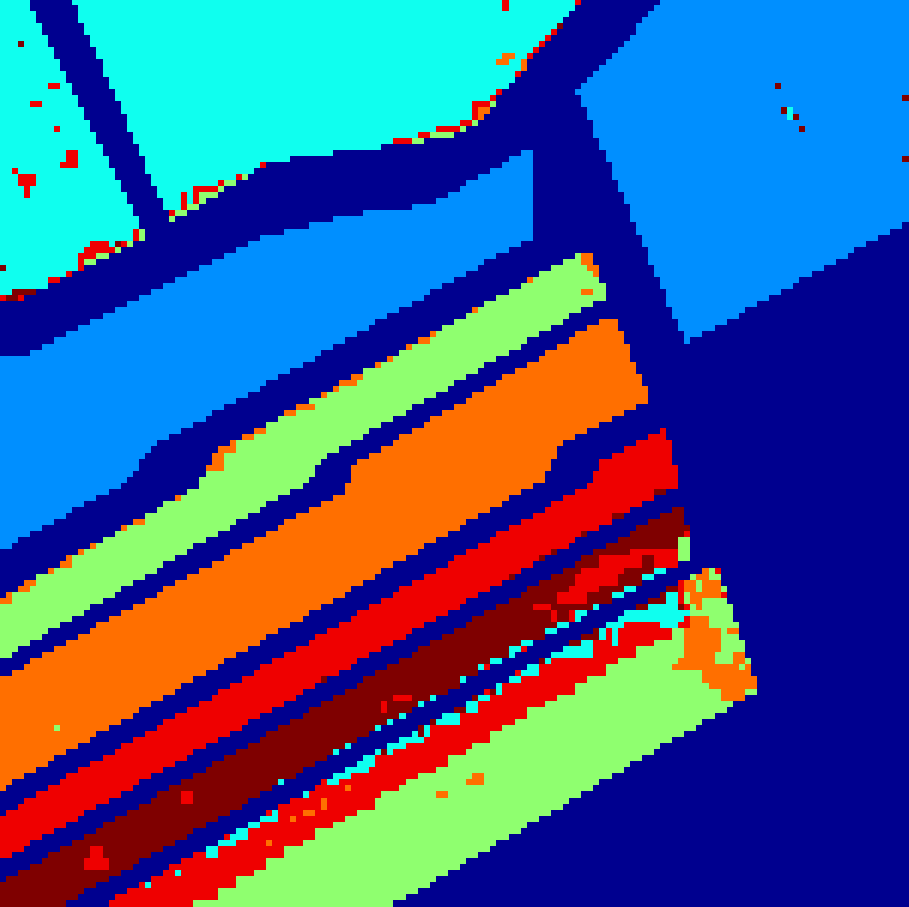}\label{salinas_UPC_m35}}
\hfil
\centering
\subfloat[PFCM]{\includegraphics[width=0.29\textwidth]{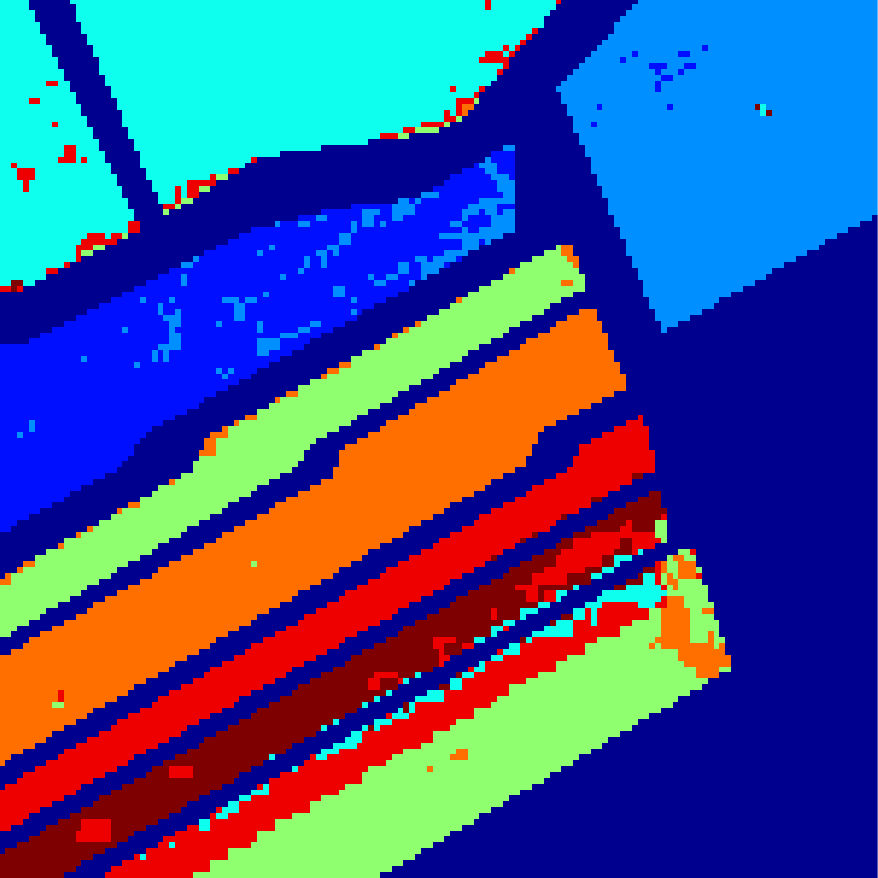}\label{salinas_PFCM_m8}}
\hfil
\centering
\subfloat[UPFC]{\includegraphics[width=0.29\textwidth]{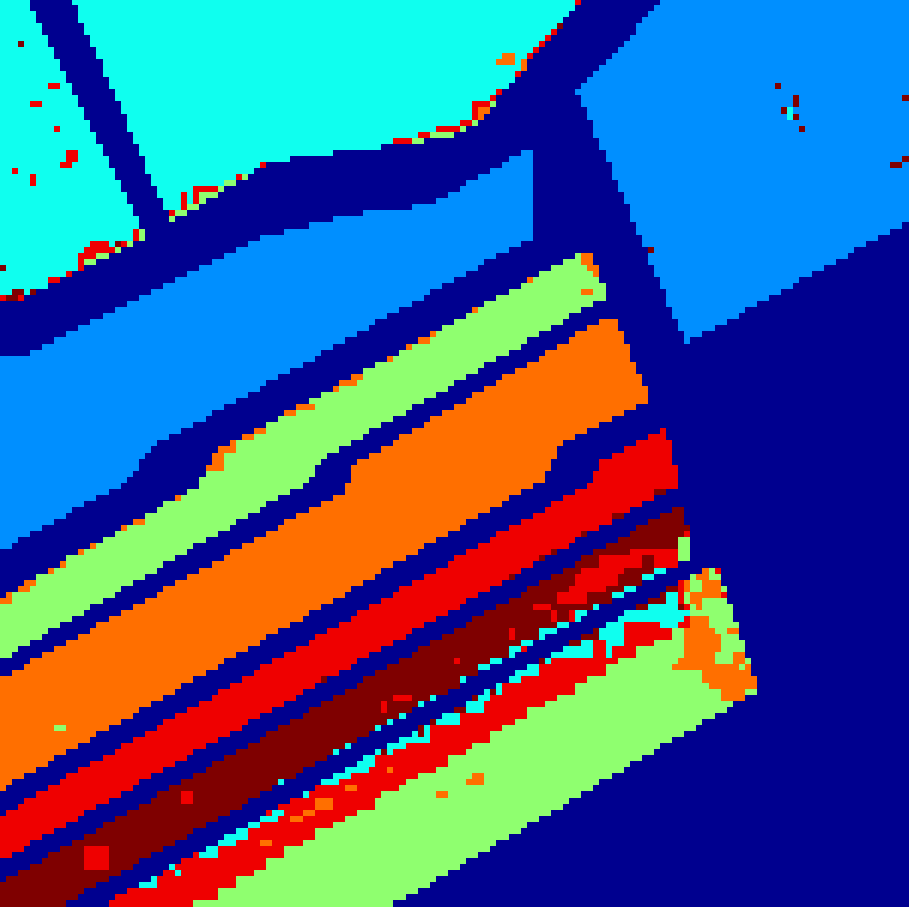}\label{salinas_UPFC_m35}}
\hfil
\centering
\subfloat[GRPCM]{\includegraphics[width=0.29\textwidth]{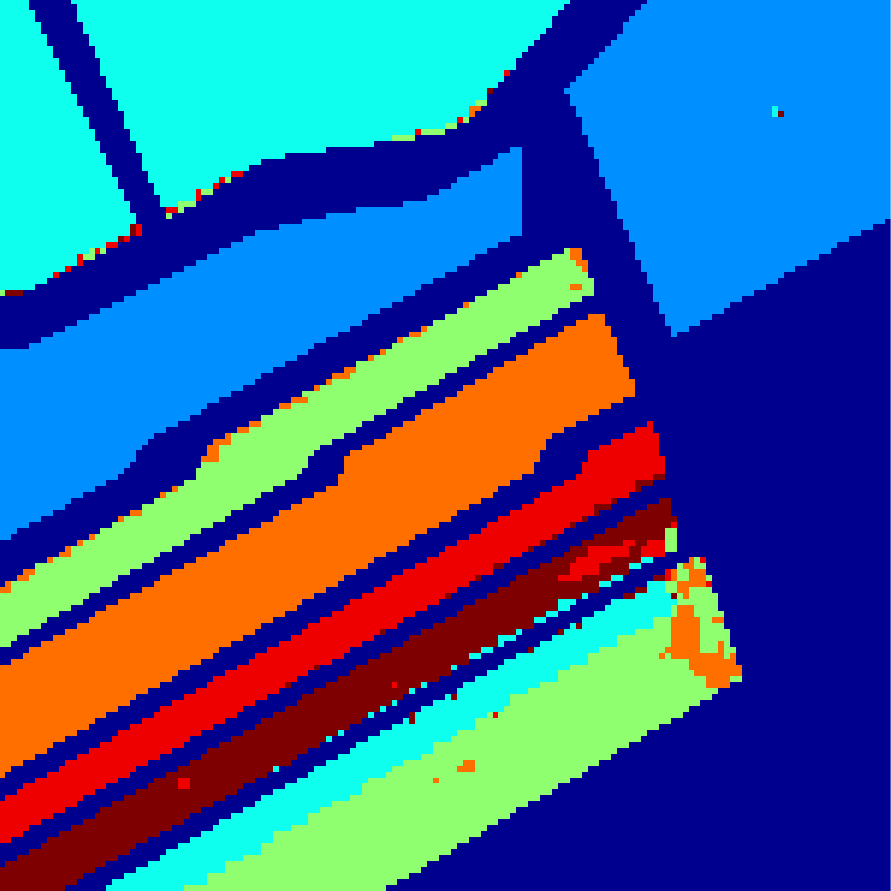}\label{salinas_GRPCM}}
\hfil
\centering
\subfloat[AMPCM]{\includegraphics[width=0.29\textwidth]{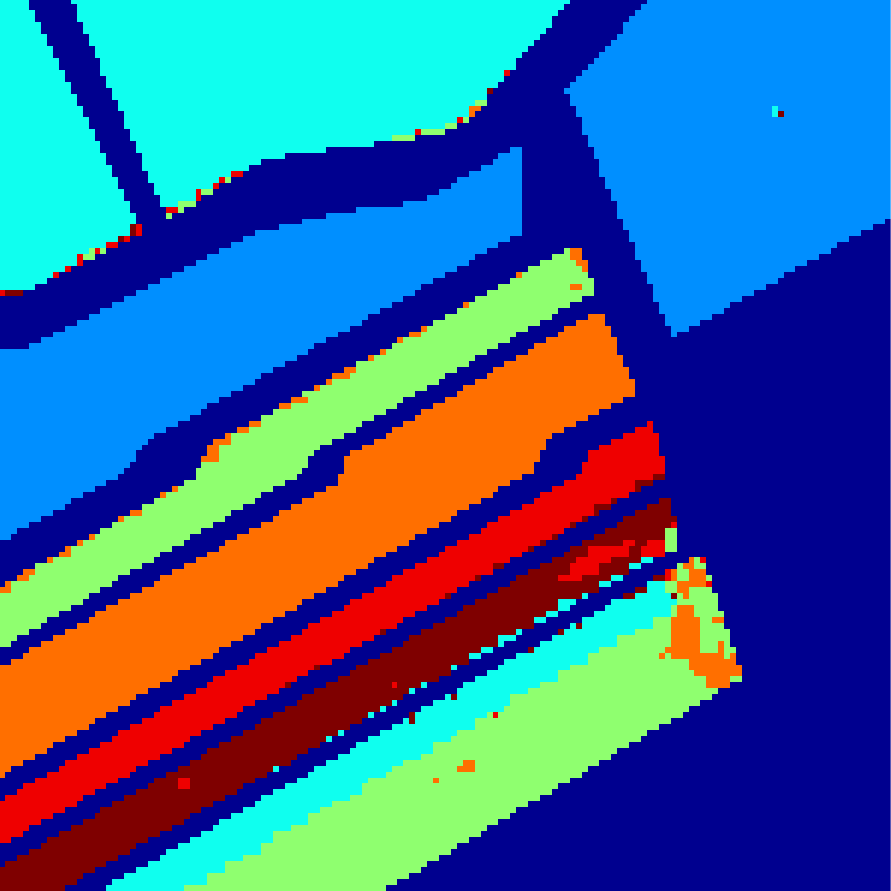}\label{salinas_AMPCM}}
\hfil
%\centering{\caption{ }}
\label{example6}
\end{figure*}
%\end{FPfigure}

\addtocounter{figure}{+1}
\begin{figure*}[htb!]
\centering
{\includegraphics[width=0.95\textwidth]{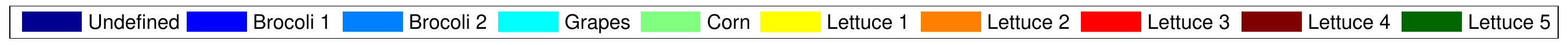}\label{salinas_legend}}
  \contcaption{(a) The 5th PC component of Salinas HSI and (b) the corresponding ground truth labeling. Clustering results of experiment 6 obtained from (c) k-means, $m_{ini}=8$, (d) FCM, $m_{ini}=8$, (e) FCM \& XB, (f) PCM, $m_{ini}=35$, (g) APCM, $m_{ini}=15$ and $\alpha=2$, (h) UPC, $m_{ini}=35$ and $q=3$, (i) PFCM, $m_{ini}=8$, $K=1$, $\alpha=1$, $\beta=6$, $q=2$ and $n=2$, (j) UPFC, $m_{ini}=35$, $\alpha=1$, $\beta=5$, $q=5$ and $n=2$, (k) GRPCM and (l) AMPCM.}% Continued caption
\end{figure*}

As it can be deduced from Fig.~\ref{example6}, when k-means and FCM are initialized with $m_{ini}=8$, they actually split the ``Broccoli 2" class into two clusters and they merge a part of ``Corn" class with ``Lettuce 3" and the rest of it with ``Lettuce 1". FCM \& XB validity index concludes to $m_{final}=9$ and merges a part of ``Corn" class with ``Lettuce 1", while the rest of it constitutes a seperate cluster, which also appears in scattered spots in the ``Grapes" class area. In addition, FCM \& XB requires long time to conclude to a final clustering result, due to the fact that FCM is executed for several values of $m_{ini}$, which increases the required computational time in high dimensional data sets. The PCM algorithm fails to uncover more than 6 discrete clusters, merging firstly ``Lettuce 3", ``Grapes" and a part of ``Corn" class and secondly the rest part of ``Corn" with ``Lettuce 1". Moreover, it merges the two types of ``Broccoli" into one, but splits the ``Lettuce 2" class into two clusters (whose pixels are spread over several classes of the image). Both UPC and UPFC algorithms are able to detect up to 6 clusters, having the same behavior as FCM and k-means (when the latter produce 8 clusters), except that they both merge the two ``Broccoli" classes into one. PFCM algorithm, after precise fine tuning of its parameters, manages additionally to distinguish the two types of ``Broccoli" classes, compared to UPC and UPFC, producing thus 7 clusters. GRPCM and AMPCM algorithms both end up with 6 clusters, merging the two ``Broccoli" classes into one, a part of ``Corn" class with ``Grapes" and the rest of it with ``Lettuce 1". However, the most important thing to be mentioned is that both these algorithms require excessively long time to converge in high dimensional data sets. Finally, APCM is the only algorithm that manages to distinguish the ``Lettuce 1" from the ``Corn" class, while at the same time it does not merge any other of the existing classes. 

Let us focus for a while on the ``Lettuce 2" class. This class forms two closely located clusters in the feature space, although this information is not reflected to the ground-truth labeling (note however that it can be deduced after inspection of the 5th PC component in Fig.~\ref{salinas_5thPCA}). It is important to note that, in contrast to APCM, none of the other algorithms succeeds in identifying each one of them. The fact that this is not reflected in the ground-truth labeling causes a misleading decrease in the SR performance of APCM. 
%This is justified by the fact that, as shown in Table~\ref{table:HSI}, the SR measure is almost the same for $m_{final}$ equal to 7 and 9 for APCM, while this is not the case with PFCM, where the SR measure for $m_{final}=8$ is significantly decreased compared to the $m_{final}=7$ case. The latter is an indication that, in contrast to APCM, moving from 7 to 8 clusters, the PFCM meshes the physical clustering structure of the data set.

\section{Conclusion}
\label{sec5}
In this paper, commencing from the classic possibilistic c-means (PCM) algorithm proposed in \cite{Kris96}, a novel possibilistic clustering algorithm, called Adaptive Possibilistic c-means (APCM), has been derived exhibiting several new features. The main one is that its parameters $\gamma$ are adapted as the algorithm evolves, in contrast to all the other possibilistic algorithms, where parameters $\gamma$, once they are set, they remain fixed during the execution of the algorithm. This gives APCM more flexibility in tracking the variations in the cluster formation as the algorithm evolves. Additional significant features are related with the computation of the parameters $\gamma$. Specifically, in contrast to previous possibilistic algorithms, each $\gamma_j$ is expressed in terms to the mean absolute deviation of the vectors that are most compatible with the $j$th cluster ($C_j$), from their mean. The use of the Euclidean distance, instead of the squared Euclidean one, gives the ability to the algorithm to distinguish closely located to each other clusters. Moreover, the use of the mean instead of the previous location of the corresponding representative in the computation of $\gamma_j$'s gives better estimates for the latter. A significant side-effect of the adaptation of $\gamma_j$'s is that APCM is now (in principle) capable to detect the true number, $m$, of physical clusters provided that it is initialized with an overestimate of it, $m_{ini}$. The latter releases APCM from the noose of knowing exactly in advance the true number of ``physical" clusters. It is worth noting that as experiments shown, $m_{ini}$ and $\alpha$ should vary inversely to each other, in order the algorithm to work properly, which makes their choice not entirely arbitrary. In addition, they show that if $\alpha$ is fixed to a value around 1 and $m_{ini}$ is around 3-4 times greater than $m$, then, in several cases, the algorithm works properly. The experimental results provided show that APCM exhibits superior performance compared to several other related algorithms, in almost all the considered data sets. In addition, Appendix B contains some indicative theoretical results, concerning the convergence behavior of APCM. Extension of APCM for identifying noisy data points and outliers, based on the concept of ``sparsity", is a subject of on going investigation.

%
%\newpage
\appendices
\section{}
\begin{prop}
Let $\gamma_j'=\frac{\sum\nolimits_{\mathbf{x}_i:u_{ij}=\max_{r=1,...,m} u_{ir}} \|\mathbf{x}_i-\boldsymbol{\mu}_j\|^2}{n_j}$ and $\eta_j^2=\left(\frac{\sum\nolimits_{\mathbf{x}_i:u_{ij}=\max_{r=1,...,m} u_{ir}}^{} \|\mathbf{x}_i-\boldsymbol{\mu}_j\|}{n_j}\right)^2$ (see eq.~\eqref{transformations2}). Then $\eta_j^2\leq\gamma_j'$.
\label{prop1}
\end{prop}

\begin{proof}[Proof]
Let $q_{ij}=\|\mathbf{x}_i-\boldsymbol{\mu}_j\|$ and $\mathbf{q}_j=[q_{i1},\ldots,q_{in_j}]^T$. Then $\gamma_j'=\frac{1}{n_j}\|\mathbf{q}_j\|^2_2$ (squared $l_2$-norm) and $\eta_j^2=\frac{1}{n_j^2}\|\mathbf{q}_j\|^2_1$ (squared $l_1$-norm). From the relation between the $l_1$ and $l_2$ norms (see e.g. \cite{Bert89}), it is: 
$\|\mathbf{q}_j\|_1\leq n_j^{1/2}\|\mathbf{q}_j\|_2\leq n_j^{1/2}\|\mathbf{q}_j\|_1$ or $\|\mathbf{q}_j\|^2_1\leq n_j\|\mathbf{q}_j\|^2_2\leq n_j\|\mathbf{q}_j\|^2_1$ or $\frac{1}{n_j^2}\|\mathbf{q}_j\|^2_1\leq\frac{1}{n_j}\|\mathbf{q}_j\|^2_2\leq
n_j\frac{1}{n_j^2}\|\mathbf{q}_j\|^2_1$ or $\eta_j^2\leq\gamma_j'\leq n_j\eta_j^2$. Note that for finite $n_j$ values, $\eta_j^2$ and $\gamma_j'$ are of the same magnitude.
\end{proof}

\section{}
In this appendix we prove some propositions that are indicative of the basic properties of APCM, namely the convergence of the representatives to the center of dense regions and cluster elimination. Note that some convergence results are given in \cite{Zhou13}. However, these are not applicable to APCM, due to the adaptation mechanism employed for the parameters $\eta_j$'s. We begin with the following proposition.

\begin{prop}
Let $\boldsymbol{\theta}_1$, $\boldsymbol{\theta}_2$ be two cluster representatives with $\eta_2<\eta_1$. The geometrical locus of the points $\mathbf{x}\in\Re^\ell$ having $u_2(\mathbf{x})>u_1(\mathbf{x})$, where $u_j(\mathbf{x})=\exp\left(-\frac{\alpha d_j(\mathbf{x})}{\eta_j\hat{\eta}}\right)$ and $d_j(\mathbf{x})=\|\mathbf{x} - \boldsymbol{\theta}_j\|^2$, $j=1,2$, is the set of points that lie in the interior of the hypersphere $C$:
\begin{equation}
\|\mathbf{x}-\frac{k\boldsymbol{\theta}_2-\boldsymbol{\theta}_1}{k-1}\|^2 = \frac{k}{(k-1)^2}\|\boldsymbol{\theta}_2-\boldsymbol{\theta}_1\|^2\equiv r^2,
\label{eqprop2}
\end{equation}
centered at $\frac{k\boldsymbol{\theta}_2-\boldsymbol{\theta}_1}{k-1}$ and having radius $r=\frac{\sqrt{k}}{k-1}\|\boldsymbol{\theta}_2-\boldsymbol{\theta}_1\|$, where $k=\eta_1/\eta_2(>1)$.
\label{prop2}
\end{prop}

\begin{proof}[Proof]
It is $u_1(\mathbf{x})<u_2(\mathbf{x}) \Leftrightarrow \frac{d_1(\mathbf{x})}{\eta_1}>\frac{d_2(\mathbf{x})}{\eta_2} \Leftrightarrow d_1(\mathbf{x})>kd_2(\mathbf{x}) \Leftrightarrow \|\mathbf{x} - \boldsymbol{\theta}_1\|^2 > k\|\mathbf{x} - \boldsymbol{\theta}_2\|^2 \Leftrightarrow \|\mathbf{x}\|^2-2\mathbf{x}^T\boldsymbol{\theta}_1+\|\boldsymbol{\theta}_1\|^2 > k\|\mathbf{x}\|^2-2k\mathbf{x}^T\boldsymbol{\theta}_2+k\|\boldsymbol{\theta}_2\|^2 \Leftrightarrow (k-1)\|\mathbf{x}\|^2-2(k\boldsymbol{\theta}_2-\boldsymbol{\theta}_1)^T\mathbf{x}
+k\|\boldsymbol{\theta}_2\|^2-\|\boldsymbol{\theta}_1\|^2<0 \Leftrightarrow \|\mathbf{x}\|^2-2\left(\frac{k\boldsymbol{\theta}_2-\boldsymbol{\theta}_1}{k-1}\right)^T\mathbf{x} + \frac{k\|\boldsymbol{\theta}_2\|^2-\|\boldsymbol{\theta}_1\|^2}{k-1}<0 \Leftrightarrow \|\mathbf{x}\|^2-2\left(\frac{k\boldsymbol{\theta}_2-\boldsymbol{\theta}_1}{k-1}\right)^T\mathbf{x} + \|\frac{k\boldsymbol{\theta}_2-\boldsymbol{\theta}_1}{k-1}\|^2 - \|\frac{k\boldsymbol{\theta}_2-\boldsymbol{\theta}_1}{k-1}\|^2 + \frac{k\|\boldsymbol{\theta}_2\|^2-\|\boldsymbol{\theta}_1\|^2}{k-1}<0 \Leftrightarrow \|\mathbf{x}-\frac{k\boldsymbol{\theta}_2-\boldsymbol{\theta}_1}{k-1}\|^2 < \frac{\|k\boldsymbol{\theta}_2-\boldsymbol{\theta}_1\|^2}{(k-1)^2} - \frac{k\|\boldsymbol{\theta}_2\|^2-\|\boldsymbol{\theta}_1\|^2}{k-1}$ or $\|\mathbf{x}-\frac{k\boldsymbol{\theta}_2-\boldsymbol{\theta}_1}{k-1}\|^2 < \frac{k}{(k-1)^2}\|\boldsymbol{\theta}_2-\boldsymbol{\theta}_1\|^2$.
\end{proof}
%Clearly, the geometrical locus of points for which $u_{i1}<u_{i2}$ is the interior of the circle $C$ of eq.~\eqref{eqprop2}.
\noindent
Note that the radius $r$ of $C$ can be written in terms of $\eta_1$, $\eta_2$ as 
\begin{equation}
r=\frac{\sqrt{\eta_1\eta_2}}{|\eta_1-\eta_2|}\|\boldsymbol{\theta}_2-\boldsymbol{\theta}_1\|^2
\label{radius}
\end{equation}

We consider next the continuous case where the data vectors are modelled by a random vector $\mathbf{x}$ that follows a continuous pdf distribution $p(\mathbf{x})$. In this case, the updating equations for the APCM algorithm (with a slight modification in notation, in order to denote explicitly the dependence of $u_j(\mathbf{x})$ from the continuous random variable $\mathbf{x}$) are given below.
%
% version of the update equations used in PCM type algortihms, i.e.

\vspace{0.1cm}
\begin{minipage}{0.9\linewidth}  
\begin{equation}  
\boldsymbol{\theta}_j^{t+1}=\frac{\int_{\Re^\ell} u_j^t(\mathbf{x})\mathbf{x}p(\mathbf{x})d\mathbf{x}}{\int_{\Re^\ell} u_j^t(\mathbf{x})p(\mathbf{x})d\mathbf{x}}
\label{eqprop31}  
\end{equation}  
\end{minipage}  

\vspace{0.1cm}

\begin{minipage}{0.9\linewidth}  
\begin{equation}  
\text{ where \ \ } u_j^t(\mathbf{x})=\exp\left(-\frac{\|\mathbf{x}-\boldsymbol{\theta}_j^t\|^2}{\gamma_j^t}\right)
\label{eqprop32}   
\end{equation}
\end{minipage}

\vspace{0.1cm}
\begin{minipage}{0.9\linewidth}  
\begin{equation}  
\gamma_j^t=\frac{\hat{\eta}}{\alpha}\frac{\int_{T_j^t} \|\mathbf{x}-\boldsymbol{\mu}_j^t\|p(\mathbf{x})d\mathbf{x}}{\int_{T_j^t} p(\mathbf{x})d\mathbf{x}}
\label{eqprop33}  
\end{equation}  
\end{minipage}

\vspace{0.1cm}

\begin{minipage}{0.9\linewidth}  
\begin{equation} 
\hspace{-2cm} 
\text{and \ \ } \boldsymbol{\mu}_j^t=\frac{\int_{T_j^t} \mathbf{x}p(\mathbf{x})d\mathbf{x}}{\int_{T_j^t} p(\mathbf{x})d\mathbf{x}}
\label{eqprop34}   
\end{equation}
\end{minipage}

\noindent
with $T_j^t=\{\mathbf{x}:u_j^t(\mathbf{x})=\max_{q=1,\ldots,m} u_q^t(\mathbf{x})\}$, $j=1,\ldots,m$.

The above equations define the iterative scheme $\boldsymbol{\theta}_j^{t+1}=f(\boldsymbol{\theta}_j^{t})$, where 
\begin{equation}
f(\boldsymbol{\theta}_j^t)=\frac{\int_{\Re^\ell} \exp\left(-\frac{\|\mathbf{x}-\boldsymbol{\theta}_j^t\|^2}{\gamma_j}\right)\mathbf{x}p(\mathbf{x})d\mathbf{x}}{\int_{\Re^\ell} \exp\left(-\frac{\|\mathbf{x}-\boldsymbol{\theta}_j^t\|^2}{\gamma_j}\right)p(\mathbf{x})d\mathbf{x}}
\label{eqprop17}
\end{equation}

%In the sequel we give a proposition that indicates the vital property of all PCM algorithms, i.e. the convergence of the cluster representatives to the centers of dense in data regions. For simplicity, the case where only a single dense region exists is considered.
In the sequel we give some indicative theoretical results concerning aspects of the behavior of APCM, namely (a) the convergence of the cluster representatives to the centers of the dense in data regions and (b) the cluster elimination mechanism. In the sequel, we state two assumptions that will be used as premises in the propositions to follow.

\noindent
\newline
{\it Assumption 1:} (a) $p(\mathbf{x})$ decreases isotropically along all directions around its center $\mathbf{c}$ \footnote{Such pdf's are e.g. the independent identically distributed (i.i.d) multivariate normal and Laplace distribution.}. 

\noindent
(b) Without loss of generality, we consider the case $\mathbf{c}=\mathbf{0}$.

\noindent
\newline
Note that this assumption indicates the existence of a {\it single} dense in data region.

\noindent
\newline
{\it Assumption 2:} $p(\mathbf{x})$ is a zero mean normal distribution ${\cal N}(\mathbf{0},\sigma^2I)$.

\noindent
(Clearly, Assumption 2 is more restrictive than assumption 1.)

\begin{prop}
Under assumption 1, the center $\mathbf{c}=\mathbf{0}$ of $p(\mathbf{x})$ is a fixed point for the iterative scheme defined by eq.~\eqref{eqprop17}.
\label{prop3}
\end{prop}

\begin{proof}[Proof]
%Without loss of generality we assume that the center of $p(\mathbf{x})$ is located at the origin, that is, it is $\mathbf{0}$.
%
%From eqs.~\eqref{eqprop31} and \eqref{eqprop32}, 
%with a slight abuse of notation (making time indices superscripts) 
Assuming that $\boldsymbol{\theta}_j^{t}=\mathbf{0}$, we will show that $\boldsymbol{\theta}_j^{t+1}=\mathbf{0}$ also. Dropping the index $j$ from $\boldsymbol{\theta}_j$, $\gamma_j$ from eq.~\eqref{eqprop17} we have
\begin{equation}
\boldsymbol{\theta}^{t+1}=\frac{\int_0^\infty \left[\int_{\|\mathbf{x}\|^2=r^2} \exp\left(-\frac{\|\mathbf{x}\|^2}{\gamma^t}\right)\mathbf{x}p(\mathbf{x})dA_r\right]dr}{\int_0^\infty \left[\int_{\|\mathbf{x}\|^2=r^2} \exp\left(-\frac{\|\mathbf{x}\|^2}{\gamma^t}\right)p(\mathbf{x})dA_r\right]dr}
\label{eqprop36}
\end{equation}
where $\int_{\|\mathbf{x}\|^2=r^2} (\cdot)dA_r$ is the integral over the hypersphere $\|\mathbf{x}\|^2=r^2$.

Continuing from eq.~\eqref{eqprop36} we have
\begin{equation}
\boldsymbol{\theta}^{t+1}=\frac{\int_0^\infty \exp\left(-\frac{r^2}{\gamma^t}\right)\left[\int_{\|\mathbf{x}\|^2=r^2} \mathbf{x}p(\mathbf{x})dA_r\right]dr}{\int_0^\infty \exp\left(-\frac{r^2}{\gamma^t}\right)\left[\int_{\|\mathbf{x}\|^2=r^2} p(\mathbf{x})dA_r\right]dr}
\label{eqprop37}
\end{equation}

But, due to the isotropic property of $p(\mathbf{x})$ along all directions around $\mathbf{0}$, all points on the hypersphere $\|\mathbf{x}\|^2=r^2$ are evenly distributed (and have the same magnitude). Thus, it is:
\begin{equation}
\int_{\|\mathbf{x}\|^2=r^2} \mathbf{x}p(\mathbf{x})dA_r=\mathbf{0}
\label{eqprop38}
\end{equation}

Noting also that $\exp\left(-\frac{r^2}{\gamma^t}\right)>0$ and $\int_{\|\mathbf{x}\|^2=r^2} p(\mathbf{x})dA_r$ is the area of the hypersphere $\|\mathbf{x}\|^2=r^2$, the denominator in eq.~\eqref{eqprop37} is positive. Thus, eqs.~\eqref{eqprop37} and \eqref{eqprop38} finally give $\boldsymbol{\theta}^{t+1}=\mathbf{0}$. In other words, $\mathbf{0}$ is indeed a fixed point of the iterative scheme defined by eq.~\eqref{eqprop17}.
\end{proof}

\begin{prop}
Adopt the assumption 2 and consider the mapping $f:\Re^\ell\rightarrow\Re^\ell$ defined by eq.~\eqref{eqprop17}. Then, the fixed point $\mathbf{0}$ of the scheme $\boldsymbol{\theta}^{t+1}=f(\boldsymbol{\theta}^{t})$ is stable.
\label{prop4}
\end{prop}

\begin{proof}[Proof]
Focusing on the $s$-th component $f_s(\boldsymbol{\theta})$ of the above mapping and utilizing the assumption 2 of $p(\mathbf{x})$ as well as the fact that $\exp\left(-\frac{\|\mathbf{x}-\boldsymbol{\theta}\|^2}{\gamma}\right)=\prod\limits_{q=1}^{\ell} \exp\left(-\frac{(x_q-\theta_q)^2}{\gamma}\right)$, it is easy to verify that:
\begin{equation}
f_s(\boldsymbol{\theta})=\frac{\int_{\Re} x_s\exp\left(-\frac{(x_s-\theta_s)^2}{\gamma}\right)p(x_s)dx_s}{\int_{\Re} \exp\left(-\frac{(x_s-\theta_s)^2}{\gamma}\right)p(x_s)dx_s}\equiv f_s(\theta_s)
\label{eqprop421}
\end{equation}
Thus, $f_s(\boldsymbol{\theta})$ depends only on $\theta_s$.

In order to prove the stability of $\boldsymbol{\theta}=\mathbf{0}$, we will compute the Jacobian matrix on $\boldsymbol{\theta}=\mathbf{0}$ and we will show that $|J(\boldsymbol{\theta})|<1$.

Since, $\frac{\partial f_s(\boldsymbol{\theta})}{\partial \theta_q}=0$, for $q\neq s$ the Jacobian is diagonal. Computing its diagonal elements at $\boldsymbol{\theta}=\mathbf{0}$, we have after some algebra
%\begin{equation}
\begin{multline}
\left.\frac{\partial f_s(\boldsymbol{\theta})}{\partial \theta_s}\right|_{\boldsymbol{\theta}=\mathbf{0}}=\frac{2}{\gamma}\frac{\int_{\Re} x_s^2\exp\left(-\frac{x_s^2}{\gamma}\right)p(x_s)dx_s}{\int_{\Re} \exp\left(-\frac{x_s^2}{\gamma}\right)p(x_s)dx_s}-\\-\frac{2}{\gamma}\frac{\left(\int_{\Re} x_s\exp\left(-\frac{x_s^2}{\gamma}\right)p(x_s)dx_s\right)^2}{\left(\int_{\Re} \exp\left(-\frac{x_s^2}{\gamma}\right)p(x_s)dx_s\right)^2}
\label{eqprop422}
%\end{equation}
\end{multline}

In addition, due to the fact that $p(x_s)$ is ${\cal N}(0,\sigma^2)$, it is easy to verify that
\begin{equation}
\exp\left(-\frac{x_s^2}{\gamma}\right)p(x_s)=\frac{\sigma '}{\sigma}\hat{p}(x_s)
\label{eqprop423}
\end{equation}
where $\hat{p}(x_s)$=${\cal N}(0,\sigma '{^2})$, with 
\begin{equation}
\sigma '{^2}=\frac{1}{2\left(\frac{1}{\gamma}+\frac{1}{2\sigma^2}\right)}
\label{eqprop424}
\end{equation}

Substituting eq.~\eqref{eqprop423} to eq.~\eqref{eqprop422} and taking into account that (a) the numinator of the second fraction is the mean of $\hat{p}(x_s)$, (b) the numinator of the first fraction is the variance of $\hat{p}(x_s)$ and (c) the denominators are both equal to 1, we end up with
\begin{equation}
\left.\frac{\partial f_s(\boldsymbol{\theta})}{\partial \theta_s}\right|_{\boldsymbol{\theta}=\mathbf{0}}=\frac{2\sigma '{^2}}{\gamma}
\label{eqprop425}
\end{equation}

Substituing eq.~\eqref{eqprop424} to eq.~\eqref{eqprop425}, it is: $\left.\frac{\partial f_s(\boldsymbol{\theta})}{\partial \theta_s}\right|_{\boldsymbol{\theta}=\mathbf{0}}=\frac{2\sigma^2}{2\sigma^2+\gamma}$, which is always less than 1, due to the positivity of $\sigma^2$ and $\gamma$.

Thus, $\boldsymbol{\theta}=\mathbf{0}$ is a stable fixed point of the iterative scheme $\boldsymbol{\theta}^{t+1}=f(\boldsymbol{\theta}^{t})$
\end{proof}

Propositions \ref{prop3} and \ref{prop4} are valid for both constant and time varying positive $\gamma_j$'s.

%Clearly, if $\boldsymbol{\theta}_j(0)$ belongs to the region of attraction of the center of the pdf $p(\mathbf{x})$, the iterative scheme of eqs.~\eqref{eqprop31}, \eqref{eqprop32} and \eqref{eqprop33} will gradually lead $\boldsymbol{\theta}_j$ towards the center of $p(\mathbf{x})$. 
In the general case where the data form more than one dense regions\footnote{That is, when $p(\mathbf{x})$ has more than one peaks.}, the above propositions are still valid, assuming that the influence on a representative that belongs to a given dense region from data points from other dense regions is negligible. This can be ensured by choosing $\gamma_j$'s properly.

%Notice that propositions \ref{prop2}, \ref{prop3} and \ref{prop4} are valid for all PCM schemes where $\gamma_j$'s are constant. 
%In the next proposition we consider the case where the data stem from a pdf $p(\mathbf{x})$ that is isotropic along all directions around its center. That is, we have a single dense region. In addition, we consider the two representatives case and we show that one of them will be finally eliminated.

In the next proposition, we focus on the cluster elimination property of APCM for the case of two representatives that lie in the same physical cluster. 

\begin{prop}
%Consider the case where the data are modelled by a pdf $p(\mathbf{x})$ as described in Proposition \ref{prop3} and let $\mathbf{c}$ be its center. 
Adopt assumption 1 and consider two cluster representatives $\boldsymbol{\theta}_1$ and $\boldsymbol{\theta}_2$. 
% initialized in the region of attraction of $\mathbf{c}$, ROA($\mathbf{c}$), i.e. $\boldsymbol{\theta}_1(0)$, $\boldsymbol{\theta}_2(0)\in$ ROA($\mathbf{c}$). 
Assuming that $\eta_1(t)\neq\eta_2(t)$ and $\eta_j(t)<+\infty$, $j=1,2$, $\forall t$, one of the clusters represented by $\boldsymbol{\theta}_1$ and $\boldsymbol{\theta}_2$ will be eliminated\footnote{Note that, in practice, the hypothesis for $\eta_1(t)$ and $\eta_2(t)$ is almost always met, due to their definition.}.
\label{prop5}
\end{prop}

\begin{proof}[Proof]
Utilizing propositions \ref{prop3} and \ref{prop4}, we have that $\boldsymbol{\theta}_1$ and $\boldsymbol{\theta}_2$ converge towards $\mathbf{c}$. Thus, the distance between them decreases towards zero, i.e. 
\begin{equation}
\|\boldsymbol{\theta}_1(t)-\boldsymbol{\theta}_2(t)\|\rightarrow 0
\label{eqprop51}
\end{equation}		
%In addition, since $\eta_1(t)\neq\eta_2(t)$, $\forall t$, there exists an iteration $t_0$ in which $\|\boldsymbol{\theta}_1(t_0)-\boldsymbol{\theta}_2(t_0)\|<\eta_1(t_0)-\eta_2(t_0)$ (assuming, without loss of generality, that $\eta_1(t_0)>\eta_2(t_0)$). That is, the hypersphere centered at $\boldsymbol{\theta}_2(t_0)$ and having radius $\eta_2(t_0)$ lies inside the hypersphere centered at  $\boldsymbol{\theta}_1(t_0)$ and having radius $\eta_1(t_0)$.

%Utilizing Proposition \ref{prop2} we have that in this case:
%$$T_2(t_0)=\{\mathbf{x}:\mathbf{x}\in int(C_{t_0})\} \text{ and } T_1(t_0)=\{\mathbf{x}:\mathbf{x}\notin int(C_{t_0})\},$$ where $int(C_{t_0})$ denotes the interior of the sphere $C_{t_0}$, defined by eq.~\eqref{eqprop2}.

Taking into account eq.~\eqref{radius}, the radius of the hypersphere $C_t$ that delimits $T_1(t)$ and $T_2(t)$ at iteration $t$ can be written as 
\begin{equation}
r_t=\frac{\sqrt{\eta_1(t)\eta_2(t)}}{|\eta_1(t)-\eta_2(t)|} \|\boldsymbol{\theta}_2(t)-\boldsymbol{\theta}_1(t)\|
\label{eqprop52}
\end{equation}

From hypothesis it follows that $\frac{\sqrt{\eta_1(t)\eta_2(t)}}{|\eta_1(t)-\eta_2(t)|}$ is finite, i.e.,
\begin{equation}
\exists M>0: \left|\frac{\sqrt{\eta_1(t)\eta_2(t)}}{\eta_1(t)-\eta_2(t)}\right|<M \ \ \ \forall t \label{eqprop53}
\end{equation}

Combining eqs.~\eqref{eqprop51}, \eqref{eqprop52} and \eqref{eqprop53} we have that $r_t\rightarrow 0$. Thus $T_j(t)$ for one of the two representatives will eventually becomes empty, which will lead the corresponding $\eta_j(t)$ to zero value (see eq.~\eqref{adapteta}) and thus to the elimination of cluster $C_j$ (from the execution of statements 13-20 of APCM).
\end{proof}

\ifCLASSOPTIONcompsoc
\else
\fi

%This research has been co-financed by the European Union (European Social Fund - ESF) and Greek national funds through the Operational Program "Education and Lifelong Learning" of the National Strategic Reference Framework (NSRF) - Research Funding Program: ARISTEIA- HSI-MARS-1413.

\ifCLASSOPTIONcaptionsoff
%  \newpage
\fi

\begin{IEEEbiography}
[{\includegraphics[width=1in,height=1.25in,clip,keepaspectratio]{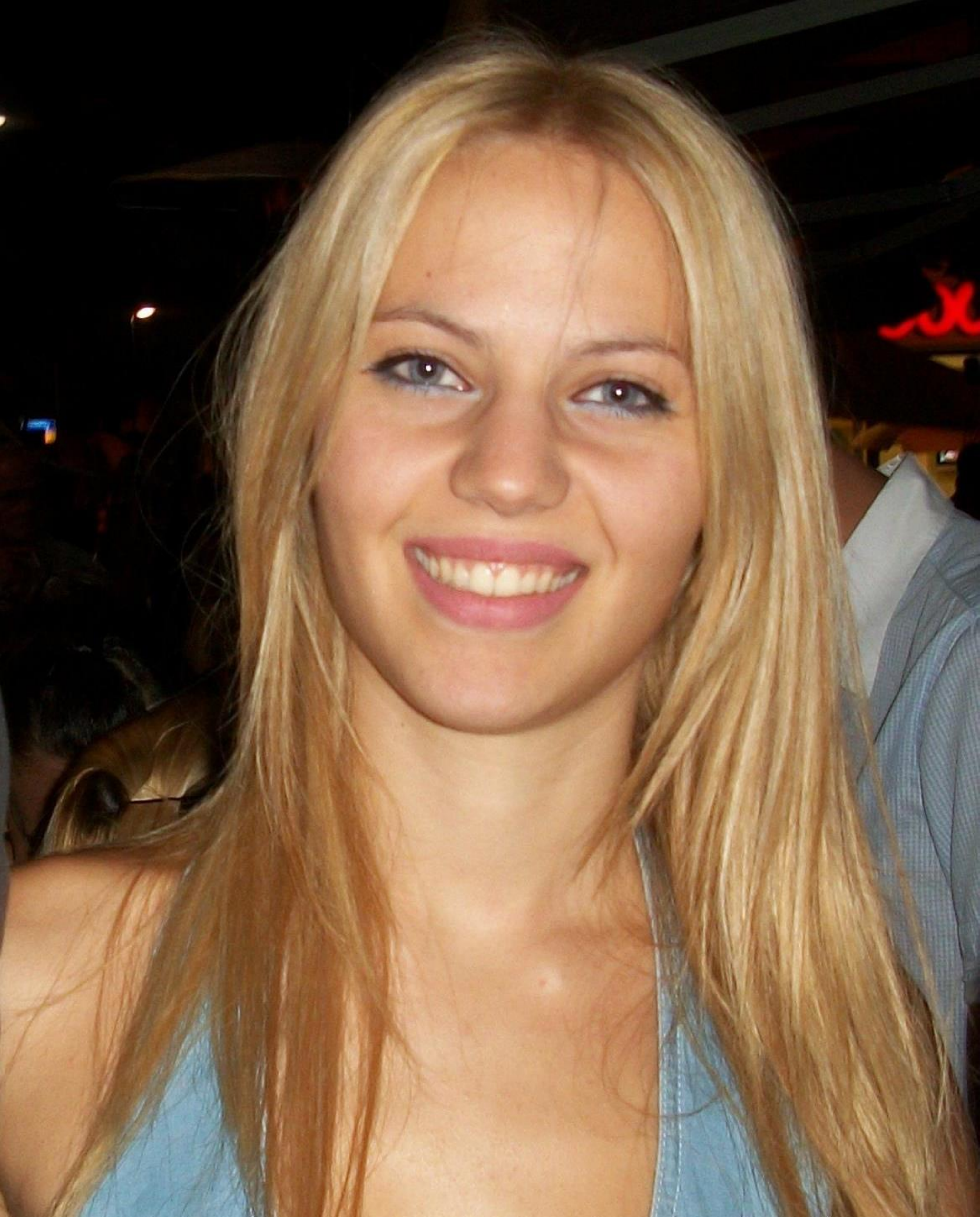}}]%
{Spyridoula Xenaki}
was born in Piraeus, Greece, in 1988. She received her Diploma degree in Infomatics and Telecommunications in 2010 and her M.Sc. degree in Signal Processing for Communication and Multimedia in 2012, both from the National and Kapodistrian University of Athens. From 2013, she is a Ph.D. student in the area of Signal Processing at the University of Athens in co-operation with the Institute for Astronomy, Astrophysics, Space Applications and Remote Sensing (IAASARS) of the National Observatory of Athens (NOA). Her research interests are in the area of signal processing and pattern recognition with application to image processing.
\end{IEEEbiography}

\begin{IEEEbiography}[{\includegraphics[width=1in,height=1.25in,clip,keepaspectratio]{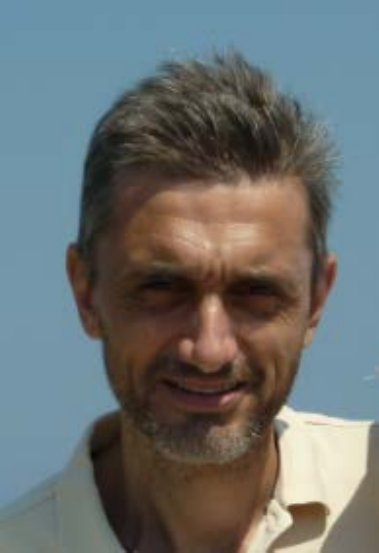}}]%
{Konstantinos Koutroumbas}
received the Diploma degree from the University of Patras (1989), an Ms.C. degree in advanced methods in computer science from the Queen Mary College of the University of London (1990) and a Ph.D. degree from the University of Athens (1995).

Since 2001 he is with the Institute of Astronomy, Astrophysics, Space Applications and Remote Sensing of the National Observatory of Athens, Greece, where currently he is a Senior Researcher. His research interests include mainly Pattern Recognition, Time Series Estimation and their application to (a) remote sensing and (b) the estimation of characteristic quantities of the upper atmosphere. He has co-authored the books Pattern Recognition (1st, 2nd, 3rd, 4th editions) and Introduction to Pattern Recognition: A MATLAB Approach. He has over 3000 citations in his work.
\end{IEEEbiography}

\begin{IEEEbiography}[{\includegraphics[width=1in,height=1.25in,clip,keepaspectratio]{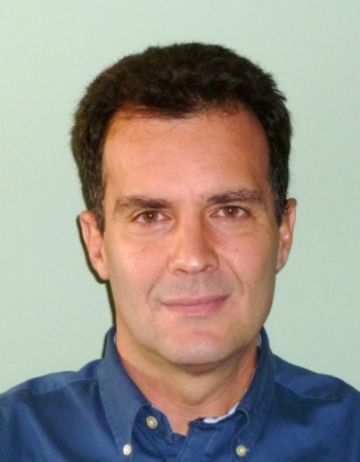}}]%
{Athanasios A. Rontogiannis}
(M'97) was born in Lefkada Island, Greece, in 1968. He received the (5 years) Diploma degree in electrical engineering from the National Technical University of Athens (NTUA), Greece, in 1991, the M.A.Sc. in electrical and computer engineering from the University of Victoria, Canada, in 1993, and the Ph.D. in communications and signal processing from the University of Athens, Greece, in 1997. 

From 1998 to 2003, he was with the University of Ioannina. In 2003 he joined the Institute for Astronomy, Astrophysics, Space Applications and Remote Sensing (IAASARS) of the National Observatory of Athens (NOA), where since 2011 he is a Senior Researcher. His research interests are in the general areas of statistical signal processing and wireless communications with emphasis on adaptive estimation, hyperspectral image processing, Bayesian compressive sensing, channel estimation/equalization and cooperative communications. Currently, he serves at the Editorial Boards of the EURASIP Journal on Advances in Signal Processing, Springer (since 2008) and the EURASIP Signal Processing Journal, Elsevier (since 2011). He is a member of the IEEE Signal Processing and Communication Societies and the Technical Chamber of Greece.
\end{IEEEbiography}

\end{document}